%% file: arxiv.tex
\pdfoutput=1 
\documentclass{article} % For LaTeX2e
\usepackage{times}
\usepackage{iclr2027_conference}
\iclrfinalcopy

\usepackage[utf8]{inputenc}
\usepackage[T1]{fontenc}
\usepackage{url}
\usepackage{booktabs}
\usepackage{tabularx}
\usepackage{amsfonts}
\usepackage{nicefrac}
\usepackage{xcolor}

\usepackage{amsmath}
\usepackage{amssymb}
\usepackage{amsthm}
\usepackage{bm}
\usepackage{graphicx}
\usepackage{algorithm}
\usepackage{algorithmic}
\usepackage{multirow,placeins}

\usepackage{hyperref}

\input{math_commands.tex}

\newtheorem{theorem}{Theorem}
\newtheorem{lemma}{Lemma}
\newtheorem{definition}{Definition}
\newtheorem{proposition}{Proposition}
\newtheorem{assumption}{Assumption}
\theoremstyle{remark}

\title{On the Limits of Sparse Linear Model\\with Shuffled Differential Privacy}

\author{
Yan Zhang$^{1,2,6}$ \quad
Jiong Tuo$^{1,3,6}$ \quad
Meng Ding$^{4}$ \quad
Yuheng Ma$^{5}$ \quad
Di Wang$^{1,6}$\thanks{Corresponding author.} \\
$^{1}$Provable Responsible AI and Data Analytics  Lab \\
$^{2}$Nankai University \\
$^{3}$Xi'an Jiaotong University \\
$^{4}$University of Massachusetts Boston \\
$^{5}$East China Normal University \\
$^{6}$King Abdullah University of Science and Technology (KAUST) \\
}

\begin{document}
\maketitle
\lhead{Preprint}

\begin{abstract}
Sparse linear regression is a fundamental problem in high-dimensional statistics, and its private counterparts leave a large gap between the central and the local model of differential privacy. Shuffled differential privacy (SDP) sits between the two, offering accuracy closer to the central model without the trust requirement of a central curator. This paper studies sparse linear regression under both non-interactive and interactive SDP protocols. In the non-interactive setting, any non-interactive $(\epsilon,\delta)$-SDP protocol is shown to incur an $\ell_2$-norm estimation error of $\Omega( \frac{\sqrt{dk}}{n\epsilon} )$ with dimension $d$, sparsity $k$, and $n$ users, revealing an inherent dimensionality barrier. In addition, a non-interactive algorithm is proposed that achieves an $\ell_2$-norm upper bound of $\tilde{\mathcal{O}}( \sqrt{\frac{k}{n}} + \frac{d\sqrt{k}}{n\epsilon})$. To overcome this barrier, we turn to the interactive user-level setting, where each user holds $m$ samples. A two-round interactive algorithm is developed that first recovers the support through a heavy hitter step and then estimates the parameter, achieving an error of $\tilde{\mathcal{O}}( \sqrt{\frac{k}{nm}} + \frac{k}{n\sqrt{m}\epsilon} )$ with no polynomial dependence on $d$. Furthermore, by extending the score attack framework to user-level data, a new lower bound $\Omega( \sqrt{\frac{k\log(d/k)}{nm}} + \frac{k\log(d/k)}{n\sqrt{m}\epsilon} )$ is established for any user-level $(\epsilon,\delta)$-DP, showing that the proposed algorithm is near-optimal with only a gap of the factor $\log(d/k)$.
\end{abstract}

\section{Introduction}
Protecting data privacy has become a paramount concern in modern information and database systems, which often store highly sensitive personal records, and (sparse) linear models, a cornerstone of both statistics and machine learning, are routinely fitted to such data. Differential privacy (DP) has emerged as the gold standard for privacy-preserving data analysis~\citep{dwork2006calibrating, dwork2014algorithmic}, providing mathematically verifiable protection that remains resilient to attackers with arbitrary auxiliary information.

The central model of DP assumes the existence of a trusted curator who collects raw data from all individuals and then applies privacy-preserving mechanisms before releasing the results. However, in many distributed settings, users may not trust a central entity with their unencrypted sensitive information. To address this, local differential privacy (LDP) was introduced, shifting the privacy-preserving step to the user's side; each individual perturbs their own data locally before sending it to a server \citep{duchi2013local, duchi_privacy_2014, kasiviswanathan2011what,wang2023generalized,wang2020empirical}. Despite this stronger privacy guarantee, a critical issue with LDP is its limited utility, as the significant amount of noise required to protect each individual entry often leads to much lower statistical accuracy compared to the central model. Shuffled differential privacy (SDP) has recently emerged as a promising middle ground that simultaneously addresses these concerns \citep{cheu2019distributed, erlingsson2019amplification, balle2019privacy}. By introducing an intermediate shuffler that de-identifies and permutes the local messages before they reach the server, SDP can provide accuracy levels approaching those of the central model while maintaining a distributed architecture that does not require a trusted curator.

Private linear regression has been studied both as stochastic optimization, which minimizes the empirical or population risk, and as statistical estimation, which recovers the parameter $\theta^* \in \mathbb{R}^d$ of the linear model $y = \langle x, \theta^* \rangle + \zeta$ under the $\ell_2$-norm error \citep{kifer_private_2012, talwar_nearly_2015, cai2021cost}. We take the estimation view and focus on the high-dimensional sparse model, where $\theta^*$ is $k$-sparse with $k \ll d$ and $d$ may far exceed the sample size $n$~\citep{raskutti2011minimax, yang2014elementary, jain2014iterative}; here the main difficulty under privacy is to avoid an error that scales polynomially with $d$.

In recent years, there have been many studies of the DP sparse linear model under different privacy notations. Specifically, in the central DP model, the $l_2$-norm estimation error for $k$-sparse parameters scales logarithmically with the dimension $d$, achieving a nearly optimal rate of $\mathcal{O}(\sqrt{\frac{k \log d}{n}} + \frac{k \log d}{n\epsilon})$ \citep{cai2021cost}. In sharp contrast, the local DP model imposes a much heavier privacy cost, where the estimation error inevitably suffers from a polynomial dependency on $d$; as summarized in Table~\ref{tab:dp_regression_summary}, the improved analysis of sparse local-private regression shows that non-interactive and interactive LDP algorithms still obey strong dimension-dependent lower bounds \citep{zhu2023improved}. This creates a massive gap that makes LDP algorithms essentially ill-suited for modern high-dimensional scenarios where $d \gg n$. However, on the other side, for SDP, recent studies show that for many problems, such as histogram estimation~\cite{balcer2019separating}, their theoretical behaviors are more similar to the central DP, while for other problems, such as parity function learning~\cite{cheu2021limits}, the theoretical behaviors are more similar to local DP. Given the large gap on the limitations of sparse linear model between the two models, we have a natural question: does sparse linear regression in the shuffled model bridge the gap toward the efficiency of central DP, or does it still reflect the dimensionality barriers of local DP? 

\begin{table}[htbp]
\centering
\caption{Comparison summary of our work with related studies on $(\epsilon, \delta)$-DP linear regression under sub-Gaussian data assumptions. Here $n$, $m$, $k$, and $d$ denote the number of users, samples per user, sparsity, and dimension. Note that we can regard $m=1$ when the privacy unit is \textit{sample}. All bounds are measured by the $\ell_2$-norm estimation error. The $\tilde{\mathcal{O}}$ notation suppresses polylogarithmic factors. The full table can be seen in the Appendix \ref{app:detailed_comparison}.}
\label{tab:dp_regression_summary}
\small
\setlength{\tabcolsep}{2pt}
\setlength{\extrarowheight}{2pt}
\renewcommand{\arraystretch}{1.85}
\begin{tabular}{@{}lllllll@{}}
\hline
\textbf{Model} & \textbf{Type} & \textbf{Unit} & \textbf{Setting} & \textbf{Upper Bound} & \textbf{Lower Bound} & \textbf{Source} \\ \hline
\multirow{2}{*}{Central} & \multirow{2}{*}{-} & sample & general & $\tilde{\mathcal{O}}(\sqrt{\frac{d}{n}}+\frac{d}{n\epsilon})$ & $\Omega(\sqrt{\frac{d}{n}}+\frac{d}{n\epsilon})$ & \cite{cai2021cost} \\ \cline{3-7} 
 &  & sample & k-sparse & $\tilde{\mathcal{O}}(\sqrt{\frac{k}{n}}+\frac{k}{n\epsilon})$ & $\Omega(\sqrt{\frac{k\log d}{n}}+\frac{k\log d}{n\epsilon})$ & \cite{cai2021cost} \\ \hline
\multirow{3}{*}{Local} & Non-interactive & sample & k-sparse & $\tilde{\mathcal{O}}( \frac{d \sqrt{k}}{\sqrt{n}\epsilon})$ & $\Omega(\sqrt{\frac{dk \log d}{n\epsilon^2}})$ & \cite{zhu2023improved} \\ \cline{2-7} 
 & \multirow{2}{*}{Interactive} & sample & k-sparse & $\tilde{\mathcal{O}}(\frac{k\sqrt{d}}{\sqrt{n}\epsilon})$ & $\Omega(\sqrt{\frac{dk}{n\epsilon^2}})$ & \cite{zhu2023improved} \\ \cline{3-7} 
 &  & user-level & k-sparse & $\tilde{\mathcal{O}}( \sqrt{\frac{k}{nm}} + \frac{k}{\sqrt{nm} \epsilon} )$ & $\Omega( \frac{k}{\sqrt{n m}\epsilon} )$ & \cite{ma2024better} \\ \hline
\multirow{2}{*}{Shuffle} & Non-interactive & sample & k-sparse & $\tilde{\mathcal{O}}( \sqrt{\frac{k}{n}} + \frac{d \sqrt{k}}{n \epsilon} )$ & $\Omega( \frac{\sqrt{dk}}{n\epsilon} )$ & \multirow{2}{*}{\textbf{Our Work}} \\ \cline{2-6} 
 & Interactive & user-level & k-sparse & $\tilde{\mathcal{O}}( \sqrt{\frac{k}{n m}} + \frac{k}{n \sqrt{m} \epsilon} )$ & $\Omega( \sqrt{\frac{k \log(d/k)}{n m}} + \frac{k \log(d/k)}{n\sqrt{m}\epsilon} )$ &  \\ \hline
\end{tabular}
\end{table}

In this work, we provide a rigorous view of sparse linear regression under SDP. Our contributions are three-fold.
\begin{itemize}
    \item We first study the non-interactive protocol. By converting robust $(\epsilon,\delta)$-SDP to pan-privacy, we establish a lower bound of $\Omega( \frac{\sqrt{dk}}{n\epsilon} )$ on the estimation error measured by the $\ell_2$-norm difference, revealing an inherent dimensionality barrier, which is more similar to LDP. We also propose an $(\epsilon,\delta)$-SDP algorithm, which achieves an upper bound of $\tilde{\mathcal{O}}( \sqrt{\frac{k}{n}} + \frac{d \sqrt{k}}{n \epsilon} )$, characterizing the limits of non-interactive SDP.
    \item We then study how we can overcome the theoretical barrier with relaxed settings. Specifically, we consider the interactive and user-level SDP protocol, where each user has $m$ samples. We develop a two-round interactive algorithm, U2R (User-level Two-Round Regression), that achieves an upper bound of $\tilde{\mathcal{O}}( \sqrt{\frac{k}{n m}} + \frac{k}{n \sqrt{m} \epsilon} )$. Moreover, we develop a novel user-level score attack method and show a lower bound of $\Omega\Big(\sqrt{\frac{k\log(d/k)}{nm}} + \frac{k\log(d/k)}{n\sqrt{m}\epsilon} \Big)$ for user-level $(\epsilon,\delta)$-DP. The two rates differ only by a factor of $\log(d/k)$, indicating that our upper bound is near-optimal and that, with interaction and user-level data, the problem is more closely related to central DP.
    \item We evaluate U2R on synthetic and real data. Across the privacy budgets $\epsilon$, dimension $d$, and user counts $n$, it has the lowest coefficient error and MSE among the private methods.
\end{itemize}

\section{Related Work}
Private high-dimensional regression has been studied across three main privacy models. We provide a comprehensive review in Appendix \ref{app:detailed_comparison}.

In the central model, sparse estimation typically incurs only a logarithmic dependence on the dimension $d$ \citep{talwar_nearly_2015, cai2021cost}, and the score attack of \citet{cai2023score} gives the lower-bound technique that we extend to the user-level regime. The local model is more restrictive, with both non-interactive and interactive estimators exhibiting polynomial dependence on $d$ \citep{wang2020sparse, zhu2023improved}; user-level multi-sample regimes alleviate these costs \citep{ma2024better}. Positioned between the two, the shuffled model leverages privacy amplification \citep{cheu2019distributed, erlingsson2019amplification, balle2019privacy}, while multi-message protocols compute a sum with central-model accuracy \citep{balle2020private}, which is the summation step both our algorithms use. On the negative side, \citet{cheu2021limits} relate shuffled DP to pan-privacy, the route our non-interactive lower bound follows. Shuffled protocols have been studied for primitives such as heavy hitters \citep{bun2019heavy}, but their application to sparse linear regression remains largely open, which is the gap we fill.

%While LDP and ULDP offer decentralized trust, they still inherently sacrifice utility compared to the central model. The Shuffled Differential Privacy (SDP) model has recently emerged as a powerful intermediate framework, utilizing an anonymous shuffler to amplify privacy guarantees and narrow the utility gap between local and central DP \cite{cheu2021limits}. Despite rapid progress in standard parameter estimation under SDP, the fundamental limits of high-dimensional sparse linear regression in the shuffle model remain largely under-explored. In this paper, we bridge this critical gap by establishing the statistical limits of sparse linear models under SDP......

%\di{For this section, you can just use ICLR'24 and Yuheng's paper. Use GPT to rewrite it is fine. Whole section will be about 0.5-1-page. If there is no space, leave additional related work in Appendix}

\section{Preliminaries}
\label{sec:Preliminaries}

% \di{Notations, one subsection for (DP, SDP, non-interactive SDP), one subsection for problem setup (with necesary assumptions). Follow ICLR'24 is fine. } 

This section introduces the notations, the models, and the problem setting used
throughout the paper.

\noindent\textbf{Notations.}
For an integer $d$, let $[d] = \{1, \dots, d\}$ and let $\mathbb{R}^d$ be the $d$-dimensional
Euclidean space. For a vector $v \in \mathbb{R}^d$, $\|v\|_p$ is the $\ell_p$-norm, and
$\|v\|_0$ is the number of non-zero entries. The support of $\theta$ is
$\mathrm{supp}(\theta) = \{j \in [d] : \theta_j \neq 0\}$, and for a set $S \subseteq [d]$
we write $\theta_S \in \mathbb{R}^{|S|}$ for the restriction of $\theta$ to the coordinates
in $S$. Given a matrix $X \in \mathbb{R}^{n \times d}$, let $x_i^\top$ be its $i$-th row and
$x_{ij}$ (or $[X]_{ij}$) be its $(i,j)$-th entry, which is also the $j$-th element of $x_i$.
We use $\|X\|_2$ for the spectral norm and
$\|X\|_{\infty,\infty} = \max_{i,j} |x_{ij}|$ for the max norm, and $\lambda_{\min}(X)$ for
the minimal eigenvalue of $X$. For an event $A$, $\mathbb{I}[A]$ denotes the indicator, i.e.,
$\mathbb{I}[A] = 1$ if $A$ occurs and $\mathbb{I}[A] = 0$ otherwise. A zero-mean random
variable $X$ is sub-Gaussian if its sub-Gaussian norm
$\|X\|_{\psi_2} = \inf\{c > 0 : \mathbb{E}[\exp(X^2/c^2)] \le 2\}$ is bounded; a random vector
$x$ is sub-Gaussian if $\langle x, u\rangle$ is sub-Gaussian for every unit vector $u$. We use
$O(\cdot)$ and $\Omega(\cdot)$ in the standard sense, and $\widetilde{O}(\cdot)$ hides
logarithmic factors.

\subsection{Differential Privacy}
\label{subsec:dp}
\begin{definition}[Differential Privacy \citep{dwork2006calibrating}]
\label{def:dp}
A randomized algorithm $\mathcal{M}$ with domain $\mathcal{D}^n$ is
$(\epsilon, \delta)$-differentially private (DP) if for all measurable sets
$\mathcal{S} \subseteq \mathrm{Range}(\mathcal{M})$ and for any two neighboring datasets
$D, D' \in \mathcal{D}^n$ that differ in exactly one record, it holds that
\begin{equation}
    \mathbb{P}[\mathcal{M}(D) \in \mathcal{S}] \le e^\epsilon \cdot
    \mathbb{P}[\mathcal{M}(D') \in \mathcal{S}] + \delta .
\end{equation}
\end{definition}

Definition~\ref{def:dp} is stated with respect to a neighboring relation on $\mathcal{D}^n$.
In the sample-level regime a record is a single pair $(x_i, y_i)$, so $D$ and $D'$ differ in
one sample. In the user-level regime a record is the whole local dataset $U_i$ of user $i$,
so $D$ and $D'$ differ in all $m$ samples held by one user. All our user-level theorems use the latter relation.

The central model assumes the curator sees the raw data. The shuffled model removes this
assumption, and instead places a trusted shuffler between the users and the data collector.

\begin{definition}[Local Randomizer, Shuffler, and Analyzer]
\label{def:protocol}
A protocol in the shuffled model is a triple $\mathcal{P} = (\mathcal{R}, \mathcal{S},
\mathcal{A})$, where the local randomizer $\mathcal{R} : \mathcal{D} \to \mathcal{Y}^t$ maps
one record to $t$ messages, the shuffler $\mathcal{S} : \mathcal{Y}^{nt} \to \mathcal{Y}^{nt}$
applies a uniformly random permutation to the messages it receives, and the analyzer
$\mathcal{A} : \mathcal{Y}^{nt} \to \mathcal{O}$ maps the shuffled messages to the output of
the protocol. On a dataset $D = (z_1, \dots, z_n)$, the view of the data collector is
\begin{equation}
    \mathcal{V}_{\mathcal{P}}(D) = \mathcal{S}\big(\mathcal{R}(z_1), \dots,
    \mathcal{R}(z_n)\big),
\end{equation}
and the output of the protocol is
$\mathcal{M}_{\mathcal{P}}(D) = \mathcal{A}(\mathcal{V}_{\mathcal{P}}(D))$.
\end{definition}

In the user-level regime, the record of user $i$ is $U_i \in \mathcal{D}^m$, so the local
randomizer is $\mathcal{R} : \mathcal{D}^m \to \mathcal{Y}^t$ and the view is
$\mathcal{V}_{\mathcal{P}}(D) = \mathcal{S}\big(\mathcal{R}(U_1), \dots, \mathcal{R}(U_n)\big)$.

\begin{definition}[Shuffled Differential Privacy \citep{cheu2019distributed, balle2019privacy}]
\label{def:sdp}
A protocol $\mathcal{P} = (\mathcal{R}, \mathcal{S}, \mathcal{A})$ satisfies
$(\epsilon, \delta)$-shuffled DP (SDP) if the view $\mathcal{V}_{\mathcal{P}}$ satisfies
Definition~\ref{def:dp}, i.e., for any two neighboring datasets $D, D'$ and any measurable
event $\mathcal{E} \subseteq \mathcal{Y}^{nt}$,
\begin{equation}
    \mathbb{P}[\mathcal{V}_{\mathcal{P}}(D) \in \mathcal{E}] \le e^{\epsilon} \cdot
    \mathbb{P}[\mathcal{V}_{\mathcal{P}}(D') \in \mathcal{E}] + \delta .
\end{equation}
\end{definition}

Definition~\ref{def:sdp} implies that $\mathcal{M}_{\mathcal{P}}$ is
$(\epsilon,\delta)$-DP in the sense of Definition~\ref{def:dp}.

\begin{definition}[Non-interactive and Interactive protocol]
\label{def:interaction}
A protocol is non-interactive if every user applies $\mathcal{R}$ to their own record once and
sends the resulting messages in a single round, so that $\mathcal{R}$ does not depend on any
other user's messages. A protocol is interactive if it proceeds in $T > 1$ rounds, where the
randomizer used by a user in round $t$ may depend on the aggregated output released by the
analyzer in rounds $1, \dots, t-1$.
\end{definition}

\subsection{Problem Setup}
\label{subsec:setup}

We consider the setting of sparse linear regression. Let
$\mathcal{D} = \mathcal{X} \times \mathcal{Y} \subseteq \mathbb{R}^d \times \mathbb{R}$ be
the data universe and let there be $n$ users in the population. We distinguish two data
regimes. In the sample-level regime, each user $i \in [n]$ holds one pair
$(x_i, y_i) \in \mathcal{X} \times \mathcal{Y}$. In the user-level regime, each user $i$
holds $m$ i.i.d. samples $U_i = \{(x_{i,j}, y_{i,j})\}_{j=1}^{m}$, which we also write in
matrix form as $(X_i, y_i) \in \mathbb{R}^{m \times d} \times \mathbb{R}^{m}$; the total
number of samples is then $N = nm$.

%The two assumptions below are used throughout the paper. The conditions needed by one section only are stated in that section.

\begin{assumption}[Sparse Linear Model]
\label{asm:model}
The samples are i.i.d. realizations of the sparse linear regression model
$y = \langle \theta^*, x \rangle + \zeta$, where the distribution of $x$ has mean zero, the
noise $\zeta$ satisfies $\mathbb{E}[\zeta \mid x] = 0$ and is sub-Gaussian with variance proxy $\sigma^2$, and the underlying parameter lies in
$\Theta_k = \{\theta \in \mathbb{R}^d : \|\theta\|_0 \le k\}$ with $k \le d$ and
$\|\theta^*\|_2 = \mathcal{O}(1)$. In the user-level regime, the samples are i.i.d. both
within and across users.
\end{assumption}

\begin{assumption}[Covariance Structure]
\label{asm:cov}
Let $\Sigma = \mathbb{E}[x x^\top]$ be the population covariance matrix. There is a constant
$\kappa_x > 0$ such that $\|\Sigma^{-1/2} x\|_{\psi_2} \le \kappa_x$.
\end{assumption}

\section{Understanding the Limitations in Non-interactive SDP}
\label{sec:non SDP}

In this section, we characterize sparse linear regression under non-interactive shuffled DP. We
first establish a minimax lower bound on the $\ell_2$-norm estimation error, and then give an
algorithm whose bound is decent.

Throughout this section the parameter space is
$\widetilde{\Theta}_k = \Theta_k \cap \{\theta \in \mathbb{R}^d : \|\theta\|_1 \le 1\}$. The
$\ell_1$ normalization strengthens the $\ell_2$ bound of Assumption~\ref{asm:model} and is what
keeps the sensitivity of a single record under control once each user privatizes their own
statistics; it is not needed in Section~\ref{sec:user-level SDP}. We also use the following
conditions, which are used in this section only.

\begin{assumption}[Curvature and Boundedness]
\label{asm:curv_bounded}
There is a constant $\kappa_\infty > 0$ such that
$\|\Sigma w\|_\infty \ge \kappa_\infty \|w\|_\infty$ for any $w \neq 0$, and there are
constants $c_x, c_y = O(1)$ such that $\|x\|_\infty \le c_x$ and $|y| \le c_y$ almost surely.
\end{assumption}

The first condition is the $\ell_\infty$ curvature used by the soft-thresholding step, and the
second lets each record be clipped at a known scale before it is privatized.

\subsection{Lower Bound}
\label{subsec:non_lower}

We begin with the fundamental limits of the non-interactive setting. The standard route to a
minimax lower bound under local DP does not apply here. In the local model each message depends
on a single record, so the KL divergence tensorizes,
$D_{\mathrm{KL}}(P^n \| U^n) = n \cdot D_{\mathrm{KL}}(P \| U)$, and private Fano's inequality or
a direct mutual information bound can be used. In the shuffled model the messages of the
$n$ users are permuted before the analyzer sees them, and the mutual information between
the input and the view $\mathcal{V}_{\mathcal{P}}$ can be unbounded, so these arguments do not
carry over. To solve the problem, we introduce pan-privacy, which constrains the internal state
of an algorithm rather than the messages of individual users.

\begin{definition}[Pan-Privacy \citep{cheu2021limits}]
\label{def:pan_privacy}
An algorithm $M$ processes a stream $Z = (z_1, \dots, z_n)$ one record at a time: it
updates an internal state from $z_i$ and the previous state, and computes an output from the
final state. Write $M_I(Z_{\le t})$ for the internal state after the first $t$ records and
$M_O(M_I(Z))$ for the output. Then $M$ is $(\epsilon, \delta)$-pan-private if for every pair of
streams $Z, Z'$ differing in one record, every $t \in [n]$, and every measurable set $T$ of
state-output pairs,
\begin{equation}
    \mathbb{P}\big[\big(M_I(Z_{\le t}), M_O(M_I(Z))\big) \in T\big] \le
    e^{\epsilon} \cdot \mathbb{P}\big[\big(M_I(Z'_{\le t}), M_O(M_I(Z'))\big) \in T\big]
    + \delta .
\end{equation}
\end{definition}
Pan-privacy protects the record against an adversary that observes the output together with the
internal state at one point in the stream. \citet{cheu2021limits} showed that any algorithm
satisfying robust $(\epsilon, \delta)$-shuffled DP can be simulated by a
$(\epsilon, \delta)$-pan-private algorithm on a constant fraction of the users; here, robustness
means that the protocol remains $(\epsilon, \delta)$-shuffled DP as long as a constant fraction
of the users follow it, a condition met by all protocols considered in this work. Therefore, to obtain the utility lower bound of SDP algorithms, it is enough to bound the error of pan-private algorithms.

The lower bound is proved on the following family of hard instances, which specifies a
distribution over the sparse parameter $\theta^*$ and over the data $(x, y)$.

\begin{definition}[Hard Distribution]
\label{def:hard_distribution}
Let $\nu$ be the target $\ell_2$ error and let $\gamma = \nu / \sqrt{k}$. Define
$\theta_Z = \gamma Z$ with $Z \in \{+1, 0, -1\}^d$, where each coordinate $Z_j$ is drawn
independently according to
\begin{equation}
    \mathbb{P}\{Z_j = +1\} = \frac{k}{4d}, \quad
    \mathbb{P}\{Z_j = -1\} = \frac{k}{4d}, \quad
    \mathbb{P}\{Z_j = 0\} = 1 - \frac{k}{2d}.
\end{equation}
Given $\theta_Z$, a sample $(x, y) \in \{-1, 1\}^d \times \{-1, 1\}$ is generated by drawing $x$
uniformly from the hypercube, so that $\mathbb{P}(x) = 2^{-d}$, and setting
$y = \langle x, \theta_Z \rangle + \zeta$ with
\begin{equation}
    \mathbb{P}\{\zeta = 1 - \langle x, \theta_Z \rangle\}
    = \frac{1 + \langle x, \theta_Z \rangle}{2}, \quad
    \mathbb{P}\{\zeta = -1 - \langle x, \theta_Z \rangle\}
    = \frac{1 - \langle x, \theta_Z \rangle}{2}.
\end{equation}
\end{definition}

The expected sparsity of $\theta_Z$ is $k/2$, and Markov's inequality gives $\|\theta_Z\|_0 \le k$ with probability at least $1/2$; under $\nu \le 1/\sqrt{k}$ we also have
$\|\theta_Z\|_1 \le 1$. The construction gives $y \in \{-1, 1\}$ and $\mathbb{E}[\zeta \mid x] = 0$, and since $|\langle x, \theta_Z \rangle| \le 1$, the noise
satisfies $|\zeta| \le 2$ almost surely, so the instances lie in the model class of
Assumptions~\ref{asm:model}--\ref{asm:curv_bounded} with parameter in $\widetilde{\Theta}_k$.

In Appendix~\ref{app:non_lower} we bound the sum over the $d$ coordinates of the total variation distance between the two conditional distributions of the output given $Z_i = +1$ and $Z_i = -1$, which gives the following result.

\begin{theorem}
\label{thm:shuffled_dp_lower}
Let $\mathcal{M}_{\epsilon, \delta}^{\mathrm{SDP}}$ be the family of all
non-interactive protocols that satisfy robust $(\epsilon, \delta)$-shuffled DP with
$0 < \epsilon \le \mathcal{O}(1)$ and $\delta = o(1/n)$. Let $Z = \{(x_i, y_i)\}_{i=1}^n$ consist
of $n$ i.i.d. samples from the sparse linear model of Assumption~\ref{asm:model} with parameter
$\theta^*$. Then we have 
\begin{equation}
    \inf_{\mathcal{A} \in \mathcal{M}_{\epsilon, \delta}^{\mathrm{SDP}}}
    \sup_{\theta^* \in \widetilde{\Theta}_k}
    \mathbb{E}\big[\|\mathcal{A}(Z) - \theta^*\|_2\big] \ge \Omega( \frac{\sqrt{dk}}{n \epsilon}),
\end{equation}
where the expectation is over the data and the randomness of the protocol.
\end{theorem}

Central DP admits the rate
$\Omega\big(\sqrt{k \log d / n} + k \log d / (n\epsilon)\big)$, while non-interactive local DP
obeys $\Omega\big(\sqrt{dk \log d / (n\epsilon^2)}\big)$ \citep{cai2021cost, zhu2023improved}.
Theorem~\ref{thm:shuffled_dp_lower} places non-interactive SDP between the two. Written as
$\sqrt{dk}/(n\epsilon)$ against $\sqrt{dk \log d}/(\sqrt{n}\epsilon)$, the shuffled bound is
smaller than the local one by a factor $\sqrt{n \log d}$, which is the gain from shuffling. The
factor $\sqrt{d}$, however, is unchanged. This is the difference from the central model, where
the privacy noise is added once to an aggregate statistic and the dimension enters only through
the $\log d$ cost of locating the support; in the local and shuffled models each user randomizes
their own record, and shuffling rescales the resulting per-coordinate noise without removing it.
Hence the error of any non-interactive SDP protocol carries a polynomial dependence on $d$, whatever the local randomizer or analyzer, which is the primary bottleneck when $d \gg n$. We next give an algorithm that achieves a nontrivial rate.

\subsection{Upper Bound}
\label{subsec:non_upper}

The sparse estimation is actually an $\ell_0$-constrained least squares problem, and this is the fundamental challenge when proposing an efficient non-interactive SDP algorithm. To control local noise scales, we therefore relax the $\ell_0$ constraint
to an $\ell_1$-regularized problem, whose solution can be written in closed form under Assumption~\ref{asm:curv_bounded}. While this relaxation will link the error to the dimension $d$, the shuffler narrows this gap toward central DP. Then we can obtain the solution from the aggregated statistics by solving the normal equation and applying soft-thresholding.

Our non-interactive SDP algorithm is given as Algorithm~\ref{alg:shuffled_sparse_regression} in
Appendix~\ref{app:algorithms}. We assume the length $D = d(d+1)/2 + d = \Theta(d^2)$.

\begin{theorem}
\label{thm:non_interactive_upper}
Let $\theta^* \in \widetilde{\Theta}_k$ and Assumptions~\ref{asm:model}--\ref{asm:curv_bounded} hold,
with $0 < \epsilon \le \mathcal{O}(1)$ and the parameters of
Algorithm~\ref{alg:shuffled_sparse_regression} set as in Appendix~\ref{app:algorithms}. If
\begin{equation}
\label{eq:thm2_condition}
    n \ge \Omega\left( \max\left\{ \frac{d^2 \log n \sqrt{\log(d/\delta) \log d}}{\epsilon \kappa_\infty}, \;
    \frac{d^2 \log d}{\kappa_\infty^2}, \;
    \frac{D \log(D/\delta)}{\epsilon^2} \right\} \right) ,
\end{equation}
then Algorithm~\ref{alg:shuffled_sparse_regression} satisfies  $(\epsilon, \delta)$-shuffled DP, and with probability at least $1 - \mathcal{O}(d^{-2})$ its output satisfies
\begin{equation}
\|\hat{\theta} - \theta^*\|_2 \le \widetilde{\mathcal{O}}\left( \sqrt{\frac{k}{n}} + \frac{d \sqrt{k}}{n \epsilon} \right).\end{equation}
\end{theorem}

Privacy follows in two steps. The protocol that receives the rounded statistics directly and perturbs their sum with the binomial noise $\mathrm{Bin}(M, 1/2) - M/2$ is $(\epsilon, \delta/2)$-DP. By the security of the summation protocol, the aggregator's view of the real protocol is within total variation distance $e^{-\epsilon}\delta/8$ of this one, and Lemma~F.2 of \citet{balle2020private} turns this distance into the stated $(\epsilon, \delta)$-shuffled DP guarantee. For the utility, the rounding error is of the same order as the binomial noise; we bound both in the $\ell_\infty$-norm and then apply the soft-thresholding bound. The resulting privacy term $d\sqrt{k}/(n\epsilon)$ leaves a gap of $\sqrt{d}$ to the lower bound $\Omega(\sqrt{dk}/(n\epsilon))$ of Theorem~\ref{thm:shuffled_dp_lower}. The gap comes from the clipping radius: the second-moment statistic has $\|x_i x_i^\top\|_F = \|x_i\|_2^2 = \Theta(d)$, and to privatize it, this scale enters the noise directly. We leave closing this gap to future work; in either case, the polynomial dependence on $d$ is not removed, which is the main barrier of the non-interactive setting.

\section{Overcoming the Theoretical Barrier}
\label{sec:user-level SDP}
Section~\ref{sec:non SDP} shows that a polynomial dependence on $d$ is unavoidable in the
non-interactive setting. In this section we show that the barrier can be overcome when each user holds multiple samples and two rounds of interaction are allowed. Specifically, we give a two-round user-level SDP algorithm whose estimation
error is independent of the dimension $d$ once the local sample size $m$ is large
enough. In this section, we consider the following additional assumption. 

\begin{assumption}[Restricted Eigenvalue]
\label{asm:re}
There is a constant $\kappa > 0$ such that $v^\top \Sigma v \ge \kappa \|v\|_2^2$ for every $S \subseteq [d]$ with $|S| \le k$, and every $v$ with $\|v_{S^c}\|_1 \le 3 \|v_S\|_1$, and there is a constant $C$ such that $\max_{\ell \in [d]} \Sigma_{\ell\ell} \le 1$ and $\lambda_{\max}(\Sigma_{S,S}) \le C$ for every $|S| \le k$.
\end{assumption}
This is the restricted eigenvalue condition of \citet{bickel2009simultaneous}, the standard
condition under which the Lasso attains its fast rate, and the cone constant $3$ is the one
produced by the basic inequality when the regularization parameter dominates the score
\citep{vandegeer2009conditions, negahban2012unified}. It holds with high probability for
sub-Gaussian designs whose population covariance has a bounded minimal eigenvalue
\citep{raskutti2010restricted}, and conditions of this form are standard in the DP sparse
regression literature \citep{cai2021cost, ma2024better, tian2025differentially}. Taking $v$
supported on $S$ gives $\lambda_{\min}(\Sigma_S) \ge \kappa$.

\subsection{Upper Bound}
\label{subsec:user_upper}

Instead of relaxing the $\ell_0$ constraint as in Section~\ref{subsec:non_upper}, we reduce the problem to a low-dimensional regression by first estimating the support of $\theta^*$. In the user-level regime, each user turns their $m$ local samples into one local estimate of an active coordinate, and a heavy hitter step aggregates these estimates into the global support set \citep{bun2019heavy, ma2024better}.

\begin{algorithm}[htbp]
\caption{User-level Two-Round Shuffled DP Sparse Regression (U2R)}
\label{alg:pipeline}
\begin{algorithmic}[1]
    \STATE {\bfseries Input:} Local datasets $\{(X_i, y_i)\}_{i=1}^{n}$, sparsity $k$, privacy parameters
    $(\epsilon, \delta)$.
    \STATE {\bfseries Stage 1: Support Recovery} \quad (users $i \in \{1, \dots, n_1\}$)
    \FOR{each user $i \in \{1, \dots, n_1\}$}
        \STATE Solve the Lasso, get a local candidate set $\hat{S}_i$, and draw one index
        $v_i \sim \mathrm{Unif}(\hat{S}_i)$.
        \STATE Send the randomized response of $v_i$ on each coordinate $\ell \in [d]$ through
        $d$ parallel shufflers.
    \ENDFOR
    \STATE {\bfseries Aggregator:} decode the shuffled
    messages to get the counts $\{\hat{W}_\ell\}_{\ell \in [d]}$, recover support $\hat{S}$ via heavy hitter identification, and release $\hat{S}$, where $s = |\hat{S}| $.
    \STATE {\bfseries Stage 2: Constrained Estimation} \quad (the remaining $n_2 = n - n_1$ users)
    \FOR{each user $i \in \{n_1 + 1, \dots, n\}$}
        \STATE Compute the local least squares estimate $\tilde{\theta}_i \in \mathbb{R}^{s}$ on
        $\hat{S}$, clip it, and send it to the shuffler.
    \ENDFOR
    \STATE {\bfseries Aggregator:} decode the messages of the first half users into a coarse estimate and
    broadcast it;
    the second half messages clip around this estimate, and yield the refined estimate $\hat{\theta}_{\hat{S}} \in \mathbb{R}^{s}$.
    \STATE {\bfseries Output:} the $k$-sparse estimate $\hat{\theta}\in \mathbb{R}^{d}$ that equals $\hat{\theta}_{\hat{S}}$
    on $\hat{S}$ and zero elsewhere.
\end{algorithmic}
\end{algorithm}
We summarize the pipeline in Algorithm \ref{alg:pipeline}. In this algorithm, \textbf{Stage 1: Support Recovery}, \textbf{Stage 2: Constrained Estimation}, and their internal hyperparameters (e.g., clipping thresholds) can be found in Algorithm \ref{alg:heavyhitter} and Algorithm \ref{alg:sparse_estimation_1} in the Appendix \ref{app:algorithms}. Algorithm~\ref{alg:pipeline} is two-round because we count rounds by the messages that users send.

The heavy hitter step of Stage 1 (Algorithm \ref{alg:heavyhitter}) is what keeps the dimension out of the final rate, so we describe it
in more detail. The heavy hitter algorithms used for the LDP model~\citep{bun2019heavy, ma2024better}
estimate the frequency of every coordinate through a prefix tree: the users are split into $\lceil \log d \rceil$ groups, with the users of one group reporting a hashed bit of one level of the tree, and the server scans the levels in order. Two costs come with this. Every report has to be private on its own, so the randomization a user applies is calibrated to their own message, and each level of the
tree is estimated from a $1/\log d$ fraction of the users. In the shuffled model, neither cost is
needed: the shuffler removes the association between a message and its sender, so the aggregator sees
nothing beyond the counts, and each user can answer at a randomization level calibrated to the view
$\mathcal{V}_{\mathcal{P}}$ rather than to a single message. We therefore replace the prefix tree by
one round of randomized response over $[d]$. User $i$ reduces their $m$ local samples to a single
index $v_i$, sends the indicator of $v_i$ with every coordinate flipped independently with probability
$\gamma$, and the aggregator debiases the shuffled counts and keeps the coordinates that pass a
threshold. To ensure $\hat{S}$ is the true support, we give Proposition \ref{prop:support_consistency}.

\begin{proposition}
\label{prop:support_consistency}
Let Assumptions~\ref{asm:model}, \ref{asm:cov} and \ref{asm:re} hold, let $S^* = \operatorname{supp}(\theta^*)$ with $|S^*| = k$, and let the parameters of Algorithm~\ref{alg:heavyhitter} be set as in Appendix~\ref{app:algorithms}. Suppose there is $\rho > 0$ with
$\min_{\ell \in S^*} |\theta^*_\ell| \ge \rho$, and that
\begin{enumerate}
    \item[(i)] \textbf{Local sample size.}
    \begin{equation}
    \label{eq:m_condition}
        m \ge \Omega\left( \max\left\{ \frac{\sigma^2 k \log d}{\kappa^2 \rho^2}, \;
        \frac{k \log(ed/k)}{\kappa^2}, \; k + \log n \right\} \right) .
    \end{equation}
    \item[(ii)] \textbf{User number of Stage 1.}
    \begin{equation}
    \label{eq:n1_condition}
        n_1 \ge \Omega\left( \max\left\{ \frac{k \sqrt{\log(dn) \log(1/\delta)}}{\epsilon}, \; k \log(dn) \right\} \right) .
    \end{equation}
\end{enumerate}
Then with probability at least $1 - \mathcal{O}(n^{-2})$,
\begin{equation}
    \hat{S} = S^* ,
\end{equation}
so that $s = |\hat{S}| = k$.
\end{proposition}

Proposition~\ref{prop:support_consistency} gives more than the containment $S^* \subseteq \hat{S}$ that Stage 2 needs: although the local set $\hat{S}_i$ of a user may differ from $S^*$, the heavy hitter step returns $S^*$ exactly. In \eqref{eq:m_condition}, the first two terms make the local selection of one user in Stage 1 correct with probability at least $1 - 1/(32k)$, independently of $n$, and the third term is what Stage 2 needs for the local least squares estimate on $\hat{S}$. The first term is equivalently the beta-min condition $\rho \ge \Omega\big((\sigma/\kappa)\sqrt{k\log d/m}\big)$, which is standard for support recovery \citep{wainwright2009sharp, ma2024better} and allows $\rho$ to decrease as $m$ grows.

\begin{theorem}
\label{thm:interactive_user_level_sdp}
Let $\theta^* \in \Theta_k$, let the conditions of Proposition~\ref{prop:support_consistency} hold,
and let the parameters of Algorithm~\ref{alg:pipeline} be set as in
Appendix~\ref{app:algorithms}. If
\begin{equation}
\label{eq:thm3_condition}
    n \epsilon \ge \Omega\left( \frac{\sqrt{m \log(k/\delta)}}{\sigma} \right) ,
\end{equation}
then Algorithm~\ref{alg:pipeline} satisfies user-level $(\epsilon, \delta)$-shuffled DP, and with
probability at least $1 - \mathcal{O}(n^{-2})$ its output satisfies
\begin{equation}
    \|\hat{\theta} - \theta^*\|_2 \le \widetilde{\mathcal{O}}\left(
    \sqrt{\frac{k}{n m}} + \frac{k}{n \sqrt{m} \epsilon} \right). 
\end{equation}
\end{theorem}

The dimension $d$ is absent from both terms. By Proposition~\ref{prop:support_consistency} the second stage runs on the $k$ coordinates of $\hat{S}$, so the vector a user sends has $k$ entries rather than
$d$, and the dimension is eliminated. What remains of $d$ is in the conditions on $m$ and $n_1$ in Proposition~\ref{prop:support_consistency}. In the synthetic dimension sweep, U2R keeps a low mean relative coefficient error as $d$ increases at fixed $k$ and $m$, as shown in Table~\ref{tab:exp_dimension_10seeds}. User-level data alone does not suffice: Appendix~\ref{app:discussion_user_level_sdp} shows that under a non-interactive protocol the lower bound for user-level SDP still carries a polynomial dependence on $d$, so interaction is also essential.

\subsection{Lower Bound}
\label{subsec:user_lower}

We have already shown a promising result in Theorem~\ref{thm:interactive_user_level_sdp}. However, a natural question is whether our upper bound can be further improved. Now we show that our upper bound is already near-optimal.

\begin{theorem}
\label{thm:user_level_sparse_lower_bound}
Let $Z = \{U_i\}_{i=1}^n$ be a
dataset in which each user $i \in [n]$ holds $m$ i.i.d. samples
$U_i = \{(x_{i,j}, y_{i,j})\}_{j=1}^m$ from the exponential family
\begin{equation}
    f_\theta(y \mid x) = h(y, \sigma) \exp\left(
    \frac{y x^\top \theta - \psi(x^\top \theta)}{c(\sigma)} \right) .
\end{equation}
Let Assumptions~\ref{asm:cov} and \ref{asm:re} hold and let $\|\psi''\|_\infty \le c_2$. Let
$\mathcal{M}_{\epsilon, \delta}$ be the family of all estimators satisfying user-level
$(\epsilon, \delta)$-DP with $0 < \epsilon < 1$ and
$\delta = o\big( \log(d/k) / (n \sqrt{m \log(nm)}) \big)$. If $k \le \sqrt{d}$, $k \ge \Omega(\log^2(d/k))$ and
$n \sqrt{m} \, \epsilon \ge \Omega\big( k^{3/4} \log^{3/2}(d/k) \big)$, then 
\begin{equation}
    \inf_{\mathcal{A} \in \mathcal{M}_{\epsilon, \delta}} \sup_{\theta^* \in \Theta_k}
    \mathbb{E}\big[\|\mathcal{A}(Z) - \theta^*\|_2\big] \ge \Omega  \left(
    \sqrt{\frac{k \log(d/k)}{nm}} + \frac{k \log(d/k)}{n \sqrt{m} \epsilon} \right) ,
\end{equation}
where the expectation is over the data and the randomness of the estimator. 
\end{theorem}

Theorem \ref{thm:user_level_sparse_lower_bound} defines the fundamental efficiency limits for any differential private estimator in the user-level setting. It is stated for the exponential family, which contains
the sparse linear model of Assumption~\ref{asm:model} as the special case $\psi(u) = u^2 / 2$ and
$c(\sigma) = \sigma^2$, so the bound applies to the model class we estimate. By Definition~\ref{def:sdp}, the output of any $(\epsilon, \delta)$-shuffled DP protocol is $(\epsilon, \delta)$-DP in the sense
of Definition~\ref{def:dp}, so Theorem~\ref{thm:user_level_sparse_lower_bound} also lower bounds the error of any user-level SDP protocol.

Comparing it with Theorem~\ref{thm:interactive_user_level_sdp}, the two rates agree in $n$, $m$, $k$ and $\epsilon$ and only differ by the factor $\log(d/k)$: our algorithm is near-optimal, and sparse linear regression in the
shuffled model is closer to central DP once interaction and enough user-level data are
available.

\noindent\textbf{User-level Score Attack.}
We now outline the proof. Minimax lower bounds under privacy constraints are difficult to
characterize: Fano's inequality and mutual information bounds give no general recipe and are often loose, particularly in the user-level regime. To handle user-level data, we extend the score attack framework of \citet{cai2023score} and introduce the
user-level score attack, which is our main technical tool here.

\begin{definition}[User-level Score Attack]
\label{def:user_score_attack}
For a dataset $Z$ in which each user $i \in [n]$ holds $m$ i.i.d. samples
$U_i = \{z_{i,1}, \dots, z_{i,m}\}$, the user-level score attack for an estimator $\mathcal{A}(Z)$ is
\begin{equation}
    A_i := \left\langle \mathcal{A}(Z) - \theta, \; \sum_{j=1}^m S_\theta(z_{i,j}) \right\rangle ,
\end{equation}
where $S_\theta(z) = \nabla_\theta \log f_\theta(z)$ is the score statistic.
\end{definition}

For the exponential family of Theorem~\ref{thm:user_level_sparse_lower_bound}, with mean function
$\mu(x^\top \theta) = \psi'(x^\top \theta)$, the score statistic is
$S_\theta(y, x) = (y - \mu(x^\top \theta)) x / c(\sigma)$, so the user-level score attack takes the
form
\begin{equation}
    A_i = \frac{1}{c(\sigma)} \left\langle \mathcal{A}(Z) - \theta, \;
    \sum_{j=1}^m \big(y_{i,j} - \mu(x_{i,j}^\top \theta)\big) x_{i,j} \right\rangle .
\end{equation}

\begin{proposition}
\label{prop:user_level_score_attack_properties}
Let $Z = \{U_1, \dots, U_n\}$ be as in Theorem~\ref{thm:user_level_sparse_lower_bound}, let
$\mathcal{A}(Z)$ be an estimator of $\theta$, and let $Z_{(i)}'$ be the adjacent dataset in which the
$i$-th user's data $U_i$ is replaced by an i.i.d. copy $\widetilde{U}_i$. Write
$A_i = \langle \mathcal{A}(Z) - \theta, S_\theta(U_i) \rangle$ and
$A_i' = \langle \mathcal{A}(Z_{(i)}') - \theta, S_\theta(U_i) \rangle$. Then
\begin{enumerate}
    \item[(i)] \textbf{Soundness:} for each $i \in [n]$, $\mathbb{E}[A_i'] = 0$ and
    \begin{equation}
        \mathbb{E}|A_i'| \le \sqrt{\mathbb{E}\|\mathcal{A}(Z) - \theta\|_2^2} \cdot
        \sqrt{\frac{m}{c(\sigma)}
        \lambda_{\max}\big(\mathbb{E}[\psi''(x^\top \theta) x x^\top]\big)} ;
    \end{equation}
    \item[(ii)] \textbf{Completeness:} under standard regularity conditions,
    \begin{equation}
        \sum_{i=1}^n \mathbb{E}[A_i]
        = \sum_{j=1}^d \frac{\partial}{\partial \theta_j} \mathbb{E}[\mathcal{A}(Z)_j] .
    \end{equation}
\end{enumerate}
\end{proposition}

With Proposition~\ref{prop:user_level_score_attack_properties}, we bound the expected sum
$\sum_{i=1}^n \mathbb{E}[A_i]$ from above by Lemma~\ref{lem:sparse_upper_bound} and from below by Lemma~\ref{lem:sparse_lower_bound}. Comparing the two gives Theorem~\ref{thm:user_level_sparse_lower_bound}.

\noindent\textbf{Numerical results.}
On synthetic data, U2R attains mean relative coefficient error $0.00631$ at the default setting, comparable to non-private Lasso.
On real-world dataset Kelmarsh, it has the lowest mean test MSE among the private methods
at both $\epsilon=1$ and $4$ for each $m\in\{10,50,100\}$; Appendix~\ref{app:experiments} reports the detailed setup and complete results.

%\textcolor{purple}{We evaluate Algorithm~\ref{alg:pipeline} against user-level and sample-level private baselines on synthetic and real data, varying $\epsilon$, $n$, $m$ and $d$. Appendix~\ref{app:experiments} gives the setup and the results, including the sample-level non-interactive estimator of Section~\ref{sec:non SDP}.}

\section{Conclusion}
We studied sparse linear regression under shuffled DP. In the non-interactive setting, we
showed a lower bound on the $\ell_2$-norm estimation error of any $(\epsilon,\delta)$-SDP
protocol and gave an algorithm with a matching form, both bounds carrying a polynomial
dependence on $d$. In the interactive user-level
setting, we gave a two-round algorithm that first applies a heavy hitter step and then estimates the parameter, whose error does not depend polynomially on $d$. By extending the score attack to user-level data, we proved a lower bound for any user-level $(\epsilon,\delta)$-DP estimator, which shows that this algorithm is near-optimal. Experiments on synthetic and real-world datasets support these results.

\subsection*{AI use statement}
In this work, we used generative AI tools as coding assistants to implement the experimental algorithms. We have not used generative AI tools for drafting sections of the paper, generating synthetic datasets, or proving mathematical claims. Additionally, we used generative AI tools to revising wording and for organizing the bibliography. We have reviewed all AI-assisted work: the authors independently examined and verified all experimental scripts and mathematical derivations. We take full responsibility for the final content of this work, including text or artifacts produced with the aid of generative AI.

\subsection*{Ethics statement}
Our experiments use the publicly available Kelmarsh wind-farm dataset, released by Cubico Sustainable Investments Ltd under the Creative Commons Attribution 4.0 International (CC BY 4.0) license. We use the dataset in accordance with its license terms and provide appropriate attribution. The experiments analyze wind-turbine operational measurements and involve neither human participants nor personally identifiable information.

\subsection*{Reproducibility statement}
We have made every effort to ensure that our results can be reproduced from the paper and its appendix. For the theoretical results, the problem setting and assumptions are stated explicitly in Section~\ref{subsec:setup} (Assumptions~\ref{asm:model} and \ref{asm:cov}). The additional conditions used in each regime are given in Assumption~\ref{asm:curv_bounded} (Section~\ref{sec:non SDP}) and Assumption~\ref{asm:re} (Section~\ref{sec:user-level SDP}). Complete proofs of all claims are provided in the appendix. Theorems~\ref{thm:shuffled_dp_lower} and \ref{thm:non_interactive_upper} are proved in Appendix~\ref{app:proofs_non_interactive}. Proposition~\ref{prop:support_consistency}, Theorem~\ref{thm:interactive_user_level_sdp}, Proposition~\ref{prop:user_level_score_attack_properties} and Theorem~\ref{thm:user_level_sparse_lower_bound} are proved in Appendix~\ref{app:proofs_user_level}. The supplementary results on non-interactive user-level SDP and dense regression, together with their proofs, are in Appendix~\ref{app:additional}. For the algorithms, Appendix~\ref{app:algorithms} gives full pseudocode for the non-interactive protocol (Algorithm~\ref{alg:shuffled_sparse_regression}) and for both stages of U2R (Algorithms~\ref{alg:heavyhitter} and \ref{alg:sparse_estimation_1}). Tables~\ref{tab:param_non_interactive}--\ref{tab:param_stage2} list every parameter required by the theorems, including clipping radii, rounding precisions, binomial noise parameters, moduli, numbers of shares, and thresholds. For the experiments, Appendix~\ref{app:exp_setup} describes the synthetic data-generating process, the default values and sweep ranges, the number of repetitions, and the evaluation metrics. It also gives the hyperparameters of U2R and of all baselines, and how they were selected. The real-data experiments use the publicly available Kelmarsh wind farm dataset (CC BY 4.0). The same appendix details the full preprocessing pipeline, including feature construction, the date-based public/training/test split, imputation, standardization, and the assignment of rows to simulated users. We do not release source code with this submission. We believe these descriptions are sufficient to reimplement the algorithms and reproduce the reported results.

\bibliography{references}
\bibliographystyle{iclr2027_conference}

\appendix
\input{appendix}

\end{document}

%% file: math_commands.tex
\def\eqref#1{equation~\ref{#1}}
\def\1{\bm{1}}

\DeclareMathAlphabet{\mathsfit}{\encodingdefault}{\sfdefault}{m}{sl}
\SetMathAlphabet{\mathsfit}{bold}{\encodingdefault}{\sfdefault}{bx}{n}

%% file: appendix.tex
\appendix

\section{Detailed Comparison and Related Work}
\label{app:detailed_comparison}

Table \ref{tab:dp_regression_comparison} provides a comprehensive comparison of our theoretical results with prior literature on $(\epsilon, \delta)$-DP linear regression. This extends the summary provided in the main text by including bounds for specific edge cases (such as 1-sparse models) and detailing previous interactive local methods.

\begin{table}[htbp]
\centering
\caption{Detailed comparison of related studies on $(\epsilon, \delta)$-DP linear regression. Data assumptions and parameter definitions are identical to Table \ref{tab:dp_regression_summary}.}
\label{tab:dp_regression_comparison}
\footnotesize
\setlength{\tabcolsep}{3pt}
\renewcommand{\arraystretch}{1.8}
\begin{tabularx}{\linewidth}{@{}l l l >{\raggedright\arraybackslash\hsize=0.85\hsize}X >{\raggedright\arraybackslash\hsize=1.35\hsize}X >{\raggedright\arraybackslash\hsize=0.8\hsize}X@{}}
\hline
\textbf{Type} & \textbf{Unit} & \textbf{Setting} & \textbf{Upper Bound} & \textbf{Lower Bound} & \textbf{Source} \\ 
\hline\hline
\multicolumn{6}{@{}l}{\textbf{Central DP}} \\ \hline
\multirow{2}{*}{-} & sample & general & \mbox{$\tilde{\mathcal{O}}(\sqrt{\frac{d}{n}}+\frac{d}{n\epsilon})$} & \mbox{$\Omega(\sqrt{\frac{d}{n}}+\frac{d}{n\epsilon})$} & \citet{cai2021cost} \\ \cline{2-6} 
 & sample & k-sparse & \mbox{$\tilde{\mathcal{O}}(\sqrt{\frac{k}{n}}+\frac{k}{n\epsilon})$} & \mbox{$\Omega(\sqrt{\frac{k\log d}{n}}+\frac{k\log d}{n\epsilon})$} & \citet{cai2021cost} \\ 
\hline\hline
\multicolumn{6}{@{}l}{\textbf{Local DP}} \\ \hline
\multirow{2}{*}{Non-interactive} & sample & 1-sparse & - & \mbox{$\Omega(\sqrt{\frac{d \log d}{n\epsilon^2}})$} & \citet{wang2020sparse} \\ \cline{2-6} 
 & sample & k-sparse & \mbox{$\tilde{\mathcal{O}}( \frac{d \sqrt{k}}{\sqrt{n}\epsilon})$} & \mbox{$\Omega(\sqrt{\frac{dk \log d}{n\epsilon^2}})$} & \citet{zhu2023improved} \\ \cline{1-6} 
\multirow{4}{*}{Interactive} & sample & general & \mbox{$\tilde{\mathcal{O}}(\frac{{d^{\frac{3}{2}}}}{\sqrt{n}\epsilon})$} & - & \citet{smith_is_2017} \\ \cline{2-6} 
 & sample & 1-sparse & - & \mbox{$\Omega(\sqrt{\frac{d}{n\epsilon^2}})$} & \citet{wang2020sparse} \\ \cline{2-6} 
 & sample & k-sparse & \mbox{$\tilde{\mathcal{O}}(\frac{k\sqrt{d}}{\sqrt{n}\epsilon})$} & \mbox{$\Omega(\sqrt{\frac{dk}{n\epsilon^2}})$} & \citet{zhu2023improved} \\ \cline{2-6} 
 & user-level & k-sparse & \mbox{$\tilde{\mathcal{O}}( \sqrt{\frac{k}{nm}} + \frac{k}{\sqrt{nm} \epsilon} )$} & \mbox{$\Omega( \frac{k}{\sqrt{n m}\epsilon} )$} & \citet{ma2024better} \\ 
\hline\hline
\multicolumn{6}{@{}l}{\textbf{Shuffle DP}} \\ \hline
- & user-level & general & - & \mbox{$\Omega (\sqrt{\frac{d}{n m}} + \frac{d}{ n \sqrt{m} \epsilon} )$} & \multirow{4}{=}{\centering\textbf{Our Work}} \\ \cline{1-5} 
\multirow{2}{*}{Non-interactive} & user-level & k-sparse & - & \mbox{$\Omega ( \sqrt{\frac{k \log(d/k)}{nm}} + \frac{\sqrt{dk}}{n \sqrt{m} \epsilon} )$} &  \\ \cline{2-5} 
 & sample & k-sparse & \mbox{$\tilde{\mathcal{O}}( \sqrt{\frac{k}{n}} + \frac{d \sqrt{k}}{n \epsilon} )$} & \mbox{$\Omega( \frac{\sqrt{dk}}{n\epsilon} )$} &  \\ \cline{1-5} 
Interactive & user-level & k-sparse & \mbox{$\tilde{\mathcal{O}}( \sqrt{\frac{k}{n m}} + \frac{k}{n \sqrt{m} \epsilon} )$} & \mbox{$\Omega( \sqrt{\frac{k \log(d/k)}{n m}} + \frac{k \log(d/k)}{n\sqrt{m}\epsilon} )$} &  \\ 
\hline
\end{tabularx}
\end{table}

Private high-dimensional regression has been studied extensively in the central model of differential privacy. Early work by \citet{chaudhuri2011differentially} and \citet{kifer_private_2012} showed that private empirical risk minimization remains feasible even in high dimension, and \citet{talwar_nearly_2015} designed a nearly optimal private LASSO algorithm under an $\ell_1$ constraint. The central-DP optimization literature has since developed a much sharper understanding of stochastic convex optimization and private linear regression, including optimal-rate results in different geometries and heavy-tailed regimes \citep{bassily2019private, feldman2020private, asi2021private, wang2020differentially, kamath2022improved, hu2022high, liu2023near, varshney2022nearly,ding2026differentially,ding2024revisiting}. On the statistical side, \citet{cai2021cost} and \citet{cai2023score} established minimax rates and lower-bound techniques for private parameter estimation, showing that sample-level central DP only incurs logarithmic dependence on the ambient dimension in sparse problems. The score attack of \citet{cai2023score} is stated for datasets in which each record is one sample; the user-level score attack of Section~\ref{subsec:user_lower} extends it to the regime where a record is the whole local dataset of a user, which is what Theorem~\ref{thm:user_level_sparse_lower_bound} needs. These results provide the benchmark that one would ideally like to approach in decentralized settings.

The local model is considerably more restrictive. Foundational results on local privacy and statistical minimax rates are due to \citet{duchi2013local} and \citet{duchi_privacy_2014}, while later work studied the role of interaction in distributed private learning and local estimation \citep{smith_is_2017, zheng_collect_2017, daniely2019locally, hu2022high, acharya2022role, wang2021estimating,zhu2024truthful}. For sparse linear regression in particular, \citet{wang2020sparse} initiated the study of the local model, proving lower bounds in the $1$-sparse case and proposing an interactive private iterative hard-thresholding procedure under strong distributional assumptions. The later improved analysis of \citet{zhu2023improved} substantially sharpened this picture in two directions: it extended the lower bounds to general $k$-sparse regression, and it provided both the first explicit non-interactive LDP estimator for sparse regression and a corrected utility analysis for the interactive hard-thresholding line. Their results show that, even after these improvements, sample-level LDP sparse regression still exhibits a polynomial dependence on the ambient dimension $d$, with non-interactive and interactive upper bounds scaling on the order of $\widetilde O(d\sqrt{k}/(\sqrt n \epsilon))$ and $\widetilde O(k\sqrt d/(\sqrt n \epsilon))$, respectively. These phenomena are closely related to the non-private high-dimensional estimation literature, where canonical sparsity-aware procedures include the Dantzig selector \citep{candes2007dantzig}, elementary estimators \citep{yang2014elementary}, and iterative hard thresholding \citep{jain2014iterative}; minimax rates over sparse parameter classes are summarized by \citet{raskutti2011minimax}.

Several papers have therefore looked for stronger primitives or better structural assumptions. On the robust statistics side, clipped or Huberized estimators are now standard tools for controlling heavy-tailed noise and stabilizing private regression \citep{sun2020adaptive, kamath2020private}; these ideas also appear explicitly in the improved non-interactive LDP construction of \citet{zhu2023improved}, where clipped sufficient statistics are privatized locally before sparse recovery is performed at the server. Both of our protocols clip in the same way, and Algorithm~\ref{alg:shuffled_sparse_regression} privatizes the same clipped sufficient statistics, the difference being that the noise is added to the aggregate rather than to each record. On the decentralized side, \citet{ma2024better} moved to the user-level multi-sample regime and proposed the two-stage 2-SLR framework, where each user first locally screens candidate coordinates before a second-stage low-dimensional estimator is applied. Their result is especially relevant to our work because it shows that once multiple samples are grouped per user, the severe dimension dependence of standard sample-level LDP can be substantially alleviated. Algorithm~\ref{alg:pipeline} follows the same two-stage structure, and the local screening step of Algorithm~\ref{alg:heavyhitter} is a Lasso together with a threshold, which asks only for the restricted eigenvalue condition of Assumption~\ref{asm:re} rather than for uncorrelated features.

Shuffled differential privacy offers a middle ground between central DP and LDP. The basic model and its privacy amplification effect were formalized by \citet{cheu2019distributed} and \citet{erlingsson2019amplification}, and the amplification view was later refined through the privacy-blanket perspective \citep{balle2019privacy}. A separate line studies multi-message protocols that compute a sum in the shuffled model with the accuracy of the central model, both for real summation and for aggregation in federated learning \citep{balle2020private, ghazi2020private, girgis2024multi}. The summation step of Algorithm~\ref{alg:shuffled_sparse_regression} and of Stage 2 of Algorithm~\ref{alg:pipeline} is of this second type: the aggregator sees a noisy sum rather than $n$ separately randomized messages, so the noise can be calibrated to the $\ell_2$ sensitivity of the query. On the negative side, separations and limitations of shuffle privacy have been established for statistical tasks such as histogram estimation, together with the relation between shuffled DP and pan privacy \citep{balcer2019separating, cheu2021limits}, which is the route our non-interactive lower bound follows. Heavy-hitter primitives under local noise provide a practical mechanism for support identification in decentralized settings \citep{bun2019heavy}; Stage 1 of Algorithm~\ref{alg:pipeline} replaces the prefix tree of these protocols by one round of randomized response over $[d]$. However, despite this rapid progress, sparse linear regression in the shuffled model has remained largely open: the existing literature either analyzes general shuffled estimation primitives or studies local/user-level baselines rather than giving a sharp treatment of sparse regression itself. Our work fills this gap by characterizing both the non-interactive sample-level shuffled regime and the interactive user-level shuffled regime for sparse linear models.

\section{Detailed Algorithms and Parameter Settings}
\label{app:algorithms}

This section gives the pseudocode of the algorithms of Sections~\ref{sec:non SDP} and
\ref{sec:user-level SDP}, together with the parameter settings used by
Theorem~\ref{thm:non_interactive_upper}, Proposition~\ref{prop:support_consistency} and
Theorem~\ref{thm:interactive_user_level_sdp}. Throughout, $D = d(d+1)/2 + d = \Theta(d^2)$,
$v(\cdot)$ vectorizes the upper triangle of a symmetric matrix including the diagonal,
$\Pi_{B(0,R)}$ is the Euclidean projection onto the ball of radius $R$, and
$c_0, c_1, c_A, c_A', c_\lambda, c_p, B_\theta, C_\beta, C_\gamma$ are absolute constants.

Both protocols privatize a sum of integer vectors in the same way: each user rounds its
statistic to the integer grid, adds a share of a centered binomial noise, and uploads
additive shares of the result through parallel shufflers, one per coordinate. The noise
is calibrated to the $\ell_2$ sensitivity of the integer query rather than to its $\ell_1$
sensitivity, which is what keeps the privacy term at $d$ in
Theorem~\ref{thm:non_interactive_upper} and removes $d$ altogether in
Theorem~\ref{thm:interactive_user_level_sdp}.

Algorithm~\ref{alg:shuffled_sparse_regression} runs on all $n$ users at the budget
$(\epsilon, \delta)$. In Algorithm~\ref{alg:pipeline} the three groups of users, namely the
$n_1$ users of Stage 1, the $n_3$ users of Stage 2a and the $n_4$ users of Stage 2b, are
disjoint, so a neighboring pair of datasets changes the view of one group only and each group
runs at the full budget $(\epsilon, \delta)$ by parallel composition. The parameter settings
below are stated accordingly. The summation step requires at least $19$ participants per
group.

\subsection{Non-interactive Algorithm}
\label{app:alg_non_interactive}

\begin{algorithm}[ht]
    \caption{Sample-Level Non-Interactive Shuffled DP Sparse Linear Regression}
    \label{alg:shuffled_sparse_regression}
\begin{algorithmic}[1]
    \STATE {\bfseries Input:} Data $\{(x_i, y_i)\}_{i=1}^n$, sparsity $k$, privacy parameters
    $(\epsilon, \delta)$, parameters of Table~\ref{tab:param_non_interactive}.
    \STATE {\bfseries Local procedure} \quad (user $i \in [n]$)
    \STATE Compute $H_i = x_i x_i^\top$, $g_i = x_i y_i$, and set
    $F_i = (v(H_i); g_i) \in \mathbb{R}^{D}$.
    \STATE Clip: $\bar{F}_i = F_i \cdot \min\{1, C_F / \|F_i\|_2\}$.
    \STATE For each $j \in [D]$, round:
    $\bar{U}_{i,j} = \lfloor p \bar{F}_{i,j} \rfloor
    + \mathrm{Ber}(p \bar{F}_{i,j} - \lfloor p \bar{F}_{i,j} \rfloor)$.
    \STATE For each $j \in [D]$, draw
    $\zeta_{i,j} \sim \mathrm{Bin}(N_0, 1/2) - N_0/2$ and set
    $W_{i,j} = (\bar{U}_{i,j} + \zeta_{i,j}) \bmod q$.
    \STATE For each $j \in [D]$, split $W_{i,j}$ into $r$ additive shares in $\mathbb{Z}_q$ and
    send them, together with one unshuffled share, through the $j$-th shuffler.
    \STATE {\bfseries Shuffler:} permute the messages of each of the $D$ streams independently.
    \STATE {\bfseries Aggregator}
    \STATE For each $j \in [D]$, sum the received shares modulo $q$ to get $S_j$, and set
    $S_j \leftarrow S_j - q$ if $S_j > q/2$.
    \STATE Rescale $\hat{F}_j = S_j / (n p)$, fill $\hat{H} \in \mathbb{R}^{d \times d}$ from
    the first $d(d+1)/2$ entries of $\hat{F}$ and $\hat{g} \in \mathbb{R}^d$ from the last $d$.
    \STATE Solve $\tilde{\theta} = \hat{H}^{-1} \hat{g}$, soft-threshold
    $\hat{\theta} = \mathcal{S}_{\lambda_n}(\tilde{\theta})$, and project $\hat{\theta}$ onto
    $\{\theta : \|\theta\|_1 \le 1\}$.
    \STATE {\bfseries Output:} the $k$-sparse estimate $\hat{\theta}$.
\end{algorithmic}
\end{algorithm}

The parameter $N_0$ is even, so the noise of a user is integer valued, and $N_0 \ge 2$, so
every user holds at least one Bernoulli trial per coordinate and the shares of the $n$ users
add up to a centered $\mathrm{Bin}(n N_0, 1/2)$ noise on each coordinate. The choice of $p$
keeps the rounding error of the same order as the binomial noise, so that both are absorbed
into $\lambda_n$. The modulus $q$ is large enough that the sum of the uploaded values does not
wrap around. Under Assumption~\ref{asm:curv_bounded} we have $\|F_i\|_2 \le c_x^2 d + c_x c_y \sqrt{d}$, so the clipping step does not change $F_i$ once
$c_0 \ge c_x^2 + c_x c_y$; the step is kept so that the algorithm also applies to data that are
not bounded, where the factor $\log n$ in $C_F$ covers the tail of $\|F_i\|_2$, which then
exceeds $C_F$ with probability at most $n^{-3}$.

\begin{table}[ht]
\caption{Parameters of Algorithm~\ref{alg:shuffled_sparse_regression}, used in
Theorem~\ref{thm:non_interactive_upper}.}
\label{tab:param_non_interactive}
\centering
\small
\renewcommand{\arraystretch}{1.5}
\begin{tabularx}{\linewidth}{@{}l l >{\raggedright\arraybackslash}X@{}}
\toprule
Symbol & Setting & Role \\
\midrule
$C_F$ & $c_0 \, d \log n$ &
  $\ell_2$ clipping radius \\
$p$ & $\Big\lceil \dfrac{\sqrt{n}\, \epsilon}{d \log n \sqrt{\log(4D/\delta)}} \Big\rceil$ &
  rounding precision \\
$\Delta_2^{(p)}$ & $2 p C_F + \sqrt{D}$ &
  $\ell_2$ sensitivity of the integer query \\
$M$ & $\Big\lceil \dfrac{32 (\Delta_2^{(p)})^2 \log(4D/\delta)}{\epsilon^2} \Big\rceil$ &
  total binomial parameter \\
$N_0$ & $2 \lceil M / (2n) \rceil$ &
  binomial parameter per user and coordinate \\
$q$ & $4 n (p C_F + N_0)$ &
  modulus of the group $\mathbb{Z}_q$ \\[5pt]
$r$ & $\Big\lceil \dfrac{2 \log_2(8 D e^{\epsilon}/\delta) + \log_2 q}
  {\log_2 n - \log_2 e} \Big\rceil + 1$ &
  number of additive shares per coordinate \\[7pt]
$\lambda_n$ & $\dfrac{c_1}{\kappa_\infty} \Big( \sigma \sqrt{\dfrac{\log d}{n}}
  + \dfrac{C_F \sqrt{\log(4D/\delta) \log D}}{n \epsilon} \Big)$ &
  soft-thresholding parameter \\[3pt]
\bottomrule
\end{tabularx}
\end{table}

\subsection{Stage 1: Support Recovery}
\label{app:alg_stage1}

\begin{table}[ht]
\caption{Parameters of Algorithm~\ref{alg:heavyhitter}, used in
Proposition~\ref{prop:support_consistency}.}
\label{tab:param_stage1}
\centering
\small
\renewcommand{\arraystretch}{1.5}
\begin{tabularx}{\linewidth}{@{}l l >{\raggedright\arraybackslash}X@{}}
\toprule
Symbol & Setting & Role \\
\midrule
$\lambda_{\mathrm{loc}}$ & $c_\lambda \, \sigma \sqrt{\dfrac{\log d}{m}}$,
  \; with $\lambda_{\mathrm{loc}} \le \dfrac{\kappa \rho}{16\sqrt{k}}$ &
  regularization parameter of the local Lasso \\
$\rho/2$ & --- &
  threshold of the local selection step \\
$\gamma$ & $\dfrac{C_\gamma \log(1/\delta)}{n_1 \epsilon^2}$ &
  flipping probability of the randomized response \\
$\tau_{HH}$ & $\dfrac{n_1}{4k}$ &
  threshold on the debiased counts \\
\bottomrule
\end{tabularx}
\end{table}

\begin{algorithm}[ht]
    \caption{Heavy Hitter Support Recovery (Stage 1)}
    \label{alg:heavyhitter}
\begin{algorithmic}[1]
    \STATE {\bfseries Input:} Local datasets $\{(X_i, y_i)\}_{i=1}^{n_1}$, sparsity $k$,
    signal strength $\rho$, privacy parameters $(\epsilon, \delta)$, parameters of
    Table~\ref{tab:param_stage1}.
    \STATE {\bfseries Local procedure} \quad (user $i \in [n_1]$)
    \STATE Solve the Lasso
    $\hat{\theta}^{(i)} = \arg\min_{\theta \in \mathbb{R}^d}
    \frac{1}{2m} \|y_i - X_i \theta\|_2^2 + \lambda_{\mathrm{loc}} \|\theta\|_1$.
    \STATE Set $\hat{S}_i = \{\ell \in [d] : |\hat{\theta}^{(i)}_\ell| \ge \rho/2\}$, and
    $\hat{S}_i = [d]$ if this set is empty.
    \STATE Draw $v_i \sim \mathrm{Unif}(\hat{S}_i)$.
    \STATE For each $\ell \in [d]$, send
    $r_{i\ell} = \mathbb{I}[v_i = \ell] \oplus \mathrm{Ber}(\gamma)$ through the $\ell$-th
    shuffler.
    \STATE {\bfseries Shuffler:} permute the messages of each of the $d$ streams independently.
    \STATE {\bfseries Aggregator}
    \STATE For each $\ell \in [d]$, form the count $C_\ell = \sum_{i=1}^{n_1} r_{i\ell}$ and
    debias it: $\hat{W}_\ell = (C_\ell - n_1 \gamma) / (1 - 2\gamma)$.
    \STATE Set $\hat{S} = \{\ell \in [d] : \hat{W}_\ell \ge \tau_{HH}\}$ and
    $s = |\hat{S}|$.
    \STATE {\bfseries Output:} the support estimate $\hat{S}$.
\end{algorithmic}
\end{algorithm}

The window for $\lambda_{\mathrm{loc}}$ is non-empty under the first term of
\eqref{eq:m_condition}: the lower end is what dominates the score
$\|m^{-1} X_i^\top \zeta_i\|_\infty$ of one user with probability at least $1 - 1/(64k)$, and the
upper end is what turns the $\ell_2$ rate of the Lasso into the coordinatewise control needed by
the threshold $\rho/2$.
The value of $\gamma$ is the smallest one at which the counts are private. It satisfies
$\gamma \le 1/4$ as long as $\delta \ge (dn)^{-C k^2 \epsilon^2}$, so that the debiasing factor
$1/(1 - 2\gamma)$ is at most $2$, and under \eqref{eq:n1_condition} the fluctuation of the
debiased counts stays below $n_1/(16k)$, so that $\tau_{HH}$ separates the two groups of
coordinates.

\subsection{Stage 2: Constrained Estimation}
\label{app:alg_stage2}

The $n_2 = n - n_1$ remaining users are split into two groups of sizes
$n_3 = \lfloor n_2 / 2 \rfloor$ and $n_4 = n_2 - n_3$. Both groups upload through the same
summation step as Algorithm~\ref{alg:shuffled_sparse_regression}, with $s$ parallel shufflers
instead of $D$. The quantity $\hat{A}$ broadcast with $\tau$ is a function of the public
parameters and of $\hat{S}$, which has already been released, so broadcasting $\tau$ costs no
privacy budget.

\begin{algorithm}[ht]
    \caption{Clipped Sparse Estimation (Stage 2)}
    \label{alg:sparse_estimation_1}
\begin{algorithmic}[1]
    \STATE {\bfseries Input:} Support $\hat{S}$ with $s = |\hat{S}|$, local datasets of the
    $n_2$ users, privacy parameters $(\epsilon, \delta)$, parameters of
    Table~\ref{tab:param_stage2}.
    \STATE {\bfseries Local procedure} \quad (every user of Stage 2)
    \STATE Compute the local least squares estimate
    $\tilde{\theta}_i = (X_{i,\hat{S}}^\top X_{i,\hat{S}})^{-1} X_{i,\hat{S}}^\top y_i
    \in \mathbb{R}^{s}$.
    \STATE {\bfseries Stage 2a} \quad (the first $n_3$ users)
    \STATE Clip $\bar{\theta}_i = \Pi_{B(0, R_0)}(\tilde{\theta}_i)$.
    \STATE For each $j \in [s]$, round $p_a \bar{\theta}_{i,j}$ to the integer grid, add
    $\mathrm{Bin}(N_0^{(a)}, 1/2) - N_0^{(a)}/2$, reduce modulo $q_a$, and send $r_a$ additive
    shares together with one unshuffled share through the $j$-th shuffler.
    \STATE {\bfseries Aggregator:} decode the shuffled messages, rescale by $n_3 p_a$ to get
    the coarse estimate $c \in \mathbb{R}^{s}$, and broadcast
    $\tau = \tau_\beta + \hat{A}$.
    \STATE {\bfseries Stage 2b} \quad (the remaining $n_4$ users)
    \STATE Clip around the coarse estimate:
    $\Delta_i = \Pi_{B(0, \tau)}(\tilde{\theta}_i - c)$.
    \STATE For each $j \in [s]$, round $p_b \Delta_{i,j}$ to the integer grid, add
    $\mathrm{Bin}(N_0^{(b)}, 1/2) - N_0^{(b)}/2$, reduce modulo $q_b$, and send $r_b$ additive
    shares together with one unshuffled share through the $j$-th shuffler.
    \STATE {\bfseries Aggregator:} decode the shuffled messages, rescale by $n_4 p_b$, and add
    back $c$ to get $\hat{\theta}_{\hat{S}} \in \mathbb{R}^{s}$.
    \STATE {\bfseries Output:} the estimate $\hat{\theta} \in \mathbb{R}^d$ that equals
    $\hat{\theta}_{\hat{S}}$ on $\hat{S}$ and zero elsewhere.
\end{algorithmic}
\end{algorithm}

The two groups of Stage 2 use the same summation step with different sizes and clipping
radii. To state their parameters once, we index them by $g \in \{a, b\}$ and set
\begin{equation*}
    (n_a, R_a) = (n_3, R_0) , \qquad (n_b, R_b) = (n_4, \tau) ,
\end{equation*}
so that $p_g, \Delta_g, M_g, N_0^{(g)}, q_g, r_g$ in Table~\ref{tab:param_stage2} are the
parameters of Stage 2a for $g = a$ and of Stage 2b for $g = b$.

\begin{table}[ht]
\caption{Parameters of Algorithm~\ref{alg:sparse_estimation_1}, used in
Theorem~\ref{thm:interactive_user_level_sdp}. The index $g \in \{a, b\}$ refers to Stage 2a
and Stage 2b, with $(n_a, R_a) = (n_3, R_0)$ and $(n_b, R_b) = (n_4, \tau)$.}
\label{tab:param_stage2}
\centering
\small
\renewcommand{\arraystretch}{1.5}
\begin{tabularx}{\linewidth}{@{}l l >{\raggedright\arraybackslash}X@{}}
\toprule
Symbol & Setting & Role \\
\midrule
$\tau_\beta$ & $C_\beta \, \sigma \sqrt{\dfrac{s \log n}{m}}$ &
  uniform bound on $\|\tilde{\theta}_i - \theta^*_{\hat{S}}\|_2$ \\
$R_0$ & $B_\theta + \tau_\beta$ &
  clipping radius of Stage 2a \\
$\hat{A}$ & $c_A \, \sigma \sqrt{\dfrac{s \log n}{n_3 m}}
  + c_A' \, \dfrac{R_0 \sqrt{s \log n \log(4s/\delta)}}{n_3 \epsilon}$ &
  error bound of the coarse estimate $c$ \\
$\tau$ & $\tau_\beta + \hat{A}$ &
  clipping radius of Stage 2b \\
\midrule
$p_g$ & $\max\Big\{ \dfrac{c_p \sqrt{n_g \log n}\, \epsilon}{R_g \sqrt{\log(1/\delta)}}, \;
  \dfrac{(\sqrt{s} + 2) \max\{1, \epsilon\}}{R_g} \Big\}$ &
  rounding precision \\
$\Delta_g$ & $2 p_g R_g + \max\{1, \epsilon\} \sqrt{s} + 2$ &
  $\ell_2$ sensitivity of the integer query; $\Delta_g \le 3 p_g R_g$ \\
$M_g$ & $\Big\lceil \dfrac{32 \Delta_g^2 \log(4s/\delta)}{\epsilon^2} \Big\rceil$ &
  total binomial parameter \\
$N_0^{(g)}$ & $2 \lceil M_g / (2 n_g) \rceil$ &
  binomial parameter per user and coordinate \\
$q_g$ & $4 n_g (p_g R_g + N_0^{(g)})$ &
  modulus of the group $\mathbb{Z}_{q_g}$ \\
$r_g$ & $\Big\lceil \dfrac{2 \log_2(8 s e^{\epsilon}/\delta) + \log_2 q_g}
  {\log_2 n_g - \log_2 e} \Big\rceil + 1$ &
  number of additive shares per coordinate \\
\bottomrule
\end{tabularx}
\end{table}

The constant $B_\theta$ is any known upper bound on $\|\theta^*\|_2$, which
Assumption~\ref{asm:model} provides; taking $B_\theta = 1$ is enough on the parameter space we
consider, so that $R_0 = \Theta(1)$. The second term in the setting of $p_g$ makes
$\Delta_g \le 3 p_g R_g$ and $\Delta_g \ge \sqrt{s} \max\{1, \epsilon\}$ hold for every value of
$s$, at the cost of a larger $N_0^{(g)}$, that is of more messages per user; it does not change
the noise seen by the aggregator. Indeed the precision $p_g$ cancels against the sensitivity
$\Delta_g$, and that noise is
$\Theta\big( R_0 \sqrt{s \log n \log(4s/\delta)} / (n_3 \epsilon) \big)$ for the coarse
estimate and $\Theta\big( \tau \sqrt{s \log n \log(4s/\delta)} / (n_4 \epsilon) \big)$ for the
refined one. Since $\tau = \Theta(\tau_\beta)$ under the condition \eqref{eq:thm3_condition},
the second scale is what produces the privacy term of
Theorem~\ref{thm:interactive_user_level_sdp}, and it involves $s$ rather than $d$.

\section{Experimental Setup and Results}
\label{app:experiments}

We evaluate U2R on synthetic sparse linear regression and nacelle
temperature prediction using the Kelmarsh wind farm dataset.

\subsection{Experimental Setup}
\label{app:exp_setup}

We compare U2R in Algorithm~\ref{alg:pipeline} with
2-SLR~\citep{ma2024better}, LDP-IHT~\citep{zhu2023improved},
Sample-SDP in Algorithm~\ref{alg:shuffled_sparse_regression}, and
non-private Lasso fitted to all $nm$ observations. U2R and 2-SLR protect
users; LDP-IHT and Sample-SDP protect individual observations. U2R and
Sample-SDP use shuffled $(\epsilon,\delta)$-DP with $\delta=10^{-5}$;
2-SLR and LDP-IHT use pure $\epsilon$-LDP. For the same protected unit,
pure $\epsilon$-DP implies $(\epsilon,10^{-5})$-DP, so we compare utility
at the same numerical budget while retaining the differences in privacy
units and trust models.

For synthetic data, each user holds $m$ independent observations
\[
 x_{ij}\sim\mathcal{N}(0,I_d),\qquad
 y_{ij}=x_{ij}^{\top}\theta^*+\zeta_{ij},\qquad
 \zeta_{ij}\sim\mathcal{N}(0,\sigma^2),
\]
with covariates independent of residuals. We choose $k$ coordinates
uniformly, set their coefficients to $1/k$, and set the others to zero.
The public defaults are $n=15000$, $m=50$, $d=256$, $k=8$,
$\sigma=0.25$, and $\epsilon=4$; each sweep varies one quantity.
The realized support is used only for evaluation.

Each setting has 30 repetitions in Exp1 and Exp2 and ten in Exp3
and Exp4, with shared training data and $4000$ independent test
observations across methods. Matching settings and
privacy budgets reuse the data and randomization, with one Lasso fit per
repetition. We simulate aggregate mechanism outputs and measure relative
coefficient error $\|\hat\theta-\theta^*\|_2/\|\theta^*\|_2$, test MSE,
support recall $|\hat S\cap S^*|/k$, and exact recovery.

U2R uses disjoint user groups of sizes $\lfloor n/2\rfloor$ for selection
and two equal remaining groups for coarse and refined estimation. Its local Lasso penalty is
$0.01\sigma\sqrt{\log(dn)/m}$, with coefficient threshold $\rho/2$ and
count threshold $n_1/(4k)$, where $\rho=1/k$. The response probability is
the larger of $32\log(1/\delta)/(n_1\epsilon^2)$ and the value calibrated
from shuffled binary response or the bound for two bits under pure local
DP. Coordinates are shuffled independently. Stage 2 follows
Algorithm~\ref{alg:sparse_estimation_1} with local OLS, clipping radius
and rounding coefficients set to one, and privacy calibration for finite
samples, including the summation protocol.

For 2-SLR, selection uses the $64$ largest positive marginal correlations,
selection size $2$, $B=1$, and $4$ bins. LDP-IHT uses $T=2$ and gradient
radius $4$; separate development seeds select $\eta=1.5$, $\tau_x=1.5$,
$\tau_y=0.75$, $k'=10$ for $\epsilon<4$, and $\eta=2$, $\tau_x=1$,
$\tau_y=0.5$, $k'=8$ otherwise. Sample-SDP uses $c_0=0.1$ and $c_1=2$.
Lasso uses five fold cross validation over $300$ penalties, no intercept,
tolerance $10^{-4}$, and at most $3000$ iterations. These rules are fixed
before the sweeps.

Kelmarsh contains measurements every ten minutes from six turbines over
2017 to 2022. We predict nacelle temperature from $90$ predictors:
$59$ physical mean signals, four cyclic time features, $22$ missing value
indicators, and five turbine indicators. Whole UTC dates are split into
public, training, and test sets in nominal proportions $20\%$, $70\%$,
and $10\%$, leaving $1{,}295{,}987$ training and $185{,}416$ test rows
after removing missing responses. Separate public dates provide $80{,}000$
rows for preprocessing and $40{,}000$ for validation. Public medians fill
missing values; public means and standard deviations scale features and
responses, with indicators centered.

Each real data repetition samples $nm$ training rows without replacement
and assigns them to $n=10000$ simulated users. We test $m\in\{10,50,100\}$
and $\epsilon\in\{1,4\}$, sharing rows and groups across methods and
budgets. The test set is fixed.
U2R uses $\rho=0.05$, a unit parameter norm bound, and $\sigma=0.25$.
Public validation over all six settings and three separate seeds selects
$k=12$ from $\{8,10,12\}$. After count thresholding, U2R keeps at most
$\min\{k,m-2\}$ coordinates ranked by debiased private counts and calibrates
Stage 2 to the selected dimension.

For real data, 2-SLR uses $16$ bins. LDP-IHT uses $T=5$, $\eta=0.1$,
$\tau_x=1$, $\tau_y=0.5$, and $k'=8$. Sample-SDP uses $c_0=0.03$,
$c_1=1$, and diagonal shift
$\max\{0,-\lambda_{\min}(\widetilde G)\}+4\sqrt d\,s_G$,
where $s_G$ is the noise standard deviation per coordinate.
Lasso cross validation is grouped by UTC date. Other settings remain
unchanged. We report ten repetitions
of MSE in the original units and
its ratio to the paired Lasso fit, reused across budgets.
\subsection{Main Results}
\label{app:exp_main}

\begin{samepage}
\noindent\textbf{Privacy budget and number of users.}
Figure~\ref{fig:exp_epsilon_users} shows relative coefficient error
as $\epsilon$ varies in Exp1 and $n$ varies in Exp2. In Exp1, U2R has
the lowest mean error among the private methods at every tested budget,
decreasing from $0.04260$ at $\epsilon=0.5$ to $0.00544$ at
$\epsilon=32$. At $\epsilon=4$, its error is $0.00631$, comparable to
non-private Lasso. Most of the reduction occurs between $\epsilon=0.5$
and $4$, with smaller gains at larger budgets.\par
\end{samepage}

\begin{samepage}
In Exp2, increasing $n$ from $8000$ to $20000$ at $m=50$ reduces
U2R's mean error from $0.01022$ to $0.00496$. Each private method has
lower mean error at the larger user count. U2R recovers the exact
support in every repetition, so its remaining coefficient error lies
on the active coordinates. Section~\ref{app:exp_support} examines
support recovery as $n$ and $m$ vary.\par
\end{samepage}

\begin{figure}[!htbp]
\centering
\includegraphics[width=\linewidth]{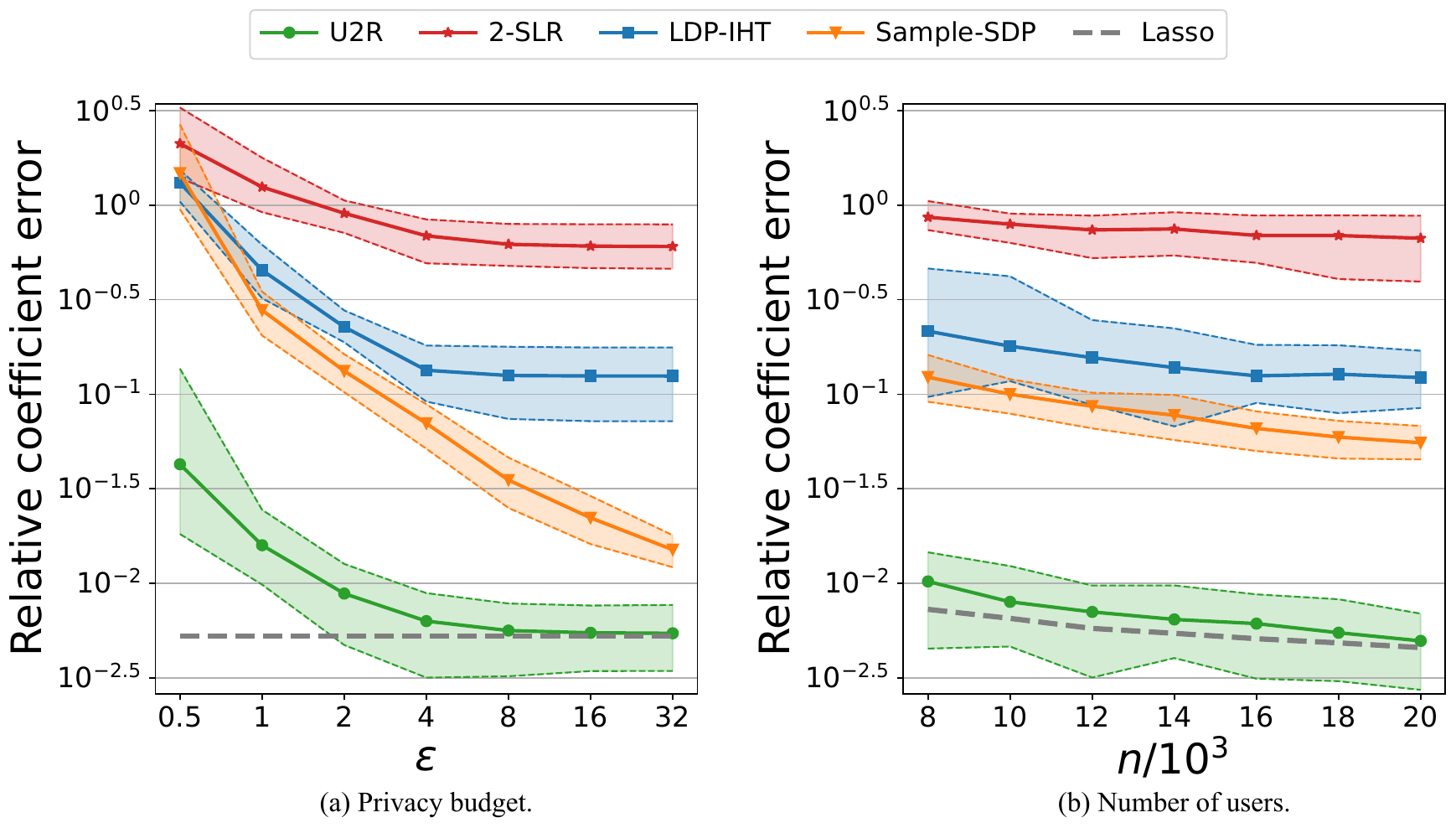}
\caption{Relative coefficient error versus $\epsilon$ in the left,
and $n$ in the right, with default $n=15000$, $m=50$, $d=256$, $k=8$ and $\epsilon=4$. Lines show means over 30 repetitions; shaded bands span the empirical 95\% coverage of 2.5th--97.5th percentiles.}
\label{fig:exp_epsilon_users}
\end{figure}

\begin{samepage}
\noindent\textbf{Dimension $d$.}
In Exp3, U2R's mean error ranges from $0.00553$ to $0.00659$ as $d$
increases from $64$ to $512$, with $k=8$ and $m=50$ fixed, as shown in
Table~\ref{tab:exp_dimension_10seeds}. It remains close to non-private
Lasso across these dimensions. Sample-SDP's error increases from
$0.01849$ to $0.14111$ over the same range.\par
\end{samepage}

\begin{table}[!htbp]
\caption{Exp3: relative coefficient error versus dimension at $n=15000$,
$m=50$, and $\epsilon=4$. Entries are means
over ten repetitions.}
\label{tab:exp_dimension_10seeds}
\centering
\small
\renewcommand{\arraystretch}{1.5}
\setlength{\tabcolsep}{6pt}
\begin{tabularx}{\linewidth}{@{}r *{5}{>{\raggedleft\arraybackslash}X}@{}}
\toprule
$d$ & U2R & 2-SLR & LDP-IHT & Sample-SDP & Lasso \\
\midrule
64 & $0.00643$ & $0.49079$ & $0.13262$ & $0.01849$ & $0.00450$ \\
\addlinespace[3pt]
128 & $0.00659$ & $0.61461$ & $0.12595$ & $0.03571$ & $0.00453$ \\
\addlinespace[3pt]
256 & $0.00553$ & $0.64914$ & $0.13354$ & $0.07027$ & $0.00520$ \\
\addlinespace[3pt]
512 & $0.00598$ & $0.77533$ & $0.18692$ & $0.14111$ & $0.00522$ \\
\bottomrule
\end{tabularx}
\end{table}

\begin{samepage}
\noindent\textbf{Real data prediction.}
Table~\ref{tab:exp_real_ratio} reports test MSE relative to non-private
Lasso in Exp5. U2R has the lowest mean ratio among the private methods
in all six settings. As $m$ increases from $10$ to $100$, its ratio
decreases from $1.48261$ to $1.12978$ at $\epsilon=1$ and from $1.43505$
to $1.13213$ at $\epsilon=4$. More observations per user bring U2R
closer to Lasso. Increasing $\epsilon$ reduces its mean prediction error
at $m=10$ and $50$; at $m=100$, the ratios are similar under both budgets.\par
\end{samepage}

\begin{table}[!htbp]
\caption{Exp5: Kelmarsh test MSE divided by the paired non-private Lasso MSE at $n=10000$ and $d=90$. U2R and Sample-SDP use \mbox{$\delta=10^{-5}$}. Entries are mean $\pm$ sample standard deviation over ten repetitions.}
\label{tab:exp_real_ratio}
\centering
\small
\renewcommand{\arraystretch}{1.5}
\setlength{\tabcolsep}{6pt}
\begin{tabularx}{\linewidth}{@{}l r *{4}{>{\raggedleft\arraybackslash}X}@{}}
\toprule
Budget & $m$ & U2R & 2-SLR & LDP-IHT & Sample-SDP \\
\midrule
 & $10$ & {\renewcommand{\arraystretch}{1}\begin{tabular}{@{}r@{}}$1.48261$ \\ $\pm 0.04759$\end{tabular}} & {\renewcommand{\arraystretch}{1}\begin{tabular}{@{}r@{}}$11.21592$ \\ $\pm 10.06853$\end{tabular}} & {\renewcommand{\arraystretch}{1}\begin{tabular}{@{}r@{}}$3.15212$ \\ $\pm 0.72151$\end{tabular}} & {\renewcommand{\arraystretch}{1}\begin{tabular}{@{}r@{}}$3.14939$ \\ $\pm 0.36992$\end{tabular}} \\
\addlinespace[3pt]
$\epsilon=1$ & $50$ & {\renewcommand{\arraystretch}{1}\begin{tabular}{@{}r@{}}$1.17245$ \\ $\pm 0.02983$\end{tabular}} & {\renewcommand{\arraystretch}{1}\begin{tabular}{@{}r@{}}$3.08118$ \\ $\pm 1.14883$\end{tabular}} & {\renewcommand{\arraystretch}{1}\begin{tabular}{@{}r@{}}$2.36752$ \\ $\pm 0.39188$\end{tabular}} & {\renewcommand{\arraystretch}{1}\begin{tabular}{@{}r@{}}$1.91000$ \\ $\pm 0.31047$\end{tabular}} \\
\addlinespace[3pt]
 & $100$ & {\renewcommand{\arraystretch}{1}\begin{tabular}{@{}r@{}}$1.12978$ \\ $\pm 0.02725$\end{tabular}} & {\renewcommand{\arraystretch}{1}\begin{tabular}{@{}r@{}}$3.23649$ \\ $\pm 1.62979$\end{tabular}} & {\renewcommand{\arraystretch}{1}\begin{tabular}{@{}r@{}}$2.25214$ \\ $\pm 0.45721$\end{tabular}} & {\renewcommand{\arraystretch}{1}\begin{tabular}{@{}r@{}}$1.50741$ \\ $\pm 0.12895$\end{tabular}} \\
\midrule
 & $10$ & {\renewcommand{\arraystretch}{1}\begin{tabular}{@{}r@{}}$1.43505$ \\ $\pm 0.01006$\end{tabular}} & {\renewcommand{\arraystretch}{1}\begin{tabular}{@{}r@{}}$3.07791$ \\ $\pm 1.29223$\end{tabular}} & {\renewcommand{\arraystretch}{1}\begin{tabular}{@{}r@{}}$2.44676$ \\ $\pm 0.59026$\end{tabular}} & {\renewcommand{\arraystretch}{1}\begin{tabular}{@{}r@{}}$2.05930$ \\ $\pm 0.36280$\end{tabular}} \\
\addlinespace[3pt]
$\epsilon=4$ & $50$ & {\renewcommand{\arraystretch}{1}\begin{tabular}{@{}r@{}}$1.15861$ \\ $\pm 0.02794$\end{tabular}} & {\renewcommand{\arraystretch}{1}\begin{tabular}{@{}r@{}}$2.08907$ \\ $\pm 0.16864$\end{tabular}} & {\renewcommand{\arraystretch}{1}\begin{tabular}{@{}r@{}}$2.36333$ \\ $\pm 0.41893$\end{tabular}} & {\renewcommand{\arraystretch}{1}\begin{tabular}{@{}r@{}}$1.48151$ \\ $\pm 0.27440$\end{tabular}} \\
\addlinespace[3pt]
 & $100$ & {\renewcommand{\arraystretch}{1}\begin{tabular}{@{}r@{}}$1.13213$ \\ $\pm 0.00306$\end{tabular}} & {\renewcommand{\arraystretch}{1}\begin{tabular}{@{}r@{}}$2.19689$ \\ $\pm 0.26004$\end{tabular}} & {\renewcommand{\arraystretch}{1}\begin{tabular}{@{}r@{}}$2.32449$ \\ $\pm 0.33028$\end{tabular}} & {\renewcommand{\arraystretch}{1}\begin{tabular}{@{}r@{}}$1.37230$ \\ $\pm 0.12374$\end{tabular}} \\
\bottomrule
\end{tabularx}
\end{table}

\subsection{Additional Results}
\label{app:exp_support}

\begin{table}[htbp]
\caption{Exp4: Stage 1 support recall and exact recovery. Recall entries
are mean $\pm$ sample standard deviation in percentage points; exact recovery
is the number of repetitions with $\hat S=S^*$ out of ten.}
\label{tab:exp_support_10seeds}
\centering
\small
\renewcommand{\arraystretch}{1.5}
\setlength{\tabcolsep}{6pt}
\begin{tabularx}{\linewidth}{@{}r *{4}{>{\raggedleft\arraybackslash}X}@{}}
\toprule
Value & \shortstack[r]{U2R recall\\in \%} & \shortstack[r]{2-SLR recall\\in \%} & U2R exact & 2-SLR exact \\
\midrule
\multicolumn{5}{l}{Exp4a: vary $m$; $n=15000$, $d=256$, $\epsilon=4$} \\
25 & $3.75 \pm 8.44$ & $26.25 \pm 13.76$ & 0/10 & 0/10 \\
\addlinespace[2pt]
50 & $100.00 \pm 0.00$ & $63.75 \pm 12.43$ & 10/10 & 0/10 \\
\addlinespace[2pt]
75 & $100.00 \pm 0.00$ & $72.50 \pm 16.46$ & 10/10 & 1/10 \\
\addlinespace[2pt]
100 & $100.00 \pm 0.00$ & $76.25 \pm 12.43$ & 10/10 & 0/10 \\
\midrule
\multicolumn{5}{l}{Exp4b: vary $n$; $m=50$, $d=256$, $\epsilon=4$} \\
8000 & $100.00 \pm 0.00$ & $35.00 \pm 12.91$ & 10/10 & 0/10 \\
\addlinespace[2pt]
10000 & $100.00 \pm 0.00$ & $46.25 \pm 10.29$ & 10/10 & 0/10 \\
\addlinespace[2pt]
12000 & $100.00 \pm 0.00$ & $50.00 \pm 10.21$ & 10/10 & 0/10 \\
\addlinespace[2pt]
15000 & $100.00 \pm 0.00$ & $63.75 \pm 12.43$ & 10/10 & 0/10 \\
\addlinespace[2pt]
18000 & $100.00 \pm 0.00$ & $53.75 \pm 15.65$ & 10/10 & 0/10 \\
\addlinespace[2pt]
20000 & $100.00 \pm 0.00$ & $61.25 \pm 12.43$ & 10/10 & 0/10 \\
\bottomrule
\end{tabularx}
\end{table}

\noindent\textbf{Support recovery.}
Table~\ref{tab:exp_support_10seeds} compares the Stage 1 selectors in
Exp4. At $n=15000$, U2R's mean recall increases from $3.75\%$ at $m=25$
to $100\%$ at $m=50$ and remains $100\%$ at $m=75$ and $100$.
It recovers the exact support in all ten repetitions at each of these
three larger values of $m$. The 2-SLR recall increases from $26.25\%$
to $76.25\%$ over the same sweep. With $m=50$, U2R recovers the exact
support in every repetition at all tested user counts from $8000$ to
$20000$. The 2-SLR recall ranges from $35.00\%$ to $63.75\%$.\par

\begin{samepage}
\noindent\textbf{Prediction error.}
At the default setting, U2R and Lasso have mean test MSEs of $0.0625570$
and $0.0625521$, respectively, close to the residual variance
$\sigma^2=0.0625$. LDP-IHT and Sample-SDP yield $0.0648534$
and $0.0631971$. At $\epsilon=32$, U2R's MSE is $0.0625546$;
its reduction in coefficient error has little effect on prediction MSE
over this budget range.\par
\end{samepage}

\begin{samepage}
\noindent\textbf{Prediction in original units.}
Table~\ref{tab:exp_real_mse} gives test MSE in squared degrees Celsius
for Exp5. At $\epsilon=1$, increasing $m$ from $10$ to $100$ reduces
U2R's mean MSE from $3.55270$ to $2.70655$ and mean RMSE from
$1.88464$ to $1.64505\,{}^\circ\mathrm{C}$; mean $R^2$ increases from
$0.87900$ to $0.90782$. At $m=100$ and $\epsilon=4$, the mean MSEs
of U2R, Lasso, and Sample-SDP are $2.71218$, $2.39565$, and $3.28753$,
respectively.\par
\end{samepage}

\begin{table}[!htbp]
\caption{Exp5: Kelmarsh test MSE in squared degrees Celsius. U2R and Sample-SDP use \mbox{$\delta=10^{-5}$}. Entries are mean $\pm$ sample standard deviation over ten repetitions; Lasso is reused across budgets.}
\label{tab:exp_real_mse}
\centering
\small
\renewcommand{\arraystretch}{1.5}
\setlength{\tabcolsep}{6pt}
\begin{tabularx}{\linewidth}{@{}l r *{5}{>{\raggedleft\arraybackslash}X}@{}}
\toprule
Budget & $m$ & U2R & 2-SLR & LDP-IHT & Sample-SDP & Lasso \\
\midrule
 & $10$ & {\renewcommand{\arraystretch}{1}\begin{tabular}{@{}r@{}}$3.55270$ \\ $\pm 0.11390$\end{tabular}} & {\renewcommand{\arraystretch}{1}\begin{tabular}{@{}r@{}}$26.89088$ \\ $\pm 24.15804$\end{tabular}} & {\renewcommand{\arraystretch}{1}\begin{tabular}{@{}r@{}}$7.55337$ \\ $\pm 1.72935$\end{tabular}} & {\renewcommand{\arraystretch}{1}\begin{tabular}{@{}r@{}}$7.54559$ \\ $\pm 0.87616$\end{tabular}} & {\renewcommand{\arraystretch}{1}\begin{tabular}{@{}r@{}}$2.39625$ \\ $\pm 0.00376$\end{tabular}} \\
\addlinespace[2pt]
$\epsilon=1$ & $50$ & {\renewcommand{\arraystretch}{1}\begin{tabular}{@{}r@{}}$2.80885$ \\ $\pm 0.07139$\end{tabular}} & {\renewcommand{\arraystretch}{1}\begin{tabular}{@{}r@{}}$7.38104$ \\ $\pm 2.75047$\end{tabular}} & {\renewcommand{\arraystretch}{1}\begin{tabular}{@{}r@{}}$5.67181$ \\ $\pm 0.93834$\end{tabular}} & {\renewcommand{\arraystretch}{1}\begin{tabular}{@{}r@{}}$4.57575$ \\ $\pm 0.74360$\end{tabular}} & {\renewcommand{\arraystretch}{1}\begin{tabular}{@{}r@{}}$2.39572$ \\ $\pm 0.00123$\end{tabular}} \\
\addlinespace[2pt]
 & $100$ & {\renewcommand{\arraystretch}{1}\begin{tabular}{@{}r@{}}$2.70655$ \\ $\pm 0.06479$\end{tabular}} & {\renewcommand{\arraystretch}{1}\begin{tabular}{@{}r@{}}$7.75405$ \\ $\pm 3.90559$\end{tabular}} & {\renewcommand{\arraystretch}{1}\begin{tabular}{@{}r@{}}$5.39543$ \\ $\pm 1.09591$\end{tabular}} & {\renewcommand{\arraystretch}{1}\begin{tabular}{@{}r@{}}$3.61122$ \\ $\pm 0.30899$\end{tabular}} & {\renewcommand{\arraystretch}{1}\begin{tabular}{@{}r@{}}$2.39565$ \\ $\pm 0.00063$\end{tabular}} \\
\midrule
 & $10$ & {\renewcommand{\arraystretch}{1}\begin{tabular}{@{}r@{}}$3.43873$ \\ $\pm 0.02432$\end{tabular}} & {\renewcommand{\arraystretch}{1}\begin{tabular}{@{}r@{}}$7.37791$ \\ $\pm 3.10457$\end{tabular}} & {\renewcommand{\arraystretch}{1}\begin{tabular}{@{}r@{}}$5.86216$ \\ $\pm 1.40887$\end{tabular}} & {\renewcommand{\arraystretch}{1}\begin{tabular}{@{}r@{}}$4.93425$ \\ $\pm 0.86685$\end{tabular}} & {\renewcommand{\arraystretch}{1}\begin{tabular}{@{}r@{}}$2.39625$ \\ $\pm 0.00376$\end{tabular}} \\
\addlinespace[2pt]
$\epsilon=4$ & $50$ & {\renewcommand{\arraystretch}{1}\begin{tabular}{@{}r@{}}$2.77572$ \\ $\pm 0.06713$\end{tabular}} & {\renewcommand{\arraystretch}{1}\begin{tabular}{@{}r@{}}$5.00484$ \\ $\pm 0.40439$\end{tabular}} & {\renewcommand{\arraystretch}{1}\begin{tabular}{@{}r@{}}$5.66193$ \\ $\pm 1.00425$\end{tabular}} & {\renewcommand{\arraystretch}{1}\begin{tabular}{@{}r@{}}$3.54916$ \\ $\pm 0.65645$\end{tabular}} & {\renewcommand{\arraystretch}{1}\begin{tabular}{@{}r@{}}$2.39572$ \\ $\pm 0.00123$\end{tabular}} \\
\addlinespace[2pt]
 & $100$ & {\renewcommand{\arraystretch}{1}\begin{tabular}{@{}r@{}}$2.71218$ \\ $\pm 0.00678$\end{tabular}} & {\renewcommand{\arraystretch}{1}\begin{tabular}{@{}r@{}}$5.26299$ \\ $\pm 0.62298$\end{tabular}} & {\renewcommand{\arraystretch}{1}\begin{tabular}{@{}r@{}}$5.56877$ \\ $\pm 0.79226$\end{tabular}} & {\renewcommand{\arraystretch}{1}\begin{tabular}{@{}r@{}}$3.28753$ \\ $\pm 0.29627$\end{tabular}} & {\renewcommand{\arraystretch}{1}\begin{tabular}{@{}r@{}}$2.39565$ \\ $\pm 0.00063$\end{tabular}} \\
\bottomrule
\end{tabularx}
\end{table}

\FloatBarrier

\section{Omitted Proofs in Section \ref{sec:non SDP}}
\label{app:proofs_non_interactive}

This section provides proofs for the theories in Section 4 of the main text.

\subsection{Proof of Theorem~\ref{thm:shuffled_dp_lower}}
\label{app:non_lower}

We first reduce a robust shuffled DP protocol to a pan-private algorithm, and then prove the
lower bound for pan-private algorithms by a hybrid argument over the stream and Assouad's
method over the coordinates.

Let $\mathcal{X} = \{\pm 1\}^{d+1}$ be the data universe of
Definition~\ref{def:hard_distribution}, let $U$ be the uniform distribution on $\mathcal{X}$,
and let $\chi_i(x, y) = y x_i$ for $i \in [d]$. These functions satisfy
$\mathbb{E}_U[\chi_i] = 0$ and $\mathbb{E}_U[\chi_i \chi_l] = \mathbb{I}[i = l]$. For
$\theta \in \mathbb{R}^d$ let $Q_\theta(x, y) = U(x, y)\big(1 + y \langle x, \theta \rangle\big)$,
which is the distribution of one sample in Definition~\ref{def:hard_distribution} with parameter
$\theta$. For $z \in \{-1, 0, 1\}^d$ we write $P_z = Q_{\gamma z}$ and
$Z_0(z) = \{i \in [d] : z_i = 0\}$. For $i \in Z_0(z)$ and $b \in \{\pm 1\}$, let
$z^{i,b} = z + b e_i$. Then
\begin{equation}
\label{eq:app_pz_diff}
    P_{z^{i,b}} - P_z = b \gamma \, \chi_i \, U .
\end{equation}
For a distribution $P$ on $\mathcal{X}$, $M(P^{n})$ denotes the distribution of the output of
$M$ on $n$ i.i.d. records from $P$.

Following \citet{cheu2021limits}, a protocol $\mathcal{P}$ is $(\epsilon, \delta, 1/3)$-robust
shuffled DP if, for every $\gamma' \ge 1/3$, the shuffled messages of any $\gamma' n$ users that
follow the protocol satisfy $(\epsilon, \delta)$-DP; this is the robustness in
Theorem~\ref{thm:shuffled_dp_lower}. Let $\mathcal{P}$ be such a non-interactive protocol and let
$n' = n/3$. By \citet[Theorem~2.6]{cheu2021limits}, there is an $(\epsilon, \delta)$-pan-private
algorithm $M^{\mathcal{P}}$ that processes a stream of $n'$ records and satisfies, for every
distribution $P$ on $\mathcal{X}$,
\begin{equation*}
    d_{TV}\big(M^{\mathcal{P}}(P^{n'}), \, \mathcal{M}_{\mathcal{P}}(P_{(2/9)}^{n})\big)
    \le e^{-c_1 n},
    \qquad P_{(b)} = b P + (1 - b) U ,
\end{equation*}
where $c_1 > 0$ is an absolute constant. Since $Q_\theta$ is linear in $\theta$, we have
$\frac{2}{9} Q_\theta + \frac{7}{9} U = Q_{2\theta/9}$, and $2\theta/9 \in \widetilde{\Theta}_k$
whenever $\theta \in \widetilde{\Theta}_k$. Let $\Pi_{1}$ be the projection onto the unit
$\ell_1$-ball and define $M = \frac{9}{2} \Pi_1(M^{\mathcal{P}})$. As a post-processing of the
output, $M$ is $(\epsilon, \delta)$-pan-private. The projection does not increase the distance to
any point of the unit $\ell_1$-ball, and $\|M - \theta\|_2 \le \frac{11}{2}$ for
$\theta \in \widetilde{\Theta}_k$, so
\begin{equation}
\label{eq:app_reduction}
    \mathbb{E}\big[\|M(Q_\theta^{n'}) - \theta\|_2\big]
    \le \frac{9}{2} \, \mathbb{E}\big[\|\mathcal{M}_{\mathcal{P}}(Q_{2\theta/9}^{n})
    - \tfrac{2}{9}\theta\|_2\big] + \frac{11}{2} e^{-c_1 n}
    \qquad \text{for all } \theta \in \widetilde{\Theta}_k .
\end{equation}
It therefore suffices to show that every $(\epsilon, \delta)$-pan-private algorithm $M$ on $n'$
records satisfies
\begin{equation}
\label{eq:app_pan_lower}
    \sup_{\theta^* \in \widetilde{\Theta}_k} \mathbb{E}\big[\|M(Q_{\theta^*}^{n'}) - \theta^*\|_2\big]
    \ge c_2 \, \frac{\sqrt{dk}}{n' \epsilon}
\end{equation}
for a constant $c_2 > 0$. Indeed, combining \eqref{eq:app_reduction} and
\eqref{eq:app_pan_lower} gives
$\sup_{\theta^*} \mathbb{E}\|\mathcal{M}_{\mathcal{P}}(Z) - \theta^*\|_2 \ge
\frac{2}{3} c_2 \frac{\sqrt{dk}}{n\epsilon} - \frac{11}{9} e^{-c_1 n}$. Since
$\epsilon \le \mathcal{O}(1)$, we have $\frac{\sqrt{dk}}{n\epsilon} \ge \Omega(\frac{1}{n})$, so for
$n$ larger than an absolute constant the last term is at most
$\frac{1}{3} c_2 \frac{\sqrt{dk}}{n\epsilon}$, which gives the theorem. In the rest of the proof
we fix an $(\epsilon, \delta)$-pan-private algorithm $M$ with internal states $S_1, \dots, S_{n'}$,
and prove \eqref{eq:app_pan_lower}. We take $n$ larger than an absolute constant, so that
$\delta = o(1/n)$ gives $n' \delta \le 1/64$.

For $j \in [n']$, a prefix $x_{<j} = (x_1, \dots, x_{j-1})$ and a record $x \in \mathcal{X}$, let
$q_j(\cdot \mid x_{<j}, x)$ be the distribution of $S_j$ when the first $j$ records are
$(x_{<j}, x)$. Applying Definition~\ref{def:pan_privacy} at time $t = j$ to two streams that
differ in the $j$-th record shows that, for every fixed $x_{<j}$, the map
$x \mapsto q_j(\cdot \mid x_{<j}, x)$ is $(\epsilon, \delta)$-DP. By
\citet[Lemma~3.7]{cheu2021limits}, there is a $(2\epsilon, 0)$-DP map
$x \mapsto q'_j(\cdot \mid x_{<j}, x)$ with
$d_{TV}\big(q_j(\cdot \mid x_{<j}, x), q'_j(\cdot \mid x_{<j}, x)\big) \le \delta$ for all $x$.
Below we drop $x_{<j}$ from the notation when it is fixed, and write sums over states $s$; for
general state spaces the sums are replaced by integrals of densities with respect to a common
dominating measure. Let $\bar{q}'_j(s) = \mathbb{E}_{X \sim U}[q'_j(s \mid X)]$ and
\begin{equation*}
    h_{j,s}(x) = \frac{q'_j(s \mid x)}{\bar{q}'_j(s)} - 1 , \qquad
    \hat{h}_{j,s}(i) = \mathbb{E}_U\big[h_{j,s} \chi_i\big] .
\end{equation*}
Since $q'_j$ is $(2\epsilon, 0)$-DP, $e^{-2\epsilon} \bar{q}'_j(s) \le q'_j(s \mid x) \le
e^{2\epsilon} \bar{q}'_j(s)$, so $|h_{j,s}(x)| \le e^{2\epsilon} - 1$. For a distribution $P$ on
$\mathcal{X}$, write $q_j \circ P$ for the distribution of $S_j$ when the $j$-th record is drawn
from $P$. If $P' - P = b \gamma \chi_i U$, then using $\mathbb{E}_U[\chi_i] = 0$,
\begin{equation}
\label{eq:app_step}
\begin{aligned}
    d_{TV}(q'_j \circ P', q'_j \circ P)
    &= \frac{1}{2} \sum_{s} \Big| \sum_{x} q'_j(s \mid x) \big(P'(x) - P(x)\big) \Big| \\
    &= \frac{\gamma}{2} \sum_{s} \big| \mathbb{E}_U[q'_j(s \mid X) \chi_i(X)] \big|
    = \frac{\gamma}{2} \, \mathbb{E}_{s \sim \bar{q}'_j} \big| \hat{h}_{j,s}(i) \big| .
\end{aligned}
\end{equation}

Fix $z \in \{-1, 0, 1\}^d$, $i \in Z_0(z)$ and $b \in \{\pm 1\}$. For $j = 0, \dots, n'$, let
$H_j$ be the distribution of the output of $M$ when $X_1, \dots, X_j \sim P_z$ and
$X_{j+1}, \dots, X_{n'} \sim P_{z^{i,b}}$, all independent. Then
$H_0 = M(P_{z^{i,b}}^{n'})$ and $H_{n'} = M(P_z^{n'})$. In $H_{j-1}$ and $H_j$ the prefix
$X_{<j}$ has the same distribution $P_z^{j-1}$, the suffix $X_{>j}$ has the same distribution
$P_{z^{i,b}}^{n'-j}$ and is independent of $(X_{<j}, X_j)$, and only $X_j$ differs. Since $M$
processes the stream one record at a time, its output is a post-processing of $(S_j, X_{>j})$.
Therefore
\begin{equation*}
    d_{TV}(H_{j-1}, H_j)
    \le \mathbb{E}_{X_{<j} \sim P_z^{j-1}} \, d_{TV}\big(q_j \circ P_{z^{i,b}}, \, q_j \circ P_z\big)
    \le \mathbb{E}_{X_{<j} \sim P_z^{j-1}} \Big[ \frac{\gamma}{2} \,
    \mathbb{E}_{s \sim \bar{q}'_j} \big| \hat{h}_{j,s}(i) \big| \Big] + 2\delta ,
\end{equation*}
where the second step replaces $q_j$ by $q'_j$ at a cost of $2\delta$ and then applies
\eqref{eq:app_pz_diff} and \eqref{eq:app_step}. By the triangle inequality,
\begin{equation}
\label{eq:app_hybrid}
    d_{TV}\big(M(P_{z^{i,b}}^{n'}), M(P_z^{n'})\big)
    \le \frac{\gamma}{2} \sum_{j=1}^{n'} \mathbb{E}_{X_{<j} \sim P_z^{j-1}}
    \mathbb{E}_{s \sim \bar{q}'_j} \big| \hat{h}_{j,s}(i) \big| + 2 n' \delta .
\end{equation}
The distribution of $X_{<j}$ and the maps $q'_j$ in \eqref{eq:app_hybrid} depend on $z$ but not
on $i$ or $b$.

Since $\chi_1, \dots, \chi_d$ are orthonormal in $L_2(U)$, Bessel's inequality gives
$\sum_{i \in [d]} \hat{h}_{j,s}(i)^2 \le \mathbb{E}_U[h_{j,s}^2] \le (e^{2\epsilon} - 1)^2$, and by
the Cauchy--Schwarz inequality
$\sum_{i \in Z_0(z)} |\hat{h}_{j,s}(i)| \le \sqrt{d} \, (e^{2\epsilon} - 1)$. Summing
\eqref{eq:app_hybrid} over $i \in Z_0(z)$ and $b \in \{\pm 1\}$, and exchanging the sum over $i$
with the expectations, which do not depend on $i$, we obtain
\begin{equation}
\label{eq:app_sum}
    T(z) := \sum_{i \in Z_0(z)} T_i(z) \le \gamma n' \sqrt{d} \, (e^{2\epsilon} - 1) + 4 d n' \delta ,
    \qquad
    T_i(z) := \sum_{b \in \{\pm 1\}} d_{TV}\big(M(P_{z^{i,b}}^{n'}), M(P_z^{n'})\big) .
\end{equation}
Note that \eqref{eq:app_sum} bounds the sum over coordinates only, and this is all that is used
below.

Draw $Z$ as in Definition~\ref{def:hard_distribution}, let $\theta_Z = \gamma Z$, draw the stream
$X \sim P_Z^{n'}$ given $Z$, and let $\hat{\theta} = M(X)$. For $i \in [d]$ let
$G_i = \{\|Z_{-i}\|_0 \le k - 1\}$ and $\psi_i = \mathrm{sign}(\hat{\theta}_i)$, with
$\mathrm{sign}(0) = +1$. Condition on $Z_{-i}$ and on $Z_i \neq 0$, and let $z = (Z_{-i}, 0)$, so
that $z_i = 0$. Then $Z_i$ is uniform on $\{\pm 1\}$ and $X \sim P_{z^{i, Z_i}}^{n'}$, so by Le Cam's
method and the triangle inequality,
\begin{equation*}
    \mathbb{P}\big(\psi_i \neq Z_i \mid Z_{-i}, Z_i \neq 0\big)
    \ge \frac{1}{2}\Big(1 - d_{TV}\big(M(P_{z^{i,+1}}^{n'}), M(P_{z^{i,-1}}^{n'})\big)\Big)
    \ge \frac{1}{2}\big(1 - T_i(z)\big) .
\end{equation*}
Let $N = \sum_{i \in [d]} \mathbb{I}[Z_i \neq 0, \, G_i, \, \psi_i \neq Z_i]$. If $N \ge 1$, some
$i$ has $Z_i \neq 0$ and $\|Z_{-i}\|_0 \le k - 1$, so $\|Z\|_0 \le k$ and $N \le \|Z\|_0 \le k$.
Each counted coordinate has $|\hat{\theta}_i - \gamma Z_i| \ge \gamma$. Hence
\begin{equation}
\label{eq:app_hamming}
    \|\hat{\theta} - \theta_Z\|_2 \, \mathbb{I}[\|Z\|_0 \le k] \ge \gamma \sqrt{N}
    \ge \frac{\gamma N}{\sqrt{k}} .
\end{equation}
Since $Z_i$ is independent of $Z_{-i}$ and $\mathbb{P}(Z_i \neq 0) = \frac{k}{2d}$,
\begin{equation*}
    \mathbb{E}[N]
    \ge \frac{k}{4d} \sum_{i=1}^d \Big( \mathbb{P}(G_i)
    - \mathbb{E}\big[T_i\big((Z_{-i}, 0)\big)\big] \Big) .
\end{equation*}
We bound the two sums in turn. First, $\|Z_{-i}\|_0 \sim \mathrm{Bin}(d-1, \frac{k}{2d})$ has mean
less than $\frac{k}{2}$, so Markov's inequality gives $\mathbb{P}(G_i^c) \le \frac{1}{2}$ and
$\sum_i \mathbb{P}(G_i) \ge \frac{d}{2}$. Second, $(Z_{-i}, 0) = Z$ on the event $\{Z_i = 0\}$,
which is independent of $Z_{-i}$ and has probability $1 - \frac{k}{2d} \ge \frac{1}{2}$, so by
\eqref{eq:app_sum},
\begin{equation*}
    \sum_{i=1}^d \mathbb{E}\big[T_i\big((Z_{-i}, 0)\big)\big]
    = \sum_{i=1}^d \frac{\mathbb{E}\big[T_i(Z) \, \mathbb{I}[Z_i = 0]\big]}{\mathbb{P}(Z_i = 0)}
    \le 2 \, \mathbb{E}\big[T(Z)\big]
    \le 2 \gamma n' \sqrt{d} \, (e^{2\epsilon} - 1) + 8 d n' \delta .
\end{equation*}
Combining the two bounds,
\begin{equation*}
    \mathbb{E}[N] \ge \frac{k}{8} - \frac{k \gamma n' (e^{2\epsilon} - 1)}{2 \sqrt{d}}
    - 2 k n' \delta .
\end{equation*}
We now choose $\gamma = \frac{\sqrt{d}}{8 n' (e^{2\epsilon} - 1)}$, so that the second term equals
$\frac{k}{16}$, and use $n' \delta \le \frac{1}{64}$, so that the third term is at most
$\frac{k}{32}$. This gives $\mathbb{E}[N] \ge \frac{k}{32}$, and by \eqref{eq:app_hamming},
\begin{equation*}
    \sup_{\theta^* \in \widetilde{\Theta}_k} \mathbb{E}\big[\|M(Q_{\theta^*}^{n'}) - \theta^*\|_2\big]
    \ge \mathbb{E}\big[\|\hat{\theta} - \theta_Z\|_2 \, \mathbb{I}[\|Z\|_0 \le k]\big]
    \ge \frac{\gamma \sqrt{k}}{32}
    = \frac{\sqrt{dk}}{256 \, n' (e^{2\epsilon} - 1)} .
\end{equation*}
Finally, for $0 < \epsilon \le c_0$ with an absolute constant $c_0$, we have
$e^{2\epsilon} - 1 \le 2 e^{2c_0} \epsilon$, which gives \eqref{eq:app_pan_lower} with
$c_2 = \frac{1}{512 e^{2c_0}}$ and completes the proof.

\subsection{Proof of Theorem~\ref{thm:non_interactive_upper}}
\label{app:non_upper}

We first give a lemma on the binomial mechanism. It is used in the privacy part of this proof,
and again in the proofs of Proposition~\ref{prop:support_consistency} and
Theorem~\ref{thm:interactive_user_level_sdp}.

\begin{lemma}
\label{lem:binomial_mechanism}
Let $s \ge 1$, $\epsilon > 0$ and $\delta \in (0, 1)$ with $\epsilon \le \log(2s/\delta)$, and let
$\Delta \ge \sqrt{s}\,\epsilon$. Let $B = (B_1, \dots, B_s)$ have independent entries
$B_j \sim \mathrm{Bin}(M, 1/2)$, where
\begin{equation}
\label{eq:binomial_condition}
    M \ge \frac{32 \Delta^2 \log(2s/\delta)}{\epsilon^2} .
\end{equation}
Then for any $v, v' \in \mathbb{Z}^s$ with $\|v - v'\|_2 \le \Delta$ and any
$E \subseteq \mathbb{Z}^s$,
\begin{equation*}
    \mathbb{P}[v + B \in E] \le e^{\epsilon} \, \mathbb{P}[v' + B \in E] + \delta .
\end{equation*}
\end{lemma}

Only an upper bound on $\|v - v'\|_2$ is used, so the condition $\Delta \ge \sqrt{s}\,\epsilon$ can
always be met by taking $\Delta$ larger. Subtracting a constant from $v + B$, for example
$M/2$ on each coordinate, is a post-processing and keeps the guarantee.

\begin{proof}[Proof of Lemma~\ref{lem:binomial_mechanism}]
Let $a = v' - v$ and $A = M/2$. For $t \in \{0, \dots, M\}^s$ define
\begin{equation*}
    \ell(t) = \sum_{j=1}^s \log \frac{\mathbb{P}[B_j = t_j]}{\mathbb{P}[B_j = t_j - a_j]} ,
\end{equation*}
with $\ell(t) = +\infty$ if $t_j - a_j \notin \{0, \dots, M\}$ for some $j$. For every output
$o$, $\mathbb{P}[v + B = o] = e^{\ell(o - v)} \, \mathbb{P}[v' + B = o]$, so
\begin{equation*}
    \mathbb{P}[v + B \in E] \le e^{\epsilon} \, \mathbb{P}[v' + B \in E]
    + \mathbb{P}[\ell(B) > \epsilon] ,
\end{equation*}
and it suffices to show $\mathbb{P}[\ell(B) > \epsilon] \le \delta$.

Write $w_j = B_j - A$, let $W = \sqrt{M \log(4s/\delta)/2}$, and let
$G = \{\max_{j} |w_j| \le W\}$. Each $w_j$ is a sum of $M$ independent terms in
$[-1/2, 1/2]$, so Hoeffding's inequality gives
$\mathbb{P}[|w_j| > W] \le 2 e^{-2W^2/M} = \delta/(2s)$, and $\mathbb{P}[G^c] \le \delta/2$.
We check that $W + \Delta \le A/2$. From $\Delta \ge \sqrt{s}\,\epsilon$ and
\eqref{eq:binomial_condition}, $M \ge 32 s \log(2s/\delta) \ge 32 \log(2s/\delta)$. Since
$\log(4s/\delta) \le 2 \log(2s/\delta)$, we get $W \le \sqrt{M \log(2s/\delta)} \le M/\sqrt{32}$.
From \eqref{eq:binomial_condition} and $\epsilon \le \log(2s/\delta)$,
\begin{equation*}
    \Delta \le \frac{\epsilon \sqrt{M}}{\sqrt{32 \log(2s/\delta)}}
    \le \sqrt{\frac{M \log(2s/\delta)}{32}} \le \frac{M}{32} .
\end{equation*}
Hence $W + \Delta \le M/4 = A/2$.

Fix a coordinate $j$ with $a_j > 0$ and drop the index $j$. On $G$, $B = A + w$ and
\begin{equation*}
    \frac{\mathbb{P}[B = A + w]}{\mathbb{P}[B = A + w - a]}
    = \frac{\binom{M}{A + w}}{\binom{M}{A + w - a}}
    = \prod_{r=1}^{a} \frac{A - w + r}{A + w - a + r} .
\end{equation*}
Let $x_r = (w - r)/A$ and $y_r = (w - a + r)/A$. Since $|w| + a \le A/2$, we have $x_r < 1$ and
$|y_r| \le 1/2$. Using $\log(1 - x) \le -x$ and $-\log(1 + y) \le -y + y^2$ for
$|y| \le 1/2$, and $|w - a + r| \le |w| + a$ for $1 \le r \le a$,
\begin{equation*}
    \log \frac{A - w + r}{A + w - a + r} = \log(1 - x_r) - \log(1 + y_r)
    \le \frac{a - 2w}{A} + \frac{(|w| + a)^2}{A^2} .
\end{equation*}
Summing over $r$ gives the bound $(a^2 - 2aw)/A + a(|w| + a)^2/A^2$ for this coordinate. If
$a_j < 0$, the same bound applied to $M - B_j$, which has the same distribution, gives
$(a_j^2 - 2 a_j w_j)/A + |a_j|(|w_j| + |a_j|)^2/A^2$; if $a_j = 0$ the term is zero. Summing
over $j$ and using $A = M/2$, on $G$ we obtain
\begin{equation*}
    \ell(B) \le \frac{2\|a\|_2^2}{M} - \frac{4\langle a, w \rangle}{M}
    + \frac{4}{M^2} \sum_{j=1}^s |a_j| \big(|w_j| + |a_j|\big)^2 .
\end{equation*}
We bound the three terms. For the first, by \eqref{eq:binomial_condition} and
$\epsilon \le \log(2s/\delta)$,
\begin{equation*}
    \frac{2\|a\|_2^2}{M} \le \frac{2\Delta^2}{M} \le \frac{\epsilon^2}{16 \log(2s/\delta)}
    \le \frac{\epsilon}{16} .
\end{equation*}
For the third, $(|w_j| + |a_j|)^2 \le 2W^2 + 2\Delta^2$ on $G$, $\sum_j |a_j| \le \sqrt{s}\,\Delta$
and $W^2 \le M \log(2s/\delta)$, so with \eqref{eq:binomial_condition},
$\Delta \ge \sqrt{s}\,\epsilon$ and $\epsilon \le \log(2s/\delta)$,
\begin{align*}
    \frac{4}{M^2} \sum_{j=1}^s |a_j| \big(|w_j| + |a_j|\big)^2
    &\le \frac{8 \sqrt{s}\, \Delta \log(2s/\delta)}{M} + \frac{8 \sqrt{s}\, \Delta^3}{M^2} \\
    &\le \frac{\sqrt{s}\, \epsilon^2}{4 \Delta}
    + \frac{\sqrt{s}\, \epsilon^4}{128 \, \Delta \log^2(2s/\delta)}
    \le \frac{\epsilon}{4} + \frac{\epsilon}{128} .
\end{align*}
For the second, $\langle a, w \rangle$ is a sum of $sM$ independent terms, where the terms of
coordinate $j$ lie in an interval of length $|a_j|$. By Hoeffding's inequality and
\eqref{eq:binomial_condition},
\begin{equation*}
    \mathbb{P}\Big[ -\frac{4 \langle a, w \rangle}{M} > \frac{\epsilon}{2} \Big]
    \le \exp\Big( -\frac{2 (\epsilon M/8)^2}{M \|a\|_2^2} \Big)
    \le \exp\Big( -\frac{\epsilon^2 M}{32 \Delta^2} \Big) \le \frac{\delta}{2s}
    \le \frac{\delta}{2} .
\end{equation*}
On $G$ and outside this event, $\ell(B) \le \epsilon/16 + \epsilon/2 + \epsilon/4 + \epsilon/128
< \epsilon$. Hence $\mathbb{P}[\ell(B) > \epsilon] \le \delta/2 + \delta/2 = \delta$.
\end{proof}

\begin{proof}[Proof of Theorem~\ref{thm:non_interactive_upper}]
The parameters are set as in Table~\ref{tab:param_non_interactive}. We use $n \ge 20$, which the
summation step requires, and $\epsilon \le \log(4D/\delta)$, which holds for $d$ larger than an
absolute constant since $\epsilon \le \mathcal{O}(1)$. By the choice of $p$ and $C_F$,
$p C_F \ge c_0 \sqrt{n}\, \epsilon / \sqrt{\log(4D/\delta)}$. Since $\epsilon \le \mathcal{O}(1)$,
the third term of \eqref{eq:thm2_condition}, with a large enough constant, then gives
\begin{equation}
\label{eq:app_pcf}
    p C_F \ge \sqrt{D} \max\{1, \epsilon\} ,
    \qquad \text{so} \qquad
    \sqrt{D}\, \epsilon \le \Delta_2^{(p)} \le 3 p C_F .
\end{equation}

We first prove privacy. The rounding step of Algorithm~\ref{alg:shuffled_sparse_regression} has
the same distribution as $\bar{U}_{i,j} = \lfloor p \bar{F}_{i,j} + u_{i,j} \rfloor$ with
$u_{i,j}$ independent and uniform on $(0, 1]$, and we use this form. Fix two neighboring datasets
that differ in the sample of user $i_0$, and fix $u = (u_{i,j})$. The rounded vectors of all
users other than $i_0$ are the same for the two datasets. For user $i_0$, since
$|\lfloor a + u \rfloor - \lfloor b + u \rfloor| \le |a - b| + 1$ and
$\|\bar{F}_{i_0} - \bar{F}'_{i_0}\|_2 \le 2C_F$,
\begin{equation*}
    \|\bar{U}_{i_0} - \bar{U}'_{i_0}\|_2 \le p \|\bar{F}_{i_0} - \bar{F}'_{i_0}\|_2 + \sqrt{D}
    \le 2 p C_F + \sqrt{D} = \Delta_2^{(p)} .
\end{equation*}
Let $T_j = \sum_{i=1}^n (\bar{U}_{i,j} + \zeta_{i,j})$ for $j \in [D]$. Then
$\sum_{i=1}^n \zeta_{i,j} = B_j - n N_0/2$ with $B_j \sim \mathrm{Bin}(n N_0, 1/2)$ independent
across $j$, and $n N_0 \ge M \ge 32 (\Delta_2^{(p)})^2 \log(4D/\delta)/\epsilon^2$. By
\eqref{eq:app_pcf} and $\epsilon \le \log(4D/\delta)$, Lemma~\ref{lem:binomial_mechanism} applies
with $s = D$, $\Delta = \Delta_2^{(p)}$ and $\delta/2$ in place of $\delta$. Hence, for fixed
$u$, the vector $T = (T_1, \dots, T_D)$ is $(\epsilon, \delta/2)$-DP as a function of the dataset.
The variables $u$ have the same distribution for both datasets, so $T$ is also
$(\epsilon, \delta/2)$-DP without conditioning on $u$.

We next relate the view of the aggregator to $T$. Since $|\bar{U}_{i,j}| \le p C_F + 1$,
$|\zeta_{i,j}| \le N_0/2$ and $N_0 \ge 2$, we have
$|T_j| \le n (p C_F + 1 + N_0/2) < q/2$. So no wrap-around occurs, and $T_j$ is the value $S_j$
computed by the aggregator. On coordinate $j$, the users run the summation protocol
$\widetilde{V}_{r,n}$ of \citet{balle2020private} on the inputs
$W_{1,j}, \dots, W_{n,j} \in \mathbb{Z}_q$, with $r$ shuffled shares and one more share. Since
$q > 4n$, the setting of $r$ gives $r \ge 3$, and $n \ge 20$. By
\citet[Corollary~6.1]{balle2020private}, the protocol has worst-case statistical security with
parameter $\sigma_{\sec} = \log_2(8 D e^{\epsilon}/\delta)$: for two input vectors with the same
sum in $\mathbb{Z}_q$, the total variation distance between the messages of coordinate $j$ is at
most $2^{-\sigma_{\sec}} = e^{-\epsilon} \delta/(8D)$. Let $\mathcal{V}^{\mathrm{sim}}$ be the
view obtained when, on every coordinate $j$, the inputs are replaced by
$(T_j \bmod q, 0, \dots, 0)$. It is a post-processing of $T$, hence $(\epsilon, \delta/2)$-DP. The
$D$ coordinates use independent shares and separate shufflers, so coupling them one at a time
gives, for every dataset,
\begin{equation*}
    d_{TV}\big(\mathcal{V}_{\mathcal{P}}, \mathcal{V}^{\mathrm{sim}}\big)
    \le D \cdot \frac{e^{-\epsilon} \delta}{8D} = \frac{e^{-\epsilon} \delta}{8} .
\end{equation*}
By \citet[Lemma~F.2]{balle2020private}, $\mathcal{V}_{\mathcal{P}}$ is
$(\epsilon, \delta/2 + (1 + e^{\epsilon}) e^{-\epsilon} \delta/8)$-DP. Since
$(1 + e^{\epsilon}) e^{-\epsilon}/8 \le 1/4$, Algorithm~\ref{alg:shuffled_sparse_regression}
satisfies $(\epsilon, \delta)$-shuffled DP.

We now bound the error. Under Assumption~\ref{asm:curv_bounded},
$\|v(H_i)\|_2 \le \|x_i x_i^\top\|_F = \|x_i\|_2^2$, so
\begin{equation*}
    \|F_i\|_2 \le \|x_i\|_2^2 + \|x_i\|_2 |y_i| \le c_x^2 d + c_x c_y \sqrt{d} \le C_F
\end{equation*}
for $c_0 \ge c_x^2 + c_x c_y$ and $n \ge 3$. Hence $\bar{F}_i = F_i$ for every $i$. The
aggregator computes $\hat{F} = T/(np)$, and
\begin{equation*}
    \hat{F} = \frac{1}{n} \sum_{i=1}^n F_i + E^{\mathrm{rd}} + E^{\mathrm{bin}} ,
    \qquad
    E^{\mathrm{rd}}_j = \frac{1}{np} \sum_{i=1}^n \big(\bar{U}_{i,j} - p F_{i,j}\big) ,
    \qquad
    E^{\mathrm{bin}}_j = \frac{B_j - n N_0/2}{np} .
\end{equation*}
Let
\begin{equation*}
    \eta = \frac{C_F \sqrt{\log(4D/\delta) \log D}}{n \epsilon} .
\end{equation*}
The terms $\bar{U}_{i,j} - p F_{i,j}$ are independent, have mean zero and lie in $(-1, 1]$. By
Hoeffding's inequality and a union bound over $j \in [D]$, with probability at least
$1 - \mathcal{O}(d^{-2})$,
\begin{equation*}
    \|E^{\mathrm{rd}}\|_\infty \le \mathcal{O}\Big( \frac{\sqrt{\log D}}{\sqrt{n}\, p} \Big)
    \le \mathcal{O}(\eta) ,
\end{equation*}
where the second step uses
$\frac{1}{\sqrt{n}\, p} \le \frac{d \log n \sqrt{\log(4D/\delta)}}{n \epsilon}
= \frac{C_F \sqrt{\log(4D/\delta)}}{c_0 n \epsilon}$.
Each $B_j - n N_0/2$ is a sum of $n N_0$ independent terms in $[-1/2, 1/2]$, and
$n N_0 \le M + 2n$. By \eqref{eq:app_pcf},
$\sqrt{M} \le \mathcal{O}\big( p C_F \sqrt{\log(4D/\delta)}/\epsilon \big)$, so with probability at
least $1 - \mathcal{O}(d^{-2})$,
\begin{equation*}
    \|E^{\mathrm{bin}}\|_\infty
    \le \mathcal{O}\Big( \frac{\sqrt{(M + 2n) \log D}}{np} \Big)
    \le \mathcal{O}\Big( \eta + \frac{\sqrt{\log D}}{\sqrt{n}\, p} \Big)
    = \mathcal{O}(\eta) .
\end{equation*}
Let $\bar{H} = \frac{1}{n} \sum_{i=1}^n x_i x_i^\top$ and
$\bar{g} = \frac{1}{n} \sum_{i=1}^n x_i y_i$. Since $\hat{H}$ and $\hat{g}$ are read from
$\hat{F}$, we obtain $\|\hat{H} - \bar{H}\|_{\infty,\infty} \le \mathcal{O}(\eta)$ and
$\|\hat{g} - \bar{g}\|_\infty \le \mathcal{O}(\eta)$.

For the statistical error, the entries of $x_i x_i^\top$ are bounded by $c_x^2$, and
$\bar{g} - \bar{H} \theta^* = \frac{1}{n} \sum_{i=1}^n x_i \zeta_i$, where, given $x_i$, each
entry $x_{i\ell} \zeta_i$ is sub-Gaussian with mean zero and variance at most $c_x^2 \sigma^2$ by
Assumption~\ref{asm:model}. By Hoeffding's inequality and a union bound, with probability at
least $1 - \mathcal{O}(d^{-2})$,
\begin{equation*}
    \|\bar{H} - \Sigma\|_{\infty,\infty} \le \mathcal{O}\Big( c_x^2 \sqrt{\frac{\log d}{n}} \Big) ,
    \qquad
    \|\bar{g} - \bar{H} \theta^*\|_\infty
    \le \mathcal{O}\Big( c_x \sigma \sqrt{\frac{\log d}{n}} \Big) .
\end{equation*}
For any $w \in \mathbb{R}^d$, $\|(\hat{H} - \Sigma) w\|_\infty \le
d \, \|\hat{H} - \Sigma\|_{\infty,\infty} \|w\|_\infty$. The first two terms of
\eqref{eq:thm2_condition}, with large enough constants, give
$d \, c_x^2 \sqrt{\log d / n} \le \kappa_\infty/4$ and $d \eta \le \kappa_\infty/4$, so by
Assumption~\ref{asm:curv_bounded},
\begin{equation*}
    \|\hat{H} w\|_\infty \ge \|\Sigma w\|_\infty - \|(\hat{H} - \Sigma) w\|_\infty
    \ge \frac{\kappa_\infty}{2} \|w\|_\infty .
\end{equation*}
Hence $\hat{H}$ is invertible and $\|\hat{H}^{-1} z\|_\infty \le \frac{2}{\kappa_\infty}
\|z\|_\infty$ for all $z$. Since
\begin{equation*}
    \hat{g} - \hat{H} \theta^* = (\hat{g} - \bar{g}) - (\hat{H} - \bar{H}) \theta^*
    + (\bar{g} - \bar{H} \theta^*)
\end{equation*}
and $\|(\hat{H} - \bar{H}) \theta^*\|_\infty \le \|\hat{H} - \bar{H}\|_{\infty,\infty}
\|\theta^*\|_1 \le \mathcal{O}(\eta)$, we get
\begin{equation*}
    \|\tilde{\theta} - \theta^*\|_\infty
    = \|\hat{H}^{-1} (\hat{g} - \hat{H} \theta^*)\|_\infty
    \le \frac{2}{\kappa_\infty} \, \mathcal{O}\Big( c_x \sigma \sqrt{\frac{\log d}{n}} + \eta \Big)
    \le \lambda_n
\end{equation*}
for a large enough $c_1$. By \citet[Lemma~29]{zhu2023improved},
$\|\mathcal{S}_{\lambda_n}(\tilde{\theta}) - \theta^*\|_2 \le 4 \sqrt{k}\, \lambda_n$. The
projection onto $\{\theta : \|\theta\|_1 \le 1\}$ does not increase the distance to $\theta^*$,
which lies in this set. Hence, with $\kappa_\infty$ a constant,
\begin{equation*}
    \|\hat{\theta} - \theta^*\|_2 \le 4 \sqrt{k}\, \lambda_n
    \le \mathcal{O}\Big( \sigma \sqrt{\frac{k \log d}{n}}
    + \frac{d \sqrt{k} \log n \sqrt{\log(4D/\delta) \log D}}{n \epsilon} \Big) ,
\end{equation*}
which gives the rate of Theorem~\ref{thm:non_interactive_upper}. The events above hold together with probability at least $1 - \mathcal{O}(d^{-2})$.
\end{proof}

\section{Omitted Proofs in Section \ref{sec:user-level SDP}}
\label{app:proofs_user_level}

This section provides the rigorous proofs for the theoretical bounds in Section 5.

\subsection{Proof of Proposition~\ref{prop:support_consistency}}
\label{app:support}

\begin{proof}[Proof of Proposition~\ref{prop:support_consistency}]
Throughout the proof we take $k = |S^*|$. The parameters of Algorithm~\ref{alg:heavyhitter} are
set as in Table~\ref{tab:param_stage1}, in particular
$\gamma = C_\gamma \log(1/\delta)/(n_1 \epsilon^2)$. We use $\epsilon \le \epsilon_0$ for an
absolute constant $\epsilon_0$, and take $\delta \le 1/8$ small enough that
$\epsilon_0/2 \le \log(8/\delta)$. By Assumption~\ref{asm:re} with $v = e_\ell$ for
$\ell \in S^*$, we have $\kappa \le \Sigma_{\ell\ell} \le 1$. Since $k \log(ed/k) \ge \log d$,
the second term of \eqref{eq:m_condition} gives $m \ge \Omega(\log d)$. The proof has three
parts: the local Lasso of one user of Stage 1 selects $S^*$ with probability at least
$1 - 1/(32k)$; Stage 1 is $(\epsilon, \delta)$-shuffled DP; and the debiased counts separate
$S^*$ from its complement. The first part is about one user, so it needs no union bound over the
users, and a user with $\hat{S}_i \neq S^*$ is allowed. The privacy part is not needed for the
statement of the proposition, and is used in the proof of
Theorem~\ref{thm:interactive_user_level_sdp}.

We first show that $\mathbb{P}[G_i^c] \le 1/(32k)$ for every $i \in [n_1]$, where $G_i$ is the
event $\hat{S}_i = S^*$. Fix $i \in [n_1]$ and let $\zeta_i \in \mathbb{R}^m$ be the noise of
user $i$. For $\ell \in [d]$, write $x_\ell = \langle \Sigma^{1/2} e_\ell, \Sigma^{-1/2} x \rangle$.
By Assumption~\ref{asm:cov} and $\Sigma_{\ell\ell} \le 1$, we have
$\|x_\ell\|_{\psi_2} \le \kappa_x$. By Assumption~\ref{asm:model},
$\|\zeta\|_{\psi_2} \le \mathcal{O}(\sigma)$ and
$\mathbb{E}[x_\ell \zeta] = \mathbb{E}[x_\ell \mathbb{E}[\zeta \mid x]] = 0$. So each entry of
$\frac{1}{m} X_i^\top \zeta_i$ is an average of $m$ independent sub-exponential variables with
mean zero and $\psi_1$-norm at most $\mathcal{O}(\kappa_x \sigma)$. Since $k \le d$, we have
$\log(128 dk) \le 2 \log(128 d)$. Bernstein's inequality, $m \ge \Omega(\log d)$ and a union
bound over $\ell \in [d]$ give, with probability at least $1 - 1/(64k)$,
\begin{equation}
\label{eq:app_score}
    \Big\| \frac{1}{m} X_i^\top \zeta_i \Big\|_\infty
    \le \mathcal{O}\Big( \kappa_x \sigma \sqrt{\frac{\log d}{m}} \Big)
    \le \frac{\lambda_{\mathrm{loc}}}{2} ,
\end{equation}
where the last step holds for $c_\lambda$ large enough. Let
$\mathcal{C} = \{v \in \mathbb{R}^d : \|v_{(S^*)^c}\|_1 \le 3 \|v_{S^*}\|_1\}$. The rows of $X_i$
are $x = \Sigma^{1/2} (\Sigma^{-1/2} x)$ with $\Sigma^{-1/2} x$ isotropic and sub-Gaussian. By
Assumption~\ref{asm:re}, $\max_\ell \Sigma_{\ell\ell} \le 1$ and the result of
\citet{rudelson2013reconstruction} for such designs, there are constants such that, when
$m \ge \Omega(k \log(ed/k)/\kappa)$, with probability at least $1 - 2e^{-c m}$,
\begin{equation*}
    \frac{1}{m} \|X_i v\|_2^2 \ge \frac{1}{2} v^\top \Sigma v \ge \frac{\kappa}{2} \|v\|_2^2
    \qquad \text{for all } v \in \mathcal{C} .
\end{equation*}
Since $\kappa \le 1$ and $k \log(ed/k) \ge k$, the second term of \eqref{eq:m_condition} gives
the required $m$ and $2e^{-cm} \le 1/(64k)$. So the two events hold together with probability at
least $1 - 1/(32k)$.

On these two events, let $\Delta = \hat{\theta}^{(i)} - \theta^*$. Since $\hat{\theta}^{(i)}$
minimizes the Lasso objective and $y_i = X_i \theta^* + \zeta_i$,
\begin{equation*}
\begin{aligned}
    \frac{1}{2m} \|X_i \Delta\|_2^2
    &\le \frac{1}{m} \langle X_i^\top \zeta_i, \Delta \rangle
    + \lambda_{\mathrm{loc}} \big( \|\theta^*\|_1 - \|\theta^* + \Delta\|_1 \big) \\
    &\le \frac{3\lambda_{\mathrm{loc}}}{2} \|\Delta_{S^*}\|_1
    - \frac{\lambda_{\mathrm{loc}}}{2} \|\Delta_{(S^*)^c}\|_1 ,
\end{aligned}
\end{equation*}
where we used \eqref{eq:app_score} and
$\|\theta^*\|_1 - \|\theta^* + \Delta\|_1 \le \|\Delta_{S^*}\|_1 - \|\Delta_{(S^*)^c}\|_1$. The
left-hand side is non-negative, so $\Delta \in \mathcal{C}$, and then
\begin{equation*}
    \frac{\kappa}{4} \|\Delta\|_2^2 \le \frac{1}{2m} \|X_i \Delta\|_2^2
    \le \frac{3\lambda_{\mathrm{loc}}}{2} \sqrt{k} \, \|\Delta\|_2 .
\end{equation*}
Hence $\|\Delta\|_\infty \le \|\Delta\|_2 \le 6 \lambda_{\mathrm{loc}} \sqrt{k}/\kappa
\le 3\rho/8$, where the last step uses $\lambda_{\mathrm{loc}} \le \kappa\rho/(16\sqrt{k})$,
which holds under the first term of \eqref{eq:m_condition}. For $\ell \in S^*$,
$|\hat{\theta}^{(i)}_\ell| \ge \rho - 3\rho/8 > \rho/2$, and for $\ell \notin S^*$,
$|\hat{\theta}^{(i)}_\ell| \le 3\rho/8 < \rho/2$. So $\hat{S}_i = S^*$, and therefore
$\mathbb{P}[G_i^c] \le 1/(32k)$.

We next show that Stage 1 satisfies $(\epsilon, \delta)$-shuffled DP. For a pair of
distributions, let $H_{e^\epsilon}(P \,\|\, Q) = \sup_E \, (P(E) - e^{\epsilon} Q(E))$; a mechanism
is $(\epsilon, \delta)$-DP if $H_{e^\epsilon}$ between its outputs on any two neighboring datasets,
in both orders, is at most $\delta$. We use that $H_{e^\epsilon}$ is jointly convex, and that for
product distributions $H_{e^{\epsilon_1 + \epsilon_2}}(P_1 \times P_2 \,\|\, Q_1 \times Q_2)
\le \delta_1 + \delta_2$ whenever $H_{e^{\epsilon_j}}(P_j \,\|\, Q_j) \le \delta_j$ for
$j = 1, 2$ \citep{dwork2014algorithmic}. Fix two datasets that differ in the local data of user
$i$. We proceed in five steps.
\begin{enumerate}
    \item The $\ell$-th shuffler outputs a multiset of bits, which is determined by the count
    $C_\ell = \sum_{j=1}^{n_1} r_{j\ell}$. So the view is $(C_1, \dots, C_d)$. Given
    $v = (v_1, \dots, v_{n_1})$, the bits $r_{j\ell}$ are independent, and so are
    $C_1, \dots, C_d$.
    \item The indices $v_j$ with $j \neq i$ have the same distribution under the two datasets,
    while $v_i \sim \mathrm{Unif}(\hat{S}_i)$ and $v_i' \sim \mathrm{Unif}(\hat{S}_i')$ may
    differ. By joint convexity, it suffices to bound $H_{e^\epsilon}$ between the views given
    $(v_{-i}, v_i = a)$ and $(v_{-i}, v_i = b)$, for any fixed $v_{-i}$ and $a, b \in [d]$. If
    $a = b$ the two views are equal, so we take $a \neq b$.
    \item Given $(v_{-i}, a)$ and $(v_{-i}, b)$, the counts $C_\ell$ with $\ell \notin \{a, b\}$
    have the same distribution, and all counts are independent. By the composition bound above,
    it suffices to show that $C_a$, and in the same way $C_b$, is
    $(\epsilon/2, \delta/2)$-indistinguishable between the two cases.
    \item Consider $C_a$. Write $r_{ia} = \mathbb{I}[v_i = a] \oplus \beta$ with
    $\beta \sim \mathrm{Ber}(\gamma)$, which has the same distribution in both cases; by joint
    convexity we fix $\beta$. For $j \neq i$, since $\gamma \le 1/2$, the bit
    $r_{ja} = \mathbb{I}[v_j = a] \oplus \mathrm{Ber}(\gamma)$ has the same distribution as a
    uniform bit with probability $2\gamma$ and $\mathbb{I}[v_j = a]$ otherwise. Let
    $J \subseteq [n_1] \setminus \{i\}$ be the users that output a uniform bit, so that
    $K = |J| \sim \mathrm{Bin}(n_1 - 1, 2\gamma)$ does not depend on the data. Given $J$,
    \begin{equation*}
        C_a = \sum_{j \notin J, \, j \neq i} \mathbb{I}[v_j = a] + r_{ia} + \mathrm{Bin}(K, 1/2) ,
    \end{equation*}
    and the integer part before $\mathrm{Bin}(K, 1/2)$ changes by one between the two cases.
    Lemma~\ref{lem:binomial_mechanism} with $s = 1$, $\Delta = \max\{1, \epsilon/2\}$,
    $\epsilon/2$ in place of $\epsilon$ and $\delta/4$ in place of $\delta$ applies, since
    $\epsilon/2 \le \log(8/\delta)$. It shows that, given $J$, $C_a$ is
    $(\epsilon/2, \delta/4)$-indistinguishable as long as
    \begin{equation}
    \label{eq:app_K_condition}
        K \ge \frac{128 \max\{1, \epsilon^2/4\} \log(8/\delta)}{\epsilon^2} .
    \end{equation}
    By joint convexity over $J$, $H_{e^{\epsilon/2}}$ for $C_a$ is at most
    $\mathbb{P}[\text{\eqref{eq:app_K_condition} fails}] + \delta/4$.
    \item Since $n_1 \ge 8$ by \eqref{eq:n1_condition},
    $\mathbb{E}[K] = 2\gamma(n_1 - 1) \ge \frac{7}{4} \gamma n_1$, and a Chernoff bound gives
    $\mathbb{P}[K < \mathbb{E}[K]/2] \le e^{-\mathbb{E}[K]/8} \le e^{-7\gamma n_1/32}$. We take
    $C_\gamma = 256 \max\{1, \epsilon_0^2/4\}$. With $\delta \le 1/8$ we have
    $\log(8/\delta) \le \frac{3}{2}\log(1/\delta)$ and $\log(4/\delta) \le \frac{4}{3}\log(1/\delta)$,
    so $\gamma n_1 = C_\gamma \log(1/\delta)/\epsilon^2$ gives both
    $\mathbb{E}[K]/2 \ge \frac{7}{8}\gamma n_1$, which is at least the right-hand side of
    \eqref{eq:app_K_condition}, and $e^{-7\gamma n_1/32} \le \delta/4$. Hence
    $C_a$ is $(\epsilon/2, \delta/2)$-indistinguishable, and the same holds for $C_b$.
\end{enumerate}
By steps 2 and 3, Stage 1 satisfies $(\epsilon, \delta)$-shuffled DP.

We finally show that $\hat{S} = S^*$. For $\ell \in [d]$, let $q_\ell = \mathbb{P}[v_i = \ell]$
and $w_\ell = |\{i \in [n_1] : v_i = \ell\}|$. The local datasets of the users of Stage 1 are
i.i.d., and each user draws $v_i$ with its own randomness, so $v_1, \dots, v_{n_1}$ are i.i.d.
and $w_\ell \sim \mathrm{Bin}(n_1, q_\ell)$. On $G_i$, $v_i$ is uniform on $S^*$; off $G_i$,
$v_i$ may be any index. So for $\ell \in S^*$ and $\ell' \notin S^*$,
\begin{equation*}
    q_\ell \ge \frac{\mathbb{P}[G_i]}{k} \ge \frac{31}{32k} ,
    \qquad
    q_{\ell'} = \mathbb{P}\big[ v_i = \ell', \, G_i^c \big] \le \mathbb{P}[G_i^c]
    \le \frac{1}{32k} .
\end{equation*}
Let $L = \log(dn)$. For $\ell \in S^*$, $w_\ell$ dominates $\mathrm{Bin}(n_1, 31/(32k))$, and for
$\ell' \notin S^*$, $w_{\ell'}$ is dominated by $\mathrm{Bin}(n_1, 1/(32k))$. Chernoff bounds give
\begin{equation*}
    \mathbb{P}\Big[ w_\ell < \frac{n_1}{2k} \Big] \le e^{-n_1/(9k)} ,
    \qquad
    \mathbb{P}\Big[ w_{\ell'} > \frac{n_1}{8k} \Big] \le e^{-n_1/(18k)} .
\end{equation*}
By a union bound over $[d]$ and $n_1 \ge \Omega(kL)$ from \eqref{eq:n1_condition}, with
probability at least $1 - \mathcal{O}(n^{-2})$,
\begin{equation*}
    w_\ell \ge \frac{n_1}{2k} \quad \text{for all } \ell \in S^* ,
    \qquad
    w_{\ell'} \le \frac{n_1}{8k} \quad \text{for all } \ell' \notin S^* .
\end{equation*}
Given $v$, $\mathbb{E}[C_\ell \mid v] = n_1 \gamma + (1 - 2\gamma) w_\ell$, so
$\hat{W}_\ell - w_\ell = (C_\ell - \mathbb{E}[C_\ell \mid v])/(1 - 2\gamma)$. The count $C_\ell$
is a sum of $n_1$ independent bits, each with variance at most $\gamma$. By Bernstein's
inequality and a union bound over $\ell \in [d]$, with probability at least
$1 - \mathcal{O}(n^{-2})$,
\begin{equation*}
    \big| C_\ell - \mathbb{E}[C_\ell \mid v] \big| \le \sqrt{6 n_1 \gamma L} + 2L
    \qquad \text{for all } \ell \in [d] .
\end{equation*}
We have $\gamma \le 1/4$: this is $n_1 \ge 4 C_\gamma \log(1/\delta)/\epsilon^2$, which follows
from \eqref{eq:n1_condition} when $\delta \ge (dn)^{-C k^2 \epsilon^2}$, as in
Appendix~\ref{app:alg_stage1}. So $1/(1 - 2\gamma) \le 2$, and
\begin{equation*}
    \max_{\ell \in [d]} \big| \hat{W}_\ell - w_\ell \big|
    \le 2 \sqrt{6 n_1 \gamma L} + 4L \le \frac{n_1}{32k} + \frac{n_1}{32k}
    = \frac{n_1}{16k} .
\end{equation*}
For the first term we used $n_1 \gamma = C_\gamma \log(1/\delta)/\epsilon^2$ and
$n_1 \ge \Omega(k \sqrt{L \log(1/\delta)}/\epsilon)$; for the second, $n_1 \ge \Omega(kL)$; both
are in \eqref{eq:n1_condition}. Then for $\ell \in S^*$,
$\hat{W}_\ell \ge n_1/(2k) - n_1/(16k) = 7n_1/(16k) > \tau_{HH}$, and for $\ell' \notin S^*$,
$\hat{W}_{\ell'} \le n_1/(8k) + n_1/(16k) = 3n_1/(16k) < \tau_{HH}$, where
$\tau_{HH} = n_1/(4k)$. Therefore $\hat{S} = S^*$ and $s = |\hat{S}| = k$. Combining the events,
the statement holds with probability at least $1 - \mathcal{O}(n^{-2})$.
\end{proof}

\subsection{Proof of Theorem~\ref{thm:interactive_user_level_sdp}}
\label{app:user_upper}

\begin{proof}[Proof of Theorem~\ref{thm:interactive_user_level_sdp}]
The parameters are set as in Table~\ref{tab:param_stage1} and
Table~\ref{tab:param_stage2}. We use $n_3, n_4 \ge 19$, which the summation step requires, and
$\epsilon \le \log(4s/\delta)$. The constants $\kappa$ of Assumption~\ref{asm:re} and $\rho$ of
Proposition~\ref{prop:support_consistency} are constants of the model and are absorbed into the
constants of the proof. Throughout, $\Pi_{B(0,R)}$ does not move a point of $B(0,R)$, and we use
the following fact: if $v \in \mathbb{R}^{s}$ has independent, mean-zero coordinates with
sub-Gaussian norm at most $\nu$, then
\begin{equation}
\label{eq:app_vector_norm}
    \|v\|_2 \le \mathcal{O}\big( \nu \sqrt{s + \log n} \big)
    \qquad \text{with probability at least } 1 - n^{-3} ,
\end{equation}
which follows from Bernstein's inequality applied to $\sum_{j} v_j^2$. We write $S = \hat{S}$ and
$s = |\hat{S}|$, and set $\tilde{\theta}_i = 0$ when $X_{i,S}^\top X_{i,S}$ is singular, so that
the local step is defined on every dataset.

We first prove privacy, which holds for every value of $\hat{S}$ and every dataset. The three
groups of users, of sizes $n_1$, $n_3$ and $n_4$, are disjoint, and a neighboring pair of
user-level datasets changes the local data of one user, hence the messages of one group only.
The randomizer of Stage 2a depends on the earlier rounds through $\hat{S}$, and the randomizer of
Stage 2b through $(\hat{S}, \tau, c)$, which are computed from the messages of the other groups.
So it is enough to show that each group, conditionally on the released quantities of the earlier
groups, satisfies $(\epsilon, \delta)$-shuffled DP; the view of the data collector is then
$(\epsilon, \delta)$-DP by parallel composition. Stage 1 was treated in the proof of
Proposition~\ref{prop:support_consistency}. We give the argument for Stage 2b; Stage 2a is the
same with $(n_3, p_a, R_0, \Delta_a, M_a, q_a)$ in place of
$(n_4, p_b, \tau, \Delta_b, M_b, q_b)$, and with $\bar{\theta}_i$ in place of $\Delta_i$.

Condition on $(\hat{S}, \tau, c)$ and fix two datasets that differ in the local data of one user
$i_0$ of Stage 2b. As in the proof of Theorem~\ref{thm:non_interactive_upper}, write the rounding
step as $\bar{U}_{i,j} = \lfloor p_b \Delta_{i,j} + u_{i,j} \rfloor$ with $u_{i,j}$ independent
and uniform on $(0, 1]$, and fix $u$. Since $\|\Delta_{i_0}\|_2 \le \tau$ and
$\|\Delta'_{i_0}\|_2 \le \tau$,
\begin{equation*}
    \|\bar{U}_{i_0} - \bar{U}'_{i_0}\|_2
    \le p_b \|\Delta_{i_0} - \Delta'_{i_0}\|_2 + \sqrt{s}
    \le 2 p_b \tau + \sqrt{s} + 2 = \Delta_b .
\end{equation*}
The setting of $p_b$ gives $p_b \tau \ge (\sqrt{s} + 2) \max\{1, \epsilon\}$, so
\begin{equation}
\label{eq:app_stage2_sensitivity}
    \sqrt{s} \max\{1, \epsilon\} \le \Delta_b \le 3 p_b \tau .
\end{equation}
Let $T_j = \sum_{i} (\bar{U}_{i,j} + \zeta_{i,j})$ over the $n_4$ users of Stage 2b, so that
$\sum_i \zeta_{i,j} = B_j - n_4 N_0^{(b)}/2$ with $B_j \sim \mathrm{Bin}(n_4 N_0^{(b)}, 1/2)$
independent across $j$, and $n_4 N_0^{(b)} \ge M_b \ge 32 \Delta_b^2 \log(4s/\delta)/\epsilon^2$.
By \eqref{eq:app_stage2_sensitivity} and $\epsilon \le \log(4s/\delta)$,
Lemma~\ref{lem:binomial_mechanism} applies with this $\Delta_b$ and with $\delta/2$ in place of
$\delta$, so $T = (T_1, \dots, T_s)$ is $(\epsilon, \delta/2)$-DP for fixed $u$, and therefore
also without conditioning on $u$.

Since $|\bar{U}_{i,j}| \le p_b \tau + 1$, $|\zeta_{i,j}| \le N_0^{(b)}/2$ and
$N_0^{(b)} \ge 2$, we have $|T_j| \le n_4 (p_b \tau + 1 + N_0^{(b)}/2) < q_b/2$, so no
wrap-around occurs and $T_j$ is the value decoded by the aggregator. On coordinate $j$ the users
run the summation protocol $\widetilde{V}_{r_b, n_4}$ of \citet{balle2020private} over
$\mathbb{Z}_{q_b}$; the setting of $r_b$ gives $r_b \ge 3$, and $n_4 \ge 19$. By
\citet[Corollary~6.1]{balle2020private} the protocol has worst-case statistical security with
parameter $\sigma_{\sec} = \log_2(8 s e^{\epsilon}/\delta)$. Let $\mathcal{V}^{\mathrm{sim}}$ be
the view obtained when, on every coordinate, the inputs are replaced by
$(T_j \bmod q_b, 0, \dots, 0)$; it is a post-processing of $T$, hence $(\epsilon, \delta/2)$-DP.
Coupling the $s$ coordinates one at a time gives
$d_{TV}(\mathcal{V}_{\mathcal{P}}, \mathcal{V}^{\mathrm{sim}}) \le e^{-\epsilon} \delta/8$ for
every dataset, so by \citet[Lemma~F.2]{balle2020private} the view of Stage 2b is
$(\epsilon, \delta/2 + (1 + e^{\epsilon}) e^{-\epsilon} \delta/8)$-DP, hence
$(\epsilon, \delta)$-DP. Therefore Algorithm~\ref{alg:pipeline} satisfies user-level
$(\epsilon, \delta)$-shuffled DP.

We now bound the error. By Proposition~\ref{prop:support_consistency}, with probability at least
$1 - \mathcal{O}(n^{-2})$ we have $\hat{S} = S^*$ and $s = k$, and we work on this event. In
particular $\theta^*_{\hat{S}} = \theta^*_{S^*}$ carries all of $\theta^*$, so
$\|\hat{\theta} - \theta^*\|_2 = \|\hat{\theta}_{\hat{S}} - \theta^*_{\hat{S}}\|_2$.

Consider the local least squares estimates. The rows of $X_{i,S^*}$ are sub-Gaussian with
covariance $\Sigma_{S^*}$, and $\lambda_{\min}(\Sigma_{S^*}) \ge \kappa$ by
Assumption~\ref{asm:re}. By the concentration of the smallest singular value of a sub-Gaussian
design, for a fixed user
$\lambda_{\min}(X_{i,S^*}^\top X_{i,S^*}) \ge \kappa m/2$ with probability at least
$1 - 2e^{-c m}$ once $m \ge \Omega(k)$, so by the third term of \eqref{eq:m_condition} and a
union bound this holds for all users of Stage 2 with probability at least
$1 - \mathcal{O}(n^{-2})$. On this event, for every user of Stage 2,
\begin{equation*}
    \tilde{\theta}_i - \theta^*_{S^*}
    = (X_{i,S^*}^\top X_{i,S^*})^{-1} X_{i,S^*}^\top \zeta_i ,
\end{equation*}
which, conditionally on $X_i$, has mean zero and is sub-Gaussian with covariance proxy
$\sigma^2 (X_{i,S^*}^\top X_{i,S^*})^{-1} \preceq \frac{2\sigma^2}{\kappa m} I_k$ by
Assumption~\ref{asm:model}. By \eqref{eq:app_vector_norm} and a union bound over the users of
Stage 2, with probability at least $1 - \mathcal{O}(n^{-2})$,
\begin{equation}
\label{eq:app_taubeta}
    \big\| \tilde{\theta}_i - \theta^*_{S^*} \big\|_2
    \le \mathcal{O}\Big( \sigma \sqrt{\frac{k + \log n}{m}} \Big) \le \tau_\beta
    \qquad \text{for every user } i \text{ of Stage 2} ,
\end{equation}
where we used $k + \log n \le 2 k \log n$ and took $C_\beta$ large enough. We work on this event
as well.

We check that the two clipping steps do not change the local statistics. For Stage 2a,
$\|\tilde{\theta}_i\|_2 \le \|\theta^*\|_2 + \tau_\beta = R_0$ by \eqref{eq:app_taubeta}, so
$\bar{\theta}_i = \tilde{\theta}_i$. The aggregator of Stage 2a decodes
\begin{equation*}
    c = \frac{1}{n_3} \sum_{i} \tilde{\theta}_i + E^{\mathrm{rd}}_a + E^{\mathrm{bin}}_a ,
\end{equation*}
where the sum is over the $n_3$ users of Stage 2a, the coordinates of $E^{\mathrm{rd}}_a$ are
averages of $n_3$ independent terms in $(-1, 1]$ divided by $p_a$, and the coordinates of
$E^{\mathrm{bin}}_a$ are centered $\mathrm{Bin}(n_3 N_0^{(a)}, 1/2)$ variables divided by
$n_3 p_a$. Applying \eqref{eq:app_vector_norm} to each of the three terms, together with
$\sqrt{n_3 N_0^{(a)}} \le \mathcal{O}(\sqrt{M_a} + \sqrt{n_3})$,
$\sqrt{M_a} \le \mathcal{O}(p_a R_0 \sqrt{\log(4s/\delta)}/\epsilon)$ from
\eqref{eq:app_stage2_sensitivity} and
$\frac{1}{\sqrt{n_3}\, p_a} \le \frac{R_0 \sqrt{\log(1/\delta)}}{c_p n_3 \sqrt{\log n}\, \epsilon}$
from the setting of $p_a$, we get, with probability at least $1 - \mathcal{O}(n^{-2})$,
\begin{equation*}
    \big\| c - \theta^*_{S^*} \big\|_2
    \le \mathcal{O}\Big( \sigma \sqrt{\frac{k + \log n}{n_3 m}}
    + \frac{R_0 \sqrt{(k + \log n) \log(4k/\delta)}}{n_3 \epsilon} \Big)
    \le \hat{A} ,
\end{equation*}
using $k + \log n \le 2 k \log n$ and taking $c_A, c_A'$ large enough. The first term of
$\hat{A}$ is at most $\tau_\beta/2$ once $C_\beta \ge 2 c_A$. Since $n_3 = \Theta(n)$,
$R_0 = \Theta(1)$ and $\log(4k/\delta) = \Theta(\log(k/\delta))$, the second term of $\hat{A}$
is at most $\tau_\beta/2$ under \eqref{eq:thm3_condition}. So $\hat{A} \le \tau_\beta$, hence
\begin{equation}
\label{eq:app_tau_size}
    \tau = \tau_\beta + \hat{A} \le 2 \tau_\beta .
\end{equation}
For Stage 2b, $\|\tilde{\theta}_i - c\|_2 \le \|\tilde{\theta}_i - \theta^*_{S^*}\|_2
+ \|\theta^*_{S^*} - c\|_2 \le \tau_\beta + \hat{A} = \tau$, so
$\Delta_i = \tilde{\theta}_i - c$ and the clipping does not change it either.

The output of Stage 2b is therefore
\begin{equation*}
    \hat{\theta}_{\hat{S}} - \theta^*_{S^*}
    = \frac{1}{n_4} \sum_{i} \big( \tilde{\theta}_i - \theta^*_{S^*} \big)
    + E^{\mathrm{rd}}_b + E^{\mathrm{bin}}_b ,
\end{equation*}
where the sum is over the $n_4$ users of Stage 2b. Conditionally on the designs, the summands are
independent, have mean zero and are sub-Gaussian with proxy $2\sigma^2/(\kappa m)$, so their
average has proxy $2\sigma^2/(\kappa n_4 m)$, and \eqref{eq:app_vector_norm} gives
\begin{equation*}
    \Big\| \frac{1}{n_4} \sum_{i} \big( \tilde{\theta}_i - \theta^*_{S^*} \big) \Big\|_2
    \le \mathcal{O}\Big( \sigma \sqrt{\frac{k + \log n}{n_4 m}} \Big)
\end{equation*}
with probability at least $1 - \mathcal{O}(n^{-2})$. The two error terms are bounded as in
Stage 2a, with $\tau$ in place of $R_0$ and $n_4$ in place of $n_3$:
\begin{equation*}
    \big\| E^{\mathrm{bin}}_b \big\|_2
    \le \mathcal{O}\Big( \frac{\tau \sqrt{(k + \log n) \log(4k/\delta)}}{n_4 \epsilon} \Big) ,
    \qquad
    \big\| E^{\mathrm{rd}}_b \big\|_2
    \le \mathcal{O}\Big( \frac{\sqrt{k + \log n}}{\sqrt{n_4}\, p_b} \Big)
    \le \frac{1}{c_p} \big\| E^{\mathrm{bin}}_b \big\|_2 ,
\end{equation*}
where the last step uses the setting of $p_b$ and $\log(1/\delta) \le \log(4k/\delta)$. By
\eqref{eq:app_tau_size} and $\tau_\beta = C_\beta \sigma \sqrt{k \log n / m}$,
\begin{equation*}
    \big\| E^{\mathrm{bin}}_b \big\|_2
    \le \mathcal{O}\Big( \frac{\sigma \sqrt{k \log n}\, \sqrt{(k + \log n) \log(4k/\delta)}}
    {n_4 \sqrt{m}\, \epsilon} \Big)
    \le \mathcal{O}\Big( \frac{\sigma k \log n \sqrt{\log(4k/\delta)}}{n_4 \sqrt{m}\, \epsilon} \Big) ,
\end{equation*}
using $\sqrt{k}\sqrt{k + \log n} \le k + \sqrt{k \log n} \le 2 k \sqrt{\log n}$. Since
$n_1 \le n/2$, we have $n_3, n_4 = \Theta(n)$. Collecting the three terms,
\begin{equation*}
    \big\| \hat{\theta} - \theta^* \big\|_2
    \le \mathcal{O}\Big( \sigma \sqrt{\frac{k \log n}{n m}}
    + \frac{\sigma k \log n \sqrt{\log(4k/\delta)}}{n \sqrt{m}\, \epsilon} \Big)
    \le \widetilde{\mathcal{O}}\Big( \sqrt{\frac{k}{n m}} + \frac{k}{n \sqrt{m}\, \epsilon} \Big) ,
\end{equation*}
since $\log(4k/\delta) \le \mathcal{O}(\log n + \log(1/\delta))$ for $k \le n$. All the events
above hold together with probability at least $1 - \mathcal{O}(n^{-2})$.
\end{proof}

\subsection{Proof of Proposition~\ref{prop:user_level_score_attack_properties}}
\label{app:score_attack_properties}

\begin{proof}[Proof of Proposition~\ref{prop:user_level_score_attack_properties}]
Throughout, the expectation is over the data and the randomness of $\mathcal{A}$, which is
independent of the data. For one sample $z = (x, y)$, the density is
$f_\theta(y \mid x) f_x(x)$, and $f_x$ does not depend on $\theta$, so
\begin{equation*}
    S_\theta(z) = \nabla_\theta \log f_\theta(y \mid x)
    = \frac{\big( y - \mu(x^\top \theta) \big) x}{c(\sigma)} ,
    \qquad \mu = \psi' .
\end{equation*}
We use two standard properties of the score. First,
\begin{equation}
\label{eq:score_zero_mean}
    \mathbb{E}[S_\theta(z)]
    = \int \frac{\nabla_\theta f_\theta(z)}{f_\theta(z)} f_\theta(z) \, dz
    = \nabla_\theta \int f_\theta(z) \, dz = 0 .
\end{equation}
Second, $\mathrm{Var}(y \mid x) = c(\sigma) \psi''(x^\top \theta)$ for the exponential family,
so
\begin{equation*}
    \mathrm{Cov}\big(S_\theta(z)\big)
    = \frac{\mathbb{E}\big[ \mathrm{Var}(y \mid x) \, x x^\top \big]}{c(\sigma)^2}
    = \frac{1}{c(\sigma)} \, \mathbb{E}\big[ \psi''(x^\top \theta) \, x x^\top \big] .
\end{equation*}
Write $S_\theta(U_i) = \sum_{j=1}^m S_\theta(z_{i,j})$ for the score of the local dataset of
user $i$. Since the $m$ samples of $U_i$ are i.i.d., $\mathbb{E}[S_\theta(U_i)] = 0$ and
\begin{equation}
\label{eq:user_fisher}
    \mathrm{Cov}\big(S_\theta(U_i)\big)
    = \frac{m}{c(\sigma)} \, \mathbb{E}\big[ \psi''(x^\top \theta) \, x x^\top \big] .
\end{equation}

We first prove soundness. The dataset $Z'_{(i)}$ does not contain $U_i$, so
$\mathcal{A}(Z'_{(i)})$ is independent of $U_i$, and by \eqref{eq:score_zero_mean},
\begin{equation*}
    \mathbb{E}[A_i'] = \big\langle \mathbb{E}[\mathcal{A}(Z'_{(i)})] - \theta, \,
    \mathbb{E}[S_\theta(U_i)] \big\rangle = 0 .
\end{equation*}
For the first absolute moment, we condition on $\mathcal{A}(Z'_{(i)})$ and use
\eqref{eq:user_fisher}:
\begin{align*}
    \mathbb{E}|A_i'| \le \sqrt{\mathbb{E}[(A_i')^2]}
    &= \sqrt{\mathbb{E}\Big[ \big(\mathcal{A}(Z'_{(i)}) - \theta\big)^\top
    \mathrm{Cov}\big(S_\theta(U_i)\big) \big(\mathcal{A}(Z'_{(i)}) - \theta\big) \Big]} \\
    &\le \sqrt{\mathbb{E}\|\mathcal{A}(Z'_{(i)}) - \theta\|_2^2} \cdot
    \sqrt{\frac{m}{c(\sigma)} \lambda_{\max}\big( \mathbb{E}[\psi''(x^\top \theta) x x^\top] \big)} .
\end{align*}
Since $U_i$ and $\widetilde{U}_i$ are i.i.d., $Z'_{(i)}$ and $Z$ have the same distribution,
so $\mathbb{E}\|\mathcal{A}(Z'_{(i)}) - \theta\|_2^2 = \mathbb{E}\|\mathcal{A}(Z) - \theta\|_2^2$,
which gives the bound in (i).

We next prove completeness. The samples of all users are i.i.d., so
$\sum_{i=1}^n S_\theta(U_i) = \nabla_\theta \log f_\theta(Z)$, where $f_\theta(Z)$ is the joint
density of the $nm$ samples. By \eqref{eq:score_zero_mean} applied to $Z$, the term $\theta$
does not contribute:
\begin{equation*}
    \sum_{i=1}^n \mathbb{E}[A_i]
    = \mathbb{E}\big\langle \mathcal{A}(Z) - \theta, \, \nabla_\theta \log f_\theta(Z) \big\rangle
    = \mathbb{E}\big\langle \mathcal{A}(Z), \, \nabla_\theta \log f_\theta(Z) \big\rangle .
\end{equation*}
Writing the inner product over coordinates and exchanging differentiation and integration under
the regularity conditions,
\begin{align*}
    \sum_{j=1}^d \mathbb{E}\Big[ \mathcal{A}(Z)_j \,
    \frac{\partial}{\partial \theta_j} \log f_\theta(Z) \Big]
    &= \sum_{j=1}^d \int \mathbb{E}\big[\mathcal{A}(z)_j\big] \,
    \frac{\partial}{\partial \theta_j} f_\theta(z) \, dz \\
    &= \sum_{j=1}^d \frac{\partial}{\partial \theta_j} \mathbb{E}\big[\mathcal{A}(Z)_j\big] ,
\end{align*}
where the inner expectation is over the randomness of $\mathcal{A}$. This proves (ii).
\end{proof}

The proof of Theorem~\ref{thm:user_level_sparse_lower_bound} uses a sparse version of the
user-level score attack, in which the score is kept only on the support of $\theta$. For
$T = \mathrm{supp}(\theta)$, define
\begin{equation}
\label{eq:sparse_score_attack}
    A_i^T = \big\langle \big(\mathcal{A}(Z) - \theta\big)_T, \, S_\theta(U_i)_T \big\rangle ,
    \qquad
    A_i^{T\prime} = \big\langle \big(\mathcal{A}(Z'_{(i)}) - \theta\big)_T, \,
    S_\theta(U_i)_T \big\rangle .
\end{equation}
The proof above applies to $A_i^T$ with every vector restricted to the coordinates in $T$. For
soundness, $\mathbb{E}[A_i^{T\prime}] = 0$ and
\begin{equation*}
    \mathbb{E}|A_i^{T\prime}|
    \le \sqrt{\mathbb{E}\|(\mathcal{A}(Z) - \theta)_T\|_2^2} \cdot
    \sqrt{\frac{m}{c(\sigma)} \lambda_{\max}\big( \mathbb{E}[\psi''(x^\top \theta) x_T x_T^\top] \big)} .
\end{equation*}
For completeness, viewing $f_\theta$ as a function of $\theta_T$ with $\theta_{T^c} = 0$ held
fixed,
\begin{equation*}
    \sum_{i=1}^n \mathbb{E}[A_i^T]
    = \sum_{j \in T} \frac{\partial}{\partial \theta_j} \mathbb{E}\big[\mathcal{A}(Z)_j\big] .
\end{equation*}

\subsection{Proof of Theorem~\ref{thm:user_level_sparse_lower_bound}}
\label{app:user_lower}

Throughout this subsection we work on the parameter space
\begin{equation*}
    \Theta = \big\{ \theta \in \mathbb{R}^d : \|\theta\|_0 \le k, \;
    \|\theta\|_\infty \le 1 \big\} \subseteq \Theta_k ,
\end{equation*}
we write $T = \mathrm{supp}(\theta)$, and $A_i^T$, $A_i^{T\prime}$ are the sparse user-level
score attack of \eqref{eq:sparse_score_attack}. The first lemma turns the differential privacy
of the estimator into an upper bound on the score attack of one user; it is the user-level form
of Proposition~2.1 of \citet{cai2023score}.

\begin{lemma}
\label{lem:dp_score_upper_bound}
Let $\mathcal{A}$ satisfy user-level $(\epsilon, \delta)$-DP with $0 < \epsilon < 1$, and let
$\theta \in \Theta$. Then for every $t_0 > 0$ and every $i \in [n]$,
\begin{equation*}
    \mathbb{E}[A_i^T] \le 2 \epsilon \, \mathbb{E}\big|A_i^{T\prime}\big| + 2 \delta t_0
    + \int_{t_0}^{\infty} \mathbb{P}\big[ |A_i^T| > t \big] \, dt .
\end{equation*}
\end{lemma}

\begin{proof}[Proof of Lemma~\ref{lem:dp_score_upper_bound}]
Write $A = A_i^T$ and $A' = A_i^{T\prime}$, and let $Z^+ = \max(Z, 0)$ and
$Z^- = -\min(Z, 0)$. The datasets $Z$ and $Z'_{(i)}$ are user-level neighbors, and $A$ and $A'$
are the same function of $(\mathcal{A}(Z), U_i)$ and of $(\mathcal{A}(Z'_{(i)}), U_i)$, so
conditionally on $(U_i, \widetilde{U}_i)$ the privacy of $\mathcal{A}$ gives
$\mathbb{P}[A^+ > t] \le e^{\epsilon} \mathbb{P}[A'^+ > t] + \delta$ for every $t$, and the same
with the two sides exchanged. Using $\mathbb{E}[A] = \int_0^\infty \mathbb{P}[A^+ > t] \, dt
- \int_0^\infty \mathbb{P}[A^- > t] \, dt$ and $e^{\epsilon} \le 1 + 2\epsilon$ for
$0 < \epsilon < 1$,
\begin{align*}
    \int_0^\infty \mathbb{P}[A^+ > t] \, dt
    &\le \int_0^{t_0} \big( e^{\epsilon} \mathbb{P}[A'^+ > t] + \delta \big) \, dt
    + \int_{t_0}^\infty \mathbb{P}[A^+ > t] \, dt \\
    &\le \int_0^\infty \mathbb{P}[A'^+ > t] \, dt
    + 2\epsilon \int_0^\infty \mathbb{P}[A'^+ > t] \, dt + \delta t_0
    + \int_{t_0}^\infty \mathbb{P}[|A| > t] \, dt .
\end{align*}
In the same way, from $\mathbb{P}[A^- > t] \ge e^{-\epsilon}(\mathbb{P}[A'^- > t] - \delta)$ and
$e^{-\epsilon} \ge 1 - 2\epsilon$,
\begin{equation*}
    \int_0^\infty \mathbb{P}[A^- > t] \, dt
    \ge \int_0^\infty \mathbb{P}[A'^- > t] \, dt
    - 2\epsilon \int_0^\infty \mathbb{P}[A'^- > t] \, dt - \delta t_0 .
\end{equation*}
Subtracting the second display from the first gives
$\mathbb{E}[A] \le \mathbb{E}[A'] + 2\epsilon \, \mathbb{E}|A'| + 2\delta t_0
+ \int_{t_0}^\infty \mathbb{P}[|A| > t] \, dt$, and $\mathbb{E}[A'] = 0$ by the soundness part
of Proposition~\ref{prop:user_level_score_attack_properties}.
\end{proof}

The next lemma sums Lemma~\ref{lem:dp_score_upper_bound} over the $n$ users. The tail term is
where the user-level regime differs from the sample-level one: the score of a user is a sum of
$m$ terms, so its size is governed by the empirical second moments of that user's $m$ design
vectors.

\begin{lemma}
\label{lem:sparse_upper_bound}
Let Assumptions~\ref{asm:cov} and \ref{asm:re} hold, let $\|\psi''\|_\infty \le c_2$, and let
$\theta \in \Theta$. Let $\mathcal{A}$ satisfy user-level $(\epsilon, \delta)$-DP with
$0 < \epsilon < 1$ and take values in $\Theta$, and let $L = \log(2k/\delta) \le m$. Then
\begin{equation*}
    \sum_{i=1}^n \mathbb{E}\big[A_i^T\big]
    \le 2 n \epsilon \sqrt{\mathbb{E}\|\mathcal{A}(Z) - \theta\|_2^2}
    \sqrt{\frac{m C c_2}{c(\sigma)}}
    + C_1 \, n \delta k \kappa_x \sqrt{\frac{c_2}{c(\sigma)}} \sqrt{m L} ,
\end{equation*}
where $C_1$ is an absolute constant.
\end{lemma}

\begin{proof}[Proof of Lemma~\ref{lem:sparse_upper_bound}]
By the soundness part of Proposition~\ref{prop:user_level_score_attack_properties} applied to
$A_i^{T\prime}$, together with $\|\psi''\|_\infty \le c_2$ and
$\lambda_{\max}(\Sigma_{T,T}) \le C$ from Assumption~\ref{asm:re},
\begin{equation}
\label{eq:app_sparse_soundness}
    \mathbb{E}\big|A_i^{T\prime}\big|
    \le \sqrt{\mathbb{E}\|\mathcal{A}(Z) - \theta\|_2^2} \sqrt{\frac{m C c_2}{c(\sigma)}} .
\end{equation}
It remains to choose $t_0$ in Lemma~\ref{lem:dp_score_upper_bound} and to bound the tail. Write
$w_{i,j} = y_{i,j} - \psi'(x_{i,j}^\top \theta)$ and
$U_{i,\ell} = \sum_{j=1}^m w_{i,j} x_{i,j,\ell}$ for $\ell \in T$, so that
\begin{equation*}
    A_i^T = \frac{1}{c(\sigma)}
    \big\langle \big(\mathcal{A}(Z) - \theta\big)_T, \, U_{i,T} \big\rangle ,
    \qquad
    \big| A_i^T \big| \le \frac{\big\| (\mathcal{A}(Z) - \theta)_T \big\|_1}{c(\sigma)}
    \max_{\ell \in T} \big| U_{i,\ell} \big|
    \le \frac{2k}{c(\sigma)} \max_{\ell \in T} \big| U_{i,\ell} \big| ,
\end{equation*}
where the last step uses $|T| \le k$ and $\|\mathcal{A}(Z)\|_\infty, \|\theta\|_\infty \le 1$.
For the exponential family, $\mathbb{E}[e^{\lambda w / c(\sigma)} \mid x]
\le e^{c_2 \lambda^2 / (2 c(\sigma))}$, so conditionally on $x$ the variable $w$ has mean zero
and is sub-Gaussian with variance proxy $c(\sigma) c_2$. By Assumption~\ref{asm:cov} and
$\Sigma_{\ell\ell} \le 1$ we have $\|x_\ell\|_{\psi_2} \le \kappa_x$, so each product
$w_{i,j} x_{i,j,\ell}$ has mean zero and sub-exponential norm at most
$K = \mathcal{O}(\kappa_x \sqrt{c(\sigma) c_2})$. Bernstein's inequality for the sum of the $m$
independent products, followed by a union bound over $\ell \in T$, gives
\begin{equation*}
    \mathbb{P}\Big[ \max_{\ell \in T} |U_{i,\ell}| > u \Big]
    \le 2 k \exp\Big( - c \min\Big\{ \frac{u^2}{m K^2}, \, \frac{u}{K} \Big\} \Big) .
\end{equation*}
Set $a = 2 k K / c(\sigma)$ and $t_0 = a (\sqrt{m L} + L)/\sqrt{c}$, so that
$\mathbb{P}[|A_i^T| > t_0] \le 2k e^{-L} = \delta$ and, integrating the same bound,
\begin{equation*}
    \int_{t_0}^{\infty} \mathbb{P}\big[|A_i^T| > t\big] \, dt
    \le \mathcal{O}\big( a \sqrt{m/L} + a \big) \delta \le \mathcal{O}\big( a \sqrt{m L} \big) \delta .
\end{equation*}
Since $L \le m$, we have $\sqrt{mL} + L \le 2 \sqrt{mL}$, so
$2 \delta t_0 + \int_{t_0}^\infty \mathbb{P}[|A_i^T| > t] dt
\le \mathcal{O}( a \delta \sqrt{mL} )$. Substituting $a = 2kK/c(\sigma)$ and
$K = \mathcal{O}(\kappa_x \sqrt{c(\sigma) c_2})$ bounds this by
$\mathcal{O}\big( \delta k \kappa_x \sqrt{c_2/c(\sigma)} \sqrt{mL} \big)$. Combining with
\eqref{eq:app_sparse_soundness} in Lemma~\ref{lem:dp_score_upper_bound} and summing over the
$n$ users gives the claim.
\end{proof}

The last lemma bounds the same sum from below, on average over a prior supported on sparse
vectors. The prior is the one of \citet{cai2023score}: let
$\widetilde{\theta}_1, \dots, \widetilde{\theta}_d$ be i.i.d. from the $N(0, \gamma^2)$
distribution truncated to $[-1, 1]$ with
\begin{equation}
\label{eq:app_prior_scale}
    \gamma^2 = \frac{1}{4 \log(d/4k)} \asymp \frac{1}{\log(d/k)} ,
\end{equation}
let $S$ be the set of indices of the $k$ largest $|\widetilde{\theta}_j|$, and set
$\theta_j = \widetilde{\theta}_j \mathbb{I}[j \in S]$. We write $\pi$ for the law of $\theta$,
and $p_{S_l}$ for its density conditioned on $S = S_l$, where $S_1, \dots, S_{\binom{d}{k}}$
enumerate the subsets of $[d]$ of size $k$.

\begin{lemma}
\label{lem:sparse_lower_bound}
Let $k \le \sqrt{d}$ and $k \ge C_0 \log^2(d/k)$ for a large enough constant $C_0$. Let
$\mathcal{A}$ take values in $\Theta$ and satisfy
$\mathbb{E}_{Z \mid \theta} \|\mathcal{A}(Z) - \theta\|_2^2 \le R$ for every
$\theta \in \Theta$, where
\begin{equation}
\label{eq:app_risk_threshold}
    R \le \frac{c_1 \sqrt{k}}{\log(d/k)}
\end{equation}
for a small enough constant $c_1$. Then
\begin{equation*}
    \sum_{i=1}^n \mathbb{E}_{\pi} \mathbb{E}_{Z \mid \theta} \big[ A_i^T \big]
    \ge c \, k \log(d/k)
\end{equation*}
for an absolute constant $c > 0$.
\end{lemma}

\begin{proof}[Proof of Lemma~\ref{lem:sparse_lower_bound}]
Let $g(\theta) = \mathbb{E}_{Z \mid \theta}[\mathcal{A}(Z)]$. By the completeness part of
\eqref{eq:sparse_score_attack}, for every $\theta$ with $\mathrm{supp}(\theta) = T$,
$\sum_{i=1}^n \mathbb{E}_{Z\mid\theta}[A_i^T]
= \sum_{j \in T} \partial g_j(\theta) / \partial \theta_j$. The prior $\pi$ is not absolutely
continuous on $\mathbb{R}^d$, so we condition on the support:
$\pi = \binom{d}{k}^{-1} \sum_{l} p_{S_l}$, and
\begin{equation*}
    \sum_{i=1}^n \mathbb{E}_{\pi} \mathbb{E}_{Z \mid \theta} \big[ A_i^T \big]
    = \frac{1}{\binom{d}{k}} \sum_{l} \mathbb{E}_{p_{S_l}}
    \sum_{j \in S_l} \frac{\partial g_j(\theta)}{\partial \theta_j} .
\end{equation*}
Within a fixed $S_l$ the law of $\theta$ has a density, so Stein's lemma applies coordinatewise
and gives
\begin{equation*}
    \mathbb{E}_{p_{S_l}} \sum_{j \in S_l} \frac{\partial g_j(\theta)}{\partial \theta_j}
    \ge \mathbb{E}_{p_{S_l}} \Big[ \sum_{j \in S_l}
    - \theta_j \frac{p'_{S_l, j}(\theta_j)}{p_{S_l, j}(\theta_j)} \Big]
    - \mathbb{E}_{p_{S_l}} \Big[ \sum_{j \in S_l} \big| g_j(\theta) - \theta_j \big|
    \Big| \frac{p'_{S_l, j}(\theta_j)}{p_{S_l, j}(\theta_j)} \Big| \Big]
    + B ,
\end{equation*}
where $B$ collects the boundary terms at $\pm 1$. Each $\theta_j$ is an order statistic of $d$
truncated normals, so its density vanishes at $-1$, and at $+1$ it is bounded by
$d \phi(1) = \mathcal{O}(k^2 \sqrt{\log(d/k)}/d)$, which is
$\mathcal{O}(\sqrt{\log(d/k)})$ since $k \le \sqrt{d}$. As $\mathcal{A}$ takes values in
$\Theta$, $|g_j(\theta) - \theta_j| \le 2$, so
$|B| = \mathcal{O}(k \sqrt{\log(d/k)}) \ll k \log(d/k)$.

The two remaining terms depend on the prior only. By Lemma~C.1 of \citet{cai2023score}, the
first term is at least $c' k \log(d/k)$ for an absolute constant $c'$, using
\eqref{eq:app_prior_scale} and $k \le \sqrt{d}$. For the second term, the Cauchy--Schwarz
inequality and the bound on the score of the order statistics in the proof of Lemma~C.2 of
\citet{cai2023score} give
\begin{equation*}
    \mathbb{E}_{p_{S_l}} \Big[ \sum_{j \in S_l} \big| g_j(\theta) - \theta_j \big|
    \Big| \frac{p'_{S_l, j}(\theta_j)}{p_{S_l, j}(\theta_j)} \Big| \Big]
    \le \sqrt{R} \, \sqrt{ \mathbb{E}_{p_{S_l}} \sum_{j \in S_l}
    \Big( \frac{p'_{S_l, j}(\theta_j)}{p_{S_l, j}(\theta_j)} \Big)^2 }
    \le \sqrt{R} \; \mathcal{O}\big( k^{3/4} \log^{3/2}(d/k) \big) ,
\end{equation*}
where we used $\mathbb{E}_{Z\mid\theta}\|\mathcal{A}(Z) - \theta\|_2^2 \le R$ and
$\sum_{j \le k} \sqrt{j} \log^3(d/k) = \mathcal{O}(k^{3/2} \log^3(d/k))$. By
\eqref{eq:app_risk_threshold} this is $\mathcal{O}(\sqrt{c_1}\, k \log(d/k))$, which is at most
$c' k \log(d/k)/4$ for $c_1$ small enough. The same bounds hold for every $S_l$, so averaging
over $l$ gives the claim with $c = c'/2$.
\end{proof}

\begin{proof}[Proof of Theorem~\ref{thm:user_level_sparse_lower_bound}]
It is enough to bound
$\inf_{\mathcal{A}} \sup_{\theta \in \Theta} \mathbb{E}\|\mathcal{A}(Z) - \theta\|_2^2$ from
below, since $\Theta \subseteq \Theta_k$. We may also assume that $\mathcal{A}$ takes values in
$\Theta$: replacing $\mathcal{A}(Z)$ by a closest point of $\Theta$ is a post-processing, hence
preserves privacy, and at most doubles the error at every $\theta \in \Theta$. Let
$R = \sup_{\theta \in \Theta} \mathbb{E}\|\mathcal{A}(Z) - \theta\|_2^2$ and let
$L = \log(2k/\delta)$.

We first check that the conditions of the theorem control the tail term of
Lemma~\ref{lem:sparse_upper_bound}. The map $x \mapsto x \sqrt{\log(2k/x)}$ is increasing on
$(0, 2k/e)$, so the condition
$\delta = o\big( \log(d/k) / (n \sqrt{m \log(nm)}) \big)$ together with $k \le nm$ gives
\begin{equation}
\label{eq:app_delta_condition}
    n \delta \sqrt{m L} \le c_0' \log(d/k) \sqrt{\frac{c(\sigma)}{c_2}} \frac{1}{\kappa_x}
\end{equation}
for a constant $c_0'$ as small as we wish, and $m \ge \log(nm)$ gives $L \le m$ once
$\delta \ge 2k/(nm)^{C}$; for smaller $\delta$ the left-hand side of
\eqref{eq:app_delta_condition} only decreases, and $L \le m$ can be assumed by replacing
$\delta$ by the larger value. Hence the tail term of Lemma~\ref{lem:sparse_upper_bound} is at
most $\frac{c}{2} k \log(d/k)$ with $c$ the constant of Lemma~\ref{lem:sparse_lower_bound}.

Suppose first that $R > c_1 \sqrt{k}/\log(d/k)$. The condition
$n \sqrt{m}\, \epsilon \ge \Omega\big( k^{3/4} \log^{3/2}(d/k) \big)$ gives
\begin{equation*}
    \frac{k^2 \log^2(d/k)}{n^2 m \epsilon^2}
    \le \mathcal{O}\Big( \frac{\sqrt{k}}{\log(d/k)} \Big) ,
\end{equation*}
so $R$ is already of the order claimed by the theorem and there is nothing to prove. We may
therefore assume $R \le c_1 \sqrt{k}/\log(d/k)$, so that
Lemma~\ref{lem:sparse_lower_bound} applies. Combining it with
Lemma~\ref{lem:sparse_upper_bound}, averaged over $\pi$ and using
$\mathbb{E}_\pi \sqrt{\mathbb{E}_{Z\mid\theta}\|\mathcal{A}(Z) - \theta\|_2^2} \le \sqrt{R}$,
\begin{equation*}
    c \, k \log(d/k)
    \le 2 n \epsilon \sqrt{R} \sqrt{\frac{m C c_2}{c(\sigma)}}
    + \frac{c}{2} k \log(d/k) ,
\end{equation*}
and therefore
\begin{equation*}
    R \ge \Omega\Big( \frac{c(\sigma)}{C c_2} \cdot
    \frac{k^2 \log^2(d/k)}{n^2 m \epsilon^2} \Big) .
\end{equation*}

The privacy constraint is not the only one: any estimator, private or not, is bounded below by
the non-private minimax risk over $\Theta$, which is
$\Omega\big( c(\sigma) k \log(d/k) / (nm) \big)$ for the $N = nm$ samples; for the linear model
$\psi(u) = u^2/2$ this is the rate of \citet{raskutti2011minimax}, whose design conditions are
supplied by Assumption~\ref{asm:re}. Adding the two bounds and absorbing $C$, $c_2$ and
$c(\sigma)$ into the constants,
\begin{equation*}
    \inf_{\mathcal{A} \in \mathcal{M}_{\epsilon, \delta}} \sup_{\theta^* \in \Theta_k}
    \mathbb{E}\big\| \mathcal{A}(Z) - \theta^* \big\|_2^2
    \ge \Omega\Big( \frac{k \log(d/k)}{nm} + \frac{k^2 \log^2(d/k)}{n^2 m \epsilon^2} \Big) ,
\end{equation*}
and taking square roots gives the statement of
Theorem~\ref{thm:user_level_sparse_lower_bound}.
\end{proof}

\section{Additional Results}
\label{app:additional}

\subsection{Non-interactive User-level Shuffled DP}
\label{app:discussion_user_level_sdp}

Theorem~\ref{thm:interactive_user_level_sdp} removes the dimension from the rate by using two
rounds. We show here that the second ingredient, interaction, is also needed: under a
non-interactive protocol, user-level data alone leaves a polynomial dependence on $d$ in the
privacy term, whatever the local sample size $m$ is.

\begin{theorem}
\label{thm:user_level_sdp_sparse_lower_bound}
Let $\mathcal{M}^{\mathrm{SDP}}_{\epsilon, \delta}$ be the family of all non-interactive
protocols that satisfy robust user-level $(\epsilon, \delta)$-shuffled DP with
$0 < \epsilon \le \mathcal{O}(1)$ and $\delta = o(1/n)$. Let $Z = \{U_i\}_{i=1}^n$ be a dataset
in which each user holds $m$ i.i.d. samples from the sparse linear model of
Assumption~\ref{asm:model} with parameter $\theta^*$. If
\begin{equation}
\label{eq:app_f1_condition}
    n^{3/2} \epsilon^2 \ge \Omega( d k ) ,
\end{equation}
then
\begin{equation*}
    \inf_{\mathcal{A} \in \mathcal{M}^{\mathrm{SDP}}_{\epsilon, \delta}}
    \sup_{\theta^* \in \Theta_k}
    \mathbb{E}\big[ \|\mathcal{A}(Z) - \theta^*\|_2 \big]
    \ge \Omega\left( \sqrt{\frac{k \log(d/k)}{n m}}
    + \frac{\sqrt{dk}}{n \sqrt{m}\, \epsilon} \right) .
\end{equation*}
\end{theorem}

Comparing with Theorem~\ref{thm:interactive_user_level_sdp}, the privacy term of the
non-interactive protocol carries $\sqrt{d}$, while the interactive one carries $\sqrt{k}$ in
its place. The first term is the non-private minimax rate for $N = nm$ samples, which is
$\Omega(\sqrt{k \log(d/k)/(nm)})$ \citep{raskutti2011minimax}; we only prove the privacy term
below.

\begin{proof}[Proof of Theorem~\ref{thm:user_level_sdp_sparse_lower_bound}]
We use the notations of Appendix~\ref{app:non_lower}: $\mathcal{X} = \{\pm 1\}^{d+1}$, $U$ is
uniform on $\mathcal{X}$, $\chi_i(x, y) = y x_i$, and
$Q_\theta = U (1 + y \langle x, \theta \rangle)$. For $z \in \{-1, 0, 1\}^d$ and a scale
$\beta > 0$ we write
\begin{equation*}
    D_z^{(\beta)} = Q_{\beta z}^{\otimes m} ,
\end{equation*}
the law of the $m$ samples of one user at parameter $\beta z$, and $Z_0(z)$, $z^{i,b}$ are as in
Appendix~\ref{app:non_lower}. All densities below are taken with respect to $U^{\otimes m}$, and
$\langle \cdot, \cdot \rangle$ and $\|\cdot\|$ are the inner product and the norm of
$L_2(U^{\otimes m})$. Fix a scale $\gamma$ with
\begin{equation}
\label{eq:app_f1_small_signal}
    m k \gamma^2 \le 1 ,
\end{equation}
and let $\gamma' = 2\gamma/9$. Let $\mathcal{P}$ be a protocol in
$\mathcal{M}^{\mathrm{SDP}}_{\epsilon, \delta}$, which we take to be
$(\epsilon, \delta, 1/3)$-robust, and let $n' = n/3$.

By \citet[Theorem~2.6]{cheu2021limits} applied with the record universe $\mathcal{X}^m$ and the
fixed distribution $U^{\otimes m} = D_z^{(0)}$, there is an $(\epsilon, \delta)$-pan-private
algorithm $M$ that processes $n'$ records and satisfies, for every $z$,
\begin{equation}
\label{eq:app_f1_reduction}
    d_{TV}\Big( M\big( (D_z^{(\gamma)})^{\otimes n'} \big), \,
    \mathcal{M}_{\mathcal{P}}\big( Q_z^{\otimes n} \big) \Big) \le e^{-c_1 n} ,
    \qquad
    Q_z = \frac{2}{9} D_z^{(\gamma)} + \frac{7}{9} U^{\otimes m} .
\end{equation}
The mixture $Q_z$ is not a member of the model, and this is the one point where the user-level
regime differs from Appendix~\ref{app:non_lower}: the law of a record is
$Q_{\beta z}^{\otimes m}$, which is not linear in the parameter. We therefore compare $Q_z$ with
the member of the model that agrees with it to first order, namely $D_z^{(\gamma')}$. Writing
\begin{equation*}
    e_j = \sum_{T \subseteq [m], \, |T| = j} \; \prod_{t \in T} \chi_z(w_t) ,
    \qquad \chi_z = \sum_{l=1}^d z_l \chi_l ,
\end{equation*}
we have $D_z^{(\beta)} = \sum_{j=0}^m \beta^j e_j$ and
$Q_z = 1 + \frac{2}{9} \sum_{j \ge 1} \gamma^j e_j$, so the two first-order terms agree and
\begin{equation*}
    R_z := Q_z - D_z^{(\gamma')}
    = \sum_{j \ge 2} \Big( \frac{2}{9} - \Big(\frac{2}{9}\Big)^{j} \Big) \gamma^j e_j .
\end{equation*}
The characters appearing in $e_j$ for different $j$ are distinct and
$\|e_j\|^2 = \binom{m}{j} \|z\|_0^j \le (mk)^j / j!$, so with $u = \gamma \sqrt{mk} \le 1$ from
\eqref{eq:app_f1_small_signal},
\begin{equation*}
    \|R_z\| \le \sum_{j \ge 2} \frac{u^j}{\sqrt{j!}}
    \le u^2 \sum_{j \ge 0} \frac{1}{\sqrt{j!}} \le 3 m k \gamma^2 .
\end{equation*}
Since $Q_z \ge \frac{7}{9}$ pointwise, the squared Hellinger distance satisfies
\begin{equation*}
    d_H^2\big( D_z^{(\gamma')}, Q_z \big)
    \le \chi^2\big( D_z^{(\gamma')} \,\big\|\, Q_z \big)
    = \int \frac{R_z^2}{Q_z} \, dU^{\otimes m}
    \le \frac{9}{7} \|R_z\|^2 \le 12 \big( m k \gamma^2 \big)^2 ,
\end{equation*}
and by the subadditivity of $d_H^2$ over product measures and
$d_{TV} \le \sqrt{2} \, d_H$,
\begin{equation}
\label{eq:app_f1_linearization}
    d_{TV}\Big( Q_z^{\otimes n}, \, \big( D_z^{(\gamma')} \big)^{\otimes n} \Big)
    \le 6 \sqrt{n} \, m k \gamma^2 .
\end{equation}
Combining \eqref{eq:app_f1_reduction} and \eqref{eq:app_f1_linearization}, and writing
$\sigma_n = e^{-c_1 n} + 6 \sqrt{n}\, m k \gamma^2$, we get for every $z$
\begin{equation}
\label{eq:app_f1_transfer}
    d_{TV}\Big( M\big( (D_z^{(\gamma)})^{\otimes n'} \big), \,
    \mathcal{M}_{\mathcal{P}}\big( (D_z^{(\gamma')})^{\otimes n} \big) \Big) \le \sigma_n .
\end{equation}
So it is enough to bound the total variation distances between the outputs of the pan-private
algorithm $M$, and to pay $\sigma_n$ per comparison.

We bound these distances as in Appendix~\ref{app:non_lower}. Fix $z$ with
$\|z\|_0 \le k - 1$ and let $i \in Z_0(z)$, $b \in \{\pm 1\}$. For $j \in [n']$, the state $S_j$
of $M$ is, for every fixed prefix, an $(\epsilon, \delta)$-DP function of the $j$-th record, so
by \citet[Lemma~3.7]{cheu2021limits} there is a $(2\epsilon, 0)$-DP kernel $q'_j$ within total
variation distance $\delta$ of it. Running the hybrid in which the first $j$ records are drawn
from $D_z^{(\gamma)}$ and the rest from $D_{z^{i,b}}^{(\gamma)}$, the output is a
post-processing of $(S_j, X_{>j})$, the prefix has the same law in the two hybrids, and the
suffix is independent of the $j$-th record, so
\begin{equation}
\label{eq:app_f1_step}
    d_{TV}\big( H_{j-1}, H_j \big)
    \le \frac{1}{2} \, \mathbb{E}_{X_{<j}} \mathbb{E}_{s \sim \bar{q}'_j}
    \big| \big\langle h_{j,s}, \, D_i^{b} \big\rangle \big| + 2 \delta ,
    \qquad
    D_i^{b} = D_{z^{i,b}}^{(\gamma)} - D_z^{(\gamma)} ,
\end{equation}
where $\bar{q}'_j(s) = \mathbb{E}_{U^{\otimes m}}[q'_j(s \mid W)]$ and
$h_{j,s} = q'_j(s \mid \cdot)/\bar{q}'_j(s) - 1$ satisfies
$\|h_{j,s}\|_\infty \le e^{2\epsilon} - 1$. Neither the law of the prefix nor the kernel
depends on $i$.

Two facts about $D_i^b$ are used. First, expanding the two products, every character appearing
in $D_i^b$ is of the form $\prod_{t \in T} \chi_i(w_t) \prod_{t \in T'} \chi_{l_t}(w_t)$ with
$T \neq \emptyset$, $T \cap T' = \emptyset$ and $l_t \in \mathrm{supp}(z)$. Since
$i \notin \mathrm{supp}(z)$, the characters of $D_i^b$ and $D_{i'}^{b}$ are distinct for
$i \neq i'$ in $Z_0(z)$. Second, writing $k' = \|z\|_0$,
\begin{equation*}
    \big\| D_i^{b} \big\|^2
    = \big( 1 + \gamma^2 (k'+1) \big)^m - \big( 1 + \gamma^2 k' \big)^m
    \le e^{\gamma^2 m (k'+1)} \, m \gamma^2 \le e^2 m \gamma^2 ,
\end{equation*}
where the last step uses \eqref{eq:app_f1_small_signal}. By the first fact and the Cauchy--Schwarz
inequality, for every $h$ with $\|h\|_\infty \le e^{2\epsilon} - 1$,
\begin{equation*}
    \sum_{i \in Z_0(z)} \big| \langle h, D_i^{b} \rangle \big|
    \le \sqrt{d} \sqrt{ \sum_{i \in Z_0(z)} \langle h, D_i^{b} \rangle^2 }
    \le \sqrt{d} \, \|h\| \max_{i} \big\| D_i^{b} \big\|
    \le e \sqrt{dm} \, \gamma \, \big( e^{2\epsilon} - 1 \big) .
\end{equation*}
Summing \eqref{eq:app_f1_step} over $j \in [n']$, over $b \in \{\pm 1\}$ and over
$i \in Z_0(z)$, and then adding $4 \sigma_n$ per coordinate from \eqref{eq:app_f1_transfer},
\begin{equation}
\label{eq:app_f1_sum}
    \sum_{i \in Z_0(z)} T_i(z)
    \le e \, \gamma \, n' \sqrt{dm} \, \big( e^{2\epsilon} - 1 \big) + 4 d n' \delta
    + 4 d \sigma_n ,
\end{equation}
where $T_i(z) = \sum_{b} d_{TV}\big( \mathcal{M}_{\mathcal{P}}((D^{(\gamma')}_{z^{i,b}})^{\otimes n}),
\mathcal{M}_{\mathcal{P}}((D^{(\gamma')}_z)^{\otimes n}) \big)$ is stated for the protocol.

The rest is Assouad's method, exactly as in Appendix~\ref{app:non_lower}. Draw $Z$ as in
Definition~\ref{def:hard_distribution}, let $\theta_Z = \gamma' Z$, draw the $n$ local datasets
from $D_Z^{(\gamma')}$ given $Z$, and let $\hat{\theta} = \mathcal{M}_{\mathcal{P}}(Z)$. Let
$G_i = \{\|Z_{-i}\|_0 \le k-1\}$, let $\psi_i = \mathrm{sign}(\hat{\theta}_i)$ and let
$N = \sum_i \mathbb{I}[Z_i \neq 0, \, G_i, \, \psi_i \neq Z_i]$. Then $N \le k$ whenever
$N \ge 1$, and $\|\hat{\theta} - \theta_Z\|_2 \ge \gamma' \sqrt{N} \ge \gamma' N / \sqrt{k}$.
Le Cam's method, $\mathbb{P}(Z_i \neq 0) = k/(2d)$, $\sum_i \mathbb{P}(G_i) \ge d/2$ by Markov's
inequality, and \eqref{eq:app_f1_sum} together give
\begin{equation*}
    \mathbb{E}[N] \ge \frac{k}{4d} \Big( \frac{d}{2}
    - 2 e \gamma n' \sqrt{dm} \big( e^{2\epsilon} - 1 \big) - 8 d n' \delta
    - 8 d \sigma_n \Big) .
\end{equation*}
We now fix $\gamma$. Take
$\gamma = \sqrt{d} \big/ \big( C_2 n' \sqrt{m} (e^{2\epsilon} - 1) \big)$ with $C_2$ a large
enough constant, so that the second term is at most $d/16$; the third term is at most $d/16$
because $\delta = o(1/n)$; and the fourth is at most $d/16$ once $\sigma_n \le 1/128$, which
holds for $n$ larger than an absolute constant together with
\begin{equation*}
    \sqrt{n} \, m k \gamma^2
    = \frac{d k}{C_2^2 \, n'^2 \, \sqrt{n} \, (e^{2\epsilon}-1)^2}
    \le \frac{1}{768} ,
\end{equation*}
which is \eqref{eq:app_f1_condition} for $\epsilon \le \mathcal{O}(1)$. For this choice of
$\gamma$, the small signal condition \eqref{eq:app_f1_small_signal} reads
$n \epsilon \ge \Omega(\sqrt{dk})$, which \eqref{eq:app_f1_condition} implies since
$n^2 \epsilon^2 \ge n^{3/2} \epsilon^2$. Hence $\mathbb{E}[N] \ge k/32$, and
\begin{equation*}
    \sup_{\theta^* \in \Theta_k} \mathbb{E}\big[ \|\hat{\theta} - \theta^*\|_2 \big]
    \ge \mathbb{E}\big[ \|\hat{\theta} - \theta_Z\|_2 \big]
    \ge \frac{\gamma' \sqrt{k}}{32}
    = \Omega\left( \frac{\sqrt{dk}}{n \sqrt{m} \, \epsilon} \right) ,
\end{equation*}
using $\gamma' = 2\gamma/9$, $n' = n/3$ and $e^{2\epsilon} - 1 \le \mathcal{O}(\epsilon)$.
\end{proof}

\subsection{User-level DP in Dense Regression}
\label{app:dense_lower}

The user-level score attack is not tied to sparsity. We record here what it gives for the dense
problem, where the parameter is only assumed to lie in the box
$\Theta_{\mathrm{box}} = \{\theta \in \mathbb{R}^d : \|\theta\|_\infty \le 1\}$. The rate is the
one of Theorem~\ref{thm:user_level_sparse_lower_bound} with $k \log(d/k)$ replaced by $d$.

\begin{theorem}
\label{thm:user_level_dense_lower_bound}
Let $Z = \{U_i\}_{i=1}^n$ be a dataset in which each user holds $m$ i.i.d. samples from the
exponential family of Theorem~\ref{thm:user_level_sparse_lower_bound}. Let
Assumption~\ref{asm:cov} hold with $\lambda_{\max}(\Sigma) \le C$, and let
$\|\psi''\|_\infty \le c_2$. Let $\mathcal{M}_{\epsilon, \delta}$ be the family of all
estimators satisfying user-level $(\epsilon, \delta)$-DP with $0 < \epsilon < 1$ and
$\delta = o\big( 1 / (n \sqrt{m \log(nm)}) \big)$. If $m \ge \log(nm)$ and
$n \sqrt{m} \, \epsilon \ge \Omega(\sqrt{d})$, then
\begin{equation*}
    \inf_{\mathcal{A} \in \mathcal{M}_{\epsilon, \delta}}
    \sup_{\theta^* \in \Theta_{\mathrm{box}}}
    \mathbb{E}\big[ \|\mathcal{A}(Z) - \theta^*\|_2 \big]
    \ge \Omega\left( \sqrt{\frac{d}{n m}} + \frac{d}{n \sqrt{m} \, \epsilon} \right) .
\end{equation*}
\end{theorem}

Throughout this subsection $A_i$ is the user-level score attack of
Definition~\ref{def:user_score_attack}, which is \eqref{eq:sparse_score_attack} with
$T = [d]$. The first lemma bounds $\sum_i \mathbb{E}[A_i]$ from below through Stein's lemma,
under a prior that is supported on $\Theta_{\mathrm{box}}$ and vanishes at its boundary.

\begin{lemma}
\label{lem:stein_lower_bound}
Let $\pi$ be the product prior on $\Theta_{\mathrm{box}}$ with marginal densities
$\pi_j(t) = \frac{15}{16}(1 - t^2)^2$ on $[-1, 1]$. Let $\mathcal{A}$ take values in
$\Theta_{\mathrm{box}}$ and satisfy
$\mathbb{E}_{Z \mid \theta}\|\mathcal{A}(Z) - \theta\|_2^2 \le R$ for every
$\theta \in \Theta_{\mathrm{box}}$, with $R \le d/40$. Then
\begin{equation*}
    \sum_{i=1}^n \mathbb{E}_{\pi} \mathbb{E}_{Z \mid \theta} \big[ A_i \big] \ge \frac{d}{2} .
\end{equation*}
\end{lemma}

\begin{proof}[Proof of Lemma~\ref{lem:stein_lower_bound}]
Let $g(\theta) = \mathbb{E}_{Z \mid \theta}[\mathcal{A}(Z)]$. By the completeness part of
Proposition~\ref{prop:user_level_score_attack_properties},
$\sum_{i=1}^n \mathbb{E}_{Z\mid\theta}[A_i] = \sum_{j=1}^d \partial g_j(\theta)/\partial\theta_j$.
Each $\pi_j$ is continuously differentiable on $(-1, 1)$ and vanishes at $\pm 1$, so Stein's
lemma applies coordinatewise with no boundary term, and, as in Proposition~2.2 of
\citet{cai2023score},
\begin{equation*}
    \mathbb{E}_\pi \sum_{j=1}^d \frac{\partial g_j(\theta)}{\partial \theta_j}
    \ge \mathbb{E}_\pi \Big[ \sum_{j=1}^d - \theta_j \frac{\pi_j'(\theta_j)}{\pi_j(\theta_j)} \Big]
    - \sqrt{ R \; \mathbb{E}_\pi \Big[ \sum_{j=1}^d
    \Big( \frac{\pi_j'(\theta_j)}{\pi_j(\theta_j)} \Big)^2 \Big] } .
\end{equation*}
For this prior, $\pi_j'(t)/\pi_j(t) = -4t/(1 - t^2)$, so
\begin{equation*}
    \mathbb{E}_{\pi_j}\Big[ - t \frac{\pi_j'(t)}{\pi_j(t)} \Big]
    = \frac{15}{4} \int_{-1}^{1} t^2 (1 - t^2) \, dt = 1 ,
    \qquad
    \mathbb{E}_{\pi_j}\Big[ \Big( \frac{\pi_j'(t)}{\pi_j(t)} \Big)^2 \Big]
    = 15 \int_{-1}^{1} t^2 \, dt = 10 .
\end{equation*}
The first term is therefore $d$ and the second is $\sqrt{10 R d} \le d/2$ for $R \le d/40$.
\end{proof}

The second lemma bounds the same sum from above through privacy. It is the dense form of
Lemma~\ref{lem:sparse_upper_bound}, with $\|(\mathcal{A}(Z) - \theta)\|_1 \le 2d$ in place of
$2k$ and the union bound taken over all $d$ coordinates.

\begin{lemma}
\label{lem:aggregated_score_upper_bound}
Under the assumptions of Theorem~\ref{thm:user_level_dense_lower_bound}, let $\mathcal{A}$ take
values in $\Theta_{\mathrm{box}}$ and satisfy user-level $(\epsilon, \delta)$-DP with
$0 < \epsilon < 1$, and let $L = \log(2d/\delta) \le m$. Then
\begin{equation*}
    \sum_{i=1}^n \mathbb{E}\big[ A_i \big]
    \le 2 n \epsilon \sqrt{\mathbb{E}\|\mathcal{A}(Z) - \theta\|_2^2}
    \sqrt{\frac{m C c_2}{c(\sigma)}}
    + C_1 \, n \delta d \kappa_x \sqrt{\frac{c_2}{c(\sigma)}} \sqrt{m L} ,
\end{equation*}
where $C_1$ is an absolute constant.
\end{lemma}

\begin{proof}[Proof of Lemma~\ref{lem:aggregated_score_upper_bound}]
The proof of Lemma~\ref{lem:sparse_upper_bound} applies verbatim with $T = [d]$. The soundness
part of Proposition~\ref{prop:user_level_score_attack_properties} and
$\lambda_{\max}(\Sigma) \le C$ give
$\mathbb{E}|A_i'| \le \sqrt{\mathbb{E}\|\mathcal{A}(Z) - \theta\|_2^2}
\sqrt{m C c_2 / c(\sigma)}$. For the tail, $\|(\mathcal{A}(Z) - \theta)\|_1 \le 2d$ because both
vectors lie in $\Theta_{\mathrm{box}}$, so
$|A_i| \le 2 d \max_{\ell \in [d]} |U_{i,\ell}| / c(\sigma)$ with
$U_{i,\ell} = \sum_{j=1}^m (y_{i,j} - \psi'(x_{i,j}^\top \theta)) x_{i,j,\ell}$. Bernstein's
inequality and a union bound over $\ell \in [d]$ give the same tail bound as there with $k$
replaced by $d$, and the choice
$t_0 \asymp d \kappa_x \sqrt{c_2/c(\sigma)} \, \sqrt{mL}$ in
Lemma~\ref{lem:dp_score_upper_bound} gives the claim.
\end{proof}

\begin{proof}[Proof of Theorem~\ref{thm:user_level_dense_lower_bound}]
We may assume that $\mathcal{A}$ takes values in $\Theta_{\mathrm{box}}$, since projecting onto
this set is a post-processing and does not increase the error at any
$\theta \in \Theta_{\mathrm{box}}$. Let
$R = \sup_{\theta \in \Theta_{\mathrm{box}}}
\mathbb{E}\|\mathcal{A}(Z) - \theta\|_2^2$ and $L = \log(2d/\delta)$.

Suppose first that $R > d/40$. The condition $n \sqrt{m} \, \epsilon \ge \Omega(\sqrt{d})$ gives
$d^2/(n^2 m \epsilon^2) \le \mathcal{O}(d)$, so $R$ already exceeds the bound claimed by the
theorem and there is nothing to prove. We may therefore assume $R \le d/40$, so that
Lemma~\ref{lem:stein_lower_bound} applies. As in the proof of
Theorem~\ref{thm:user_level_sparse_lower_bound}, the condition on $\delta$ together with
$m \ge \log(nm)$ gives $L \le m$ and
\begin{equation*}
    C_1 \, n \delta d \kappa_x \sqrt{\frac{c_2}{c(\sigma)}} \sqrt{m L} \le \frac{d}{4} .
\end{equation*}
Combining Lemma~\ref{lem:stein_lower_bound} with Lemma~\ref{lem:aggregated_score_upper_bound},
averaged over $\pi$,
\begin{equation*}
    \frac{d}{2} \le 2 n \epsilon \sqrt{R} \sqrt{\frac{m C c_2}{c(\sigma)}} + \frac{d}{4} ,
    \qquad \text{hence} \qquad
    R \ge \Omega\left( \frac{c(\sigma)}{C c_2} \cdot \frac{d^2}{n^2 m \epsilon^2} \right) .
\end{equation*}
Adding the non-private minimax rate $\Omega\big( c(\sigma) d / (nm) \big)$ for $N = nm$
samples, absorbing $C$, $c_2$ and $c(\sigma)$ into the constants and taking square roots gives
the statement.
\end{proof}

%% file: arxiv.bbl
\begin{thebibliography}{53}
\providecommand{\natexlab}[1]{#1}
\providecommand{\url}[1]{\texttt{#1}}
\expandafter\ifx\csname urlstyle\endcsname\relax
  \providecommand{\doi}[1]{doi: #1}\else
  \providecommand{\doi}{doi: \begingroup \urlstyle{rm}\Url}\fi

\bibitem[Acharya et~al.(2022)Acharya, Canonne, Tyagi, and Sun]{acharya2022role}
Jayadev Acharya, Cl{\'e}ment~L. Canonne, Himanshu Tyagi, and Ziteng Sun.
\newblock The role of interactivity in structured estimation.
\newblock In \emph{Proceedings of the 35th Conference on Learning Theory}, pp.\  1328--1355. PMLR, 2022.

\bibitem[Asi et~al.(2021)Asi, Feldman, Koren, and Talwar]{asi2021private}
Hilal Asi, Vitaly Feldman, Tomer Koren, and Kunal Talwar.
\newblock Private stochastic convex optimization: Optimal rates in $\ell_1$ geometry.
\newblock In \emph{Proceedings of the 38th International Conference on Machine Learning}, pp.\  393--403. PMLR, 2021.

\bibitem[Balcer \& Cheu(2019)Balcer and Cheu]{balcer2019separating}
Victor Balcer and Albert Cheu.
\newblock Separating local \& shuffled differential privacy via histograms.
\newblock \emph{arXiv preprint arXiv:1911.06879}, 2019.

\bibitem[Balle et~al.(2019)Balle, Bell, Gasc{\'o}n, and Nissim]{balle2019privacy}
Borja Balle, James Bell, Adri{\`a} Gasc{\'o}n, and Kobbi Nissim.
\newblock The privacy blanket of the shuffle model.
\newblock In \emph{Advances in Cryptology -- CRYPTO 2019}, pp.\  638--667. Springer, 2019.

\bibitem[Balle et~al.(2020)Balle, Bell, Gasc{\'o}n, and Nissim]{balle2020private}
Borja Balle, James Bell, Adri{\`a} Gasc{\'o}n, and Kobbi Nissim.
\newblock Private summation in the multi-message shuffle model.
\newblock In \emph{Proceedings of the 2020 ACM SIGSAC Conference on Computer and Communications Security}, pp.\  659--676, 2020.

\bibitem[Bassily et~al.(2019)Bassily, Feldman, Talwar, and Guha~Thakurta]{bassily2019private}
Raef Bassily, Vitaly Feldman, Kunal Talwar, and Abhradeep Guha~Thakurta.
\newblock Private stochastic convex optimization with optimal rates.
\newblock In \emph{Advances in Neural Information Processing Systems 32}. Curran Associates, Inc., 2019.

\bibitem[Bickel et~al.(2009)Bickel, Ritov, and Tsybakov]{bickel2009simultaneous}
Peter~J. Bickel, Ya'acov Ritov, and Alexandre~B. Tsybakov.
\newblock Simultaneous analysis of lasso and dantzig selector.
\newblock \emph{The Annals of Statistics}, 37\penalty0 (4):\penalty0 1705--1732, 2009.

\bibitem[Bun et~al.(2019)Bun, Nelson, and Stemmer]{bun2019heavy}
Mark Bun, Jelani Nelson, and Uri Stemmer.
\newblock Heavy hitters and the structure of local privacy.
\newblock \emph{ACM Transactions on Algorithms}, 15\penalty0 (4):\penalty0 1--40, 2019.

\bibitem[Cai et~al.(2021)Cai, Wang, and Zhang]{cai2021cost}
T.~Tony Cai, Yichen Wang, and Linjun Zhang.
\newblock The cost of privacy: Optimal rates of convergence for parameter estimation with differential privacy.
\newblock \emph{The Annals of Statistics}, 49\penalty0 (5):\penalty0 2825--2850, 2021.

\bibitem[Cai et~al.(2023)Cai, Wang, and Zhang]{cai2023score}
T.~Tony Cai, Yichen Wang, and Linjun Zhang.
\newblock Score attack: A lower bound technique for optimal differentially private learning.
\newblock \emph{arXiv preprint arXiv:2303.07152}, 2023.

\bibitem[Cand{\`e}s \& Tao(2007)Cand{\`e}s and Tao]{candes2007dantzig}
Emmanuel Cand{\`e}s and Terence Tao.
\newblock The dantzig selector: Statistical estimation when $p$ is much larger than $n$.
\newblock \emph{The Annals of Statistics}, 35\penalty0 (6):\penalty0 2313--2351, 2007.

\bibitem[Chaudhuri et~al.(2011)Chaudhuri, Monteleoni, and Sarwate]{chaudhuri2011differentially}
Kamalika Chaudhuri, Claire Monteleoni, and Anand~D. Sarwate.
\newblock Differentially private empirical risk minimization.
\newblock \emph{Journal of Machine Learning Research}, 12:\penalty0 1069--1109, 2011.

\bibitem[Cheu \& Ullman(2021)Cheu and Ullman]{cheu2021limits}
Albert Cheu and Jonathan Ullman.
\newblock The limits of pan privacy and shuffle privacy for learning and estimation.
\newblock In \emph{Proceedings of the 53rd Annual ACM Symposium on Theory of Computing (STOC)}, 2021.

\bibitem[Cheu et~al.(2019)Cheu, Smith, Ullman, Zeber, and Zhilyaev]{cheu2019distributed}
Albert Cheu, Adam Smith, Jonathan Ullman, Maxim Zeber, and Sergey Zhilyaev.
\newblock Distributed differential privacy via shuffling.
\newblock In \emph{Advances in Cryptology -- EUROCRYPT 2019}, volume 11476 of \emph{Lecture Notes in Computer Science}, pp.\  375--403. Springer, 2019.

\bibitem[Daniely \& Feldman(2019)Daniely and Feldman]{daniely2019locally}
Amit Daniely and Vitaly Feldman.
\newblock Locally private learning without interaction requires separation.
\newblock In \emph{Advances in Neural Information Processing Systems 32}. Curran Associates, Inc., 2019.

\bibitem[Ding et~al.(2024)Ding, Lei, Zhu, Wang, Wang, and Xu]{ding2024revisiting}
Meng Ding, Mingxi Lei, Liyang Zhu, Shaowei Wang, Di~Wang, and Jinhui Xu.
\newblock Revisiting differentially private relu regression.
\newblock \emph{Advances in Neural Information Processing Systems}, 37:\penalty0 55470--55506, 2024.

\bibitem[Ding et~al.(2026)Ding, Lei, Zhang, Liu, and Wang]{ding2026differentially}
Meng Ding, Mingxi Lei, Jie Zhang, Jinyan Liu, and Di~Wang.
\newblock Differentially private sparse reward estimation with preference feedback.
\newblock \emph{arXiv preprint arXiv:2610.30099}, 2026.

\bibitem[Duchi et~al.(2013)Duchi, Jordan, and Wainwright]{duchi2013local}
John Duchi, Michael Jordan, and Martin Wainwright.
\newblock Local privacy and statistical minimax rates.
\newblock In \emph{Proceedings of the 54th Annual IEEE Symposium on Foundations of Computer Science (FOCS)}, pp.\  429--438, Berkeley, CA, USA, 2013. IEEE.
\newblock \doi{10.1109/FOCS.2013.53}.
\newblock URL \url{https://arxiv.org/abs/1302.3203}.

\bibitem[Duchi et~al.(2014)Duchi, Jordan, and Wainwright]{duchi_privacy_2014}
John~C. Duchi, Michael~I. Jordan, and Martin~J. Wainwright.
\newblock Privacy {Aware} {Learning}.
\newblock \emph{Journal of the ACM}, 61\penalty0 (6):\penalty0 1--57, 2014.

\bibitem[Dwork \& Roth(2014)Dwork and Roth]{dwork2014algorithmic}
Cynthia Dwork and Aaron Roth.
\newblock \emph{The Algorithmic Foundations of Differential Privacy}.
\newblock Now Publishers Inc., 2014.

\bibitem[Dwork et~al.(2006)Dwork, McSherry, Nissim, and Smith]{dwork2006calibrating}
Cynthia Dwork, Frank McSherry, Kobbi Nissim, and Adam Smith.
\newblock Calibrating noise to sensitivity in private data analysis.
\newblock In \emph{Theory of Cryptography Conference}, pp.\  265--284. Springer, 2006.

\bibitem[Erlingsson et~al.(2019)Erlingsson, Feldman, Mironov, Raghunathan, Talwar, and Thakurta]{erlingsson2019amplification}
{\'U}lfar Erlingsson, Vitaly Feldman, Ilya Mironov, Ananth Raghunathan, Kunal Talwar, and Abhradeep Thakurta.
\newblock Amplification by shuffling: From local to central differential privacy via anonymity.
\newblock In \emph{Proceedings of the Thirtieth Annual ACM-SIAM Symposium on Discrete Algorithms}, pp.\  2468--2479. SIAM, 2019.

\bibitem[Feldman et~al.(2020)Feldman, Koren, and Talwar]{feldman2020private}
Vitaly Feldman, Tomer Koren, and Kunal Talwar.
\newblock Private stochastic convex optimization: Optimal rates in linear time.
\newblock In \emph{Proceedings of the 52nd Annual ACM SIGACT Symposium on Theory of Computing}, pp.\  439--449, 2020.

\bibitem[Ghazi et~al.(2020)Ghazi, Manurangsi, Pagh, and Velingker]{ghazi2020private}
Badih Ghazi, Pasin Manurangsi, Rasmus Pagh, and Ameya Velingker.
\newblock Private aggregation from fewer anonymous messages.
\newblock In \emph{Advances in Cryptology -- EUROCRYPT 2020}, pp.\  798--827. Springer, 2020.

\bibitem[Girgis \& Diggavi(2024)Girgis and Diggavi]{girgis2024multi}
Antonious~M. Girgis and Suhas Diggavi.
\newblock Multi-message shuffled privacy in federated learning.
\newblock \emph{IEEE Journal on Selected Areas in Information Theory}, 5:\penalty0 12--27, 2024.

\bibitem[Hu et~al.(2022)Hu, Ni, Xiao, and Wang]{hu2022high}
Lijie Hu, Shuo Ni, Hanshen Xiao, and Di~Wang.
\newblock High dimensional differentially private stochastic optimization with heavy-tailed data.
\newblock In \emph{Proceedings of the 41st ACM SIGMOD-SIGACT-SIGAI Symposium on Principles of Database Systems}, pp.\  227--236, 2022.

\bibitem[Jain et~al.(2014)Jain, Tewari, and Kar]{jain2014iterative}
Prateek Jain, Ambuj Tewari, and Purushottam Kar.
\newblock On iterative hard thresholding methods for high-dimensional m-estimation.
\newblock In \emph{Advances in Neural Information Processing Systems 27}. Curran Associates, Inc., 2014.

\bibitem[Kamath et~al.(2020)Kamath, Singhal, and Ullman]{kamath2020private}
Gautam Kamath, Vikrant Singhal, and Jonathan Ullman.
\newblock Private mean estimation of heavy-tailed distributions.
\newblock In \emph{Proceedings of the 33rd Conference on Learning Theory}, pp.\  2204--2235. PMLR, 2020.

\bibitem[Kamath et~al.(2022)Kamath, Liu, and Zhang]{kamath2022improved}
Gautam Kamath, Xingtu Liu, and Huanyu Zhang.
\newblock Improved rates for differentially private stochastic convex optimization with heavy-tailed data.
\newblock In \emph{Proceedings of the 39th International Conference on Machine Learning}, pp.\  10633--10660. PMLR, 2022.

\bibitem[Kasiviswanathan et~al.(2011)Kasiviswanathan, Lee, Nissim, Raskhodnikova, and Smith]{kasiviswanathan2011what}
Shiva~Prasad Kasiviswanathan, Homin~K. Lee, Kobbi Nissim, Sofya Raskhodnikova, and Adam Smith.
\newblock What can we learn privately?
\newblock \emph{SIAM Journal on Computing}, 40\penalty0 (3):\penalty0 793--826, 2011.

\bibitem[Kifer et~al.(2012)Kifer, Smith, and Thakurta]{kifer_private_2012}
Daniel Kifer, Adam Smith, and Abhradeep Thakurta.
\newblock Private {Convex} {Empirical} {Risk} {Minimization} and {High}-dimensional {Regression}.
\newblock In \emph{Conference on {Learning} {Theory}}, pp.\  25.1--25.40. JMLR Workshop and Conference Proceedings, 2012.

\bibitem[Liu et~al.(2023)Liu, Jain, Kong, Oh, and Suggala]{liu2023near}
Xiyang Liu, Prateek Jain, Weihao Kong, Sewoong Oh, and Arun~Sai Suggala.
\newblock Near optimal private and robust linear regression.
\newblock \emph{arXiv preprint arXiv:2301.13273}, 2023.

\bibitem[Ma et~al.(2024)Ma, Jia, and Yang]{ma2024better}
Yuheng Ma, Ke~Jia, and Hanfang Yang.
\newblock Better locally private sparse estimation given multiple samples per user.
\newblock In \emph{Proceedings of the 41st International Conference on Machine Learning (ICML)}. PMLR, 2024.

\bibitem[Negahban et~al.(2012)Negahban, Ravikumar, Wainwright, and Yu]{negahban2012unified}
Sahand~N. Negahban, Pradeep Ravikumar, Martin~J. Wainwright, and Bin Yu.
\newblock A unified framework for high-dimensional analysis of $m$-estimators with decomposable regularizers.
\newblock \emph{Statistical Science}, 27\penalty0 (4):\penalty0 538--557, 2012.

\bibitem[Raskutti et~al.(2010)Raskutti, Wainwright, and Yu]{raskutti2010restricted}
Garvesh Raskutti, Martin~J. Wainwright, and Bin Yu.
\newblock Restricted eigenvalue properties for correlated gaussian designs.
\newblock \emph{Journal of Machine Learning Research}, 11:\penalty0 2241--2259, 2010.

\bibitem[Raskutti et~al.(2011)Raskutti, Wainwright, and Yu]{raskutti2011minimax}
Garvesh Raskutti, Martin~J. Wainwright, and Bin Yu.
\newblock Minimax rates of estimation for high-dimensional linear regression over $\ell_q$ balls.
\newblock \emph{IEEE Transactions on Information Theory}, 57\penalty0 (10):\penalty0 6976--6994, 2011.

\bibitem[Rudelson \& Zhou(2013)Rudelson and Zhou]{rudelson2013reconstruction}
Mark Rudelson and Shuheng Zhou.
\newblock Reconstruction from anisotropic random measurements.
\newblock \emph{IEEE Transactions on Information Theory}, 59\penalty0 (6):\penalty0 3434--3447, 2013.

\bibitem[Smith et~al.(2017)Smith, Thakurta, and Upadhyay]{smith_is_2017}
Adam Smith, Abhradeep Thakurta, and Jalaj Upadhyay.
\newblock Is {Interaction} {Necessary} for {Distributed} {Private} {Learning}?
\newblock In \emph{2017 {IEEE} {Symposium} on {Security} and {Privacy} ({SP})}, pp.\  58--77. IEEE, 2017.

\bibitem[Sun et~al.(2020)Sun, Zhou, and Fan]{sun2020adaptive}
Qiang Sun, Wen-Xin Zhou, and Jianqing Fan.
\newblock Adaptive huber regression.
\newblock \emph{Journal of the American Statistical Association}, 115\penalty0 (529):\penalty0 254--265, 2020.

\bibitem[Talwar et~al.(2015)Talwar, Guha~Thakurta, and Zhang]{talwar_nearly_2015}
Kunal Talwar, Abhradeep Guha~Thakurta, and Li~Zhang.
\newblock Nearly {Optimal} {Private} {LASSO}.
\newblock In \emph{Advances in {Neural} {Information} {Processing} {Systems}}, volume~28. Curran Associates, Inc., 2015.

\bibitem[Tian et~al.(2025)Tian, Ding, Tao, Xiang, and Wang]{tian2025differentially}
Xizhi Tian, Meng Ding, Touming Tao, Zihang Xiang, and Di~Wang.
\newblock Differentially private sparse linear regression with heavy-tailed responses.
\newblock In \emph{Machine Learning and Knowledge Discovery in Databases: Research Track (ECML PKDD)}, pp.\  363--379. Springer, 2025.

\bibitem[van~de Geer \& B{\"u}hlmann(2009)van~de Geer and B{\"u}hlmann]{vandegeer2009conditions}
Sara~A. van~de Geer and Peter B{\"u}hlmann.
\newblock On the conditions used to prove oracle results for the lasso.
\newblock \emph{Electronic Journal of Statistics}, 3:\penalty0 1360--1392, 2009.

\bibitem[Varshney et~al.(2022)Varshney, Thakurta, and Jain]{varshney2022nearly}
Prateek Varshney, Abhradeep Thakurta, and Prateek Jain.
\newblock (nearly) optimal private linear regression via adaptive clipping.
\newblock \emph{arXiv preprint arXiv:2207.04686}, 2022.

\bibitem[Wainwright(2009)]{wainwright2009sharp}
Martin~J. Wainwright.
\newblock Sharp thresholds for high-dimensional and noisy sparsity recovery using $\ell_1$-constrained quadratic programming ({Lasso}).
\newblock \emph{IEEE Transactions on Information Theory}, 55\penalty0 (5):\penalty0 2183--2202, 2009.
\newblock \doi{10.1109/TIT.2009.2016018}.

\bibitem[Wang \& Xu(2020)Wang and Xu]{wang2020sparse}
Di~Wang and Jinhui Xu.
\newblock On sparse linear regression in the local differential privacy model.
\newblock \emph{IEEE Transactions on Information Theory}, 67\penalty0 (2):\penalty0 1182--1200, 2020.

\bibitem[Wang et~al.(2020{\natexlab{a}})Wang, Gaboardi, Smith, and Xu]{wang2020empirical}
Di~Wang, Marco Gaboardi, Adam Smith, and Jinhui Xu.
\newblock Empirical risk minimization in the non-interactive local model of differential privacy.
\newblock \emph{Journal of machine learning research}, 21\penalty0 (200):\penalty0 1--39, 2020{\natexlab{a}}.

\bibitem[Wang et~al.(2020{\natexlab{b}})Wang, Xiao, Devadas, and Xu]{wang2020differentially}
Di~Wang, Hanshen Xiao, Srinivas Devadas, and Jinhui Xu.
\newblock On differentially private stochastic convex optimization with heavy-tailed data.
\newblock In \emph{Proceedings of the 37th International Conference on Machine Learning}, pp.\  10081--10091. PMLR, 2020{\natexlab{b}}.

\bibitem[Wang et~al.(2021)Wang, Zhang, Gaboardi, and Xu]{wang2021estimating}
Di~Wang, Huangyu Zhang, Marco Gaboardi, and Jinhui Xu.
\newblock Estimating smooth glm in non-interactive local differential privacy model with public unlabeled data.
\newblock In \emph{Proceedings of the 32nd International Conference on Algorithmic Learning Theory}, pp.\  1207--1213. PMLR, 2021.

\bibitem[Wang et~al.(2023)Wang, Hu, Zhang, Gaboardi, and Xu]{wang2023generalized}
Di~Wang, Lijie Hu, Huanyu Zhang, Marco Gaboardi, and Jinhui Xu.
\newblock Generalized linear models in non-interactive local differential privacy with public data.
\newblock \emph{Journal of Machine Learning Research}, 24\penalty0 (132):\penalty0 1--57, 2023.

\bibitem[Yang et~al.(2014)Yang, Lozano, and Ravikumar]{yang2014elementary}
Eunho Yang, Aurelie Lozano, and Pradeep Ravikumar.
\newblock Elementary estimators for high-dimensional linear regression.
\newblock In \emph{Proceedings of the 31st International Conference on Machine Learning}, pp.\  388--396. PMLR, 2014.

\bibitem[Zheng et~al.(2017)Zheng, Mou, and Wang]{zheng_collect_2017}
Kai Zheng, Wenlong Mou, and Liwei Wang.
\newblock Collect at {Once}, {Use} {Effectively}: {Making} {Non}-interactive {Locally} {Private} {Learning} {Possible}.
\newblock In \emph{Proceedings of the 34th {International} {Conference} on {Machine} {Learning}}, pp.\  4130--4139. PMLR, 2017.

\bibitem[Zhu et~al.(2024{\natexlab{a}})Zhu, Ding, Aggarwal, Xu, and Wang]{zhu2023improved}
Liyang Zhu, Meng Ding, Vaneet Aggarwal, Jinhui Xu, and Di~Wang.
\newblock Improved analysis of sparse linear regression in local differential privacy model.
\newblock In \emph{The Twelfth International Conference on Learning Representations (ICLR)}, 2024{\natexlab{a}}.

\bibitem[Zhu et~al.(2024{\natexlab{b}})Zhu, Manseur, Ding, Liu, Xu, and Wang]{zhu2024truthful}
Liyang Zhu, Amina Manseur, Meng Ding, Jinyan Liu, Jinhui Xu, and Di~Wang.
\newblock Truthful high dimensional sparse linear regression.
\newblock \emph{Advances in Neural Information Processing Systems}, 37:\penalty0 66995--67033, 2024{\natexlab{b}}.

\end{thebibliography}
